%% file: main.tex
\documentclass[11pt,a4paper]{article}
\usepackage{booktabs}
\usepackage[utf8]{inputenc}
\usepackage{geometry}
\usepackage{graphicx}
\usepackage{multirow}
\usepackage{amsmath}
\usepackage{amsthm}
\usepackage{stackengine}
\usepackage{amssymb}
\usepackage{algorithm}
\usepackage{algorithmicx}
\usepackage{algpseudocode} 
\usepackage{makecell}
\usepackage{color}
\usepackage{tikz}
\usepackage{caption}
\usepackage{lscape}
\usepackage{pdflscape}
\usepackage{rotating}
\usepackage{algpseudocode}
\usepackage{empheq} 
\usepackage{color,soul}
\usepackage{hyperref}

\usepackage{amsmath}

\usetikzlibrary{positioning, automata, chains, arrows.meta}
\usetikzlibrary{matrix,positioning,calc}
\newcommand{\True}{\mbox{{1}}}
\newcommand{\False}{\mbox{{0}}}

\newtheorem{mytheorem}{Theorem}
\newtheorem{mylemma}{Lemma}

\newtheorem{remark}{Remark}

\algnewcommand{\LineComment}[1]{\State \(\triangleright\) #1}

\newcommand{\xv}{\boldsymbol{x}}
\newcommand{\stS}{\mathcal{S}}

\newcommand\tab[1][1cm]{\hspace*{#1}}

\title{On the Convergence of Tsetlin Machines for the XOR Operator}
\author{Lei Jiao, Xuan Zhang,  Ole-Christoffer Granmo, and K. Darshana Abeyrathna}

\begin{document}

\maketitle

\begin{abstract} \label{abstract}
The Tsetlin Machine (TM) is a novel machine learning algorithm with several distinct properties,  including transparent inference and learning using hardware-near building blocks. Although numerous papers explore the TM empirically, many of its properties have not yet been analyzed mathematically. In this article, we analyze the convergence of the TM when input is non-linearly related to output by the XOR-operator. Our analysis reveals that the TM, with just two conjunctive clauses, can converge almost surely to reproducing XOR, learning from training data over an infinite time horizon. Furthermore, the analysis shows how the hyper-parameter $T$ guides clause construction so that the clauses capture the distinct sub-patterns in the data. Our analysis of convergence for XOR thus lays the foundation for analyzing other more complex logical expressions. These analyses altogether, from a mathematical perspective, provide new insights on why TMs have obtained state-of-the-art performance on several pattern recognition problems.
\end{abstract}

\section{Introduction}


The Tsetlin Machine (TM) \cite{granmo2018tsetlin} employs groups of Tsetlin Automata (TAs)  \cite{Tsetlin1961}, which operate on binary data using propositional logic. Via a game-theoretic collaboration scheme, the TAs self-organize to capture the distinct patterns in the data. In brief, each group of TAs builds a conjunctive clause that captures a specific pattern.   The dynamics of the collaboration involves three interacting mechanisms. High pattern recall is enforced by a resource allocation mechanism that diversifies clause construction. Simultaneously, a mechanism that forces the clauses to capture frequent patterns combats overfitting. Finally, without compromising high pattern frequency, the discrimination power of the clauses is optimized by injecting discriminative features.

TMs provide two main advantages: transparent inference and learning combined with hardware-near building blocks. TM transparency, which unravels the reasoning behind the decision making process, addresses one of the most critical challenges in Artificial Intelligence (AI) research -- lack of interpretability~\cite{ribeiro2016should}. In particular, deep learning-based approaches mainly employ post-processing for \emph{approximate} local interpretation of individual predictions, which do not guarantee model fidelity~\cite{Rudin2019}. TMs, on the other hand, is founded on conjunctive clauses in propositional logic, which have been postulated as particularly easy for humans to comprehend~\cite{valiant12}. TMs further facilitate derivation of closed formula expressions for both local and global interpretability, akin to SHAP \cite{Blakely2020}. Computationally, TMs can be realized via a set of finite-state automata --- the TAs --- which are well-suited for implementation in hardware, such as on FPGA \cite{wheeldon2020learning}. Different from the extensive arithmetic operations required by most other AI approaches, a TA learns using increment and decrement operations only \cite{Tsetlin1961}. Indeed, due to the robustness of TA learning and TM pattern representation, the TM paradigm is shown to be inherently fault-tolerant, completely masking stuck-at faults~\cite{shafik2020explainability}.

There are many variations of TMs, with two main architectures being the convolutional TM (CTM) \cite{granmo2019convolutional} and the regression TM (RTM) \cite{ abeyrathna2019nonlinear,abeyrathna2020integerregression}. These have been employed in several application domains, such as medical text analysis \cite{berge2019using}, aspect-based sentiment analysis~\cite{rohan2021AAAI}, disease outbreak forecasting \cite{abeyrathna2019scheme}, and other medical applications \cite{abeyrathna2020integer}. The above studies report that TMs, with smaller memory footprint and higher computational efficiency, obtain better or competitive classification and regression accuracy compared with most of the state-of-the-art AI techniques, while maintaining transparency. 
Although numerous papers explore the TM empirically, many of its properties have not yet been analyzed mathematically.
In \cite{zhang2020convergence}, convergence for unary operators on one-bit data, i.e., the IDENTITY- and the NOT operators, is analyzed. There, we first proved that the TM can converge almost surely to the intended pattern when the training data is noise-free. Thereafter, we analyzed the effect of noise, establishing how the noise probability of the data and the granularity parameter of the TM govern convergence \cite{zhang2020convergence}.




\textbf{Paper Contributions.} In this paper, we analyze the ``XOR'' case, which deals with the binary XOR operator, encompassing two critical sub-patterns. We start from a simple structure of two clauses, each of which has four TAs with only two states. For this structure, we prove convergence via discrete time Markov chain (DTMC) analysis, analysing the ability of TMs to learn the XOR operator from data. Thereafter, we investigate the convergence behavior for more than two clauses. From the latter analysis, we reveal the crucial role the hyper-parameter $T$ of TMs plays, showing how this parameter controls the ability to robustly capture multiple sub-patterns within one class, through allocating sparse pattern representation resources (the clauses). 

\textbf{Paper Organization.} The remaining of the paper is organized as follows. Section \ref{Review} briefly reviews the TM and specifies the training process for XOR. In Section~\ref{XORstudy}, we present our analytical procedure and the main analytical results. We conclude the paper in Section \ref{conclusions}.

\section{Review of the Tsetlin Machine}\label{Review}

In this section, we present the TM in brief, including an overview of TA, the TM architecture, and the training process of TMs. A more comprehensive exposition can be found in \cite{granmo2018tsetlin}. 

\subsection{Tsetlin Automata (TA)}\label{TA}
A TA is a fixed structure deterministic learning automaton \cite{Narendra1989LearningIntroduction,zhang2019conclusive}, forming a crucial component of TM learning. By interacting with the environment, a TA aims to learn the action that offers the highest probability of providing a reward~\cite{Tsetlin1961}. Figure~\ref{figure:TAarchitecture_basic_2n} illustrates a two-action TA with $2N$ states, where $N\in[1, +\infty)$. Which action a TA selects is decided by its current state, which triggers a response from the environment followed by the TM making a state transition. That is, when the TA is in states $1$ to $N$, i.e., on the left-hand side of the state-space shown in Figure~\ref{figure:TAarchitecture_basic_2n}, Action 1 is chosen. If the TA on the other hand finds itself in states $N+1$ to $2N$, i.e., on the right-hand side, Action 2 is chosen. 
Once an action is chosen, the environment responds with either a reward or a penalty. When the TA receives a penalty, it will move towards the opposite half of the state space, that is, towards the other action. This transition is marked by the solid arrows in Figure~\ref{figure:TAarchitecture_basic_2n}. Conversely, if the TA receives a reward, it will switch to a ``deeper'' state by transitioning to the left or the right end of the chain, depending on whether the current action is Action 1 or Action 2. In the figure, this transition is captured by the dashed arrows. Note that the number of states in a TA, i.e., $2N$, can be adjusted. The larger the number, the slower the convergence. However, the TA learns more accurately in a stochastic environment with a larger number of states. 

\input{Figures/TA}
\input{Figures/TAteam}
\input{Figures/voting}

\subsection{Tsetlin Machines (TMs)}\label{sect:TM}
A TM is formed by $m$ teams of TAs. 
The TAs operate on binary input and employs propositional logic to represent patterns. In general, the input of a TM can be represented by $\bold{X}=[x_1, x_2, \ldots, x_o]$, with $x_k \in \{0, 1\}, k=1, 2, \ldots, o$. Each TA team contains $o$ pairs of TAs, with each pair being responsible for a certain input variable $x_k$. Figure \ref{fig:tateam} shows such a TA team $\mathcal{G}^i_j=\{\mathrm{TA}^{i,j}_{k'}|1\leq k'\leq 2o\}$ that has $2o$ TAs. The index $i$ refers to a specific pattern class and $j$ is the index of a specific clause. The automaton $\mathrm{TA}^{i,j}_{2k-1}$ returns the input $x_k$ as is, whereas $\mathrm{TA}^{i,j}_{2k}$ addresses the negation of $x_k$, i.e., $\neg x_k$. Note that the inputs and their negations are jointly referred to as literals.

Each TA chooses one of two actions, i.e., it either ``Includes'' or ``Excludes'' its literal, outputting $I(\cdot)$ and $E(\cdot)$, respectively.  Let $I(x)=x, ~I(\neg x)=\neg x$, and $E(\cdot)=1$, with the latter meaning that an excluded literal does not contribute to the output. 
Collectively, the $I(\cdot)$/$E(\cdot)$-outputs of the TA team then take part in a conjunction, expressed by the conjunctive clause \cite{zhang2020convergence}:
\begin{equation}
\label{eqn:clause1}
C^i_j(\bold{X}) = \begin{cases}
\left(\bigwedge\limits_{k \in I^i_j} {x_k}\right) \wedge \left(\bigwedge\limits_{k \in \bar{I}^i_j} {\neg x_k}\right) \wedge 1 & \mathrm{During\ training},\\
\left(\left(\bigwedge\limits_{k \in I^i_j} {x_k}\right) \wedge \left(\bigwedge\limits_{k \in \bar{I}^i_j} {\neg x_k}\right)\right) \vee 0 & \mathrm{During\ testing}.
\end{cases}
\end{equation}
In Eq. (\ref{eqn:clause1}), $I^i_j$ and $\bar{I}^i_j$ are the subsets of indexes for the literals that have been included in the clause. $I^i_j$ contains the indexes of included non-negated inputs, $x_k$, whereas $\bar{I}^i_j$ contains the indexes of included negated inputs, $\neg x_k$. The ``0'' and ``1'' in Eq. (\ref{eqn:clause1}) make sure that $C^i_j(\bold{X})$ also is defined when all the TAs choose to exclude their literals. As can be observed, during training, an ``empty'' clause outputs $1$, while it outputs $0$ during testing (operation).


Multiple TA teams, i.e., clauses, are finally assembled into a complete TM. There are two architectures for clause assembling: Disjunctive Normal Form Architecture and Voting Architecture. In this study, we focus on the latter one, 
as shown in Figure \ref{fig:TMVoting}. For this architecture, the voting consists of summing the output of the clauses:
\begin{equation}
\label{eqn:summation}
f_{\sum}(\mathcal{C}^i(\bold{X}))= \sum^m\limits_{j =1} C_j^i(\bold{X}).
\end{equation}
The output of the TM, in turn, is decided by the unit step function:
\begin{align}
\label{eqn:yivoting}
\hat{y}^i={\begin{cases}\False&{\text{for }}f_{\sum}(\mathcal{C}^i(\bold{X}))<Th\\\True&{\text{for }}f_{\sum}(\mathcal{C}^i(\bold{X}))\geq Th\end{cases}},
\end{align} 
where $Th$ is a predefined threshold for classification. Note that for this architecture, the TM can assign a polarity to each TA team \cite{granmo2018tsetlin}. For example, TA teams with odd indexes get positive polarity, and they vote for class $i$. The remaining TA teams get negative polarity and vote against class $i$. The voting consists of summing the output of the clauses, according to polarity, and the threshold $Th$ is configured as zero. In this study, for ease of analysis, we consider only positive polarity clauses. Nevertheless, this does not change the nature of TM learning (negative polarity clauses simply ``invert" the feedback given to them).


\subsection{The Tsetlin Machine Game for Learning Patterns}\label{sect:TMTraining}
\subsubsection{The Tsetlin Machine Game}\label{sect:TMGame}
The TM trains the TA teams, associated with the clauses, to make the clauses $C^i_j,~j=1,2,...,m$, capture the sub-patterns that characterize the class $i$. 
Data $(\bold{X}=[x_1,x_2,...,x_o],~y^i)$ for training is obtained from a dataset $\mathcal{S}$, distributed according to the probability distribution $P(\bold{X}, y^i)$. The training process is built on letting all the TAs take part in a decentralized game. In the game, each TA is guided by Type I Feedback and Type II Feedback defined in Table \ref{table:type_i_feedback} and Table \ref{table:type_ii_feedback}, respectively. Type I Feedback is triggered when the training sample has a positive label, i.e., $y^i=1$, meaning that the sample belongs to class $i$. When the training sample is labeled as not belonging to class $i$, i.e., $y^i=0$, Type II Feedback is utilized for generating responses. These two types of feedback are designed to reinforce true positive output, i.e., $(\hat{y}^i=1, y^i=1)$ and true negative output, i.e., $(\hat{y}^i=0, y^i=0)$. Simultaneously, they suppress false positive, i.e., $(\hat{y}^i=1, y^i=0)$, and false negative output, i.e., $(\hat{y}^i=0, y^i=1)$.

The formation of patterns is founded on frequent pattern mining. That is, a parameter $s$ controls the granularity of the clauses. A larger $s$ allows more literals to be included in each clause, making the corresponding sub-patterns more fine-grained. A more detailed analysis on parameter $s$ can be found in \cite{zhang2020convergence}.

\begin{table}[h!]
\centering
\begin{tabular}{c|ccccc}
\multicolumn{2}{r|}{{\it Value of the clause} $C^i_j(\bold{X})$ }&\multicolumn{2}{c}{\True}&\multicolumn{2}{c}{\False}\\ 
\multicolumn{2}{r|}{{\it Value of the Literal} $x_k$/$\lnot x_k$}&{\True}&{\False}&{\True}&{\False}\\
\hline
\hline
\multirow{3}{*}{TA Action: \bf Include Literal}&\multicolumn{1}{c|}{$P(\mathrm{Reward})$}&$\frac{s-1}{s}$&NA&$0$&$0$\\
&\multicolumn{1}{c|}{$P(\mathrm{Inaction})$}&$\frac{1}{s}$&NA&$\frac{s-1}{s}$&$\frac{s-1}{s}$\\
&\multicolumn{1}{c|}{$P(\mathrm{Penalty})$}&$0$&NA&$\frac{1}{s} $&$\frac{1}{s}$\\
\hline
\multirow{3}{*}{TA Action: \bf Exclude Literal }&\multicolumn{1}{c|}{$P(\mathrm{Reward})$}&$0$&$\frac{1}{s}$&$\frac{1}{s}$ &$\frac{1}{s}$\\
&\multicolumn{1}{c|}{$P(\mathrm{Inaction})$}&$\frac{1}{s}$&$\frac{s-1}{s}$&$\frac{s-1}{s}$ &$\frac{s-1}{s}$\\
&\multicolumn{1}{c|}{$P(\mathrm{Penalty})$}&$\frac{s-1}{s}$&$0$&$0$&$0$\\
\hline
\end{tabular}
\caption{Type I Feedback --- Feedback upon receiving a sample with label $y=1$, for a single TA to decide whether to Include or Exclude a given literal $x_k/\neg x_k$ into $C^i_j$. NA means not applicable \cite{granmo2018tsetlin}.}
\label{table:type_i_feedback}
\end{table}

\begin{table}[h!]
\centering
\begin{tabular}{c|ccccc}
\multicolumn{2}{r|}{\it Value of the clause $C^i_j(\bold{X})$}&\multicolumn{2}{c}{\True}&\multicolumn{2}{c}{\False}\\ 
\multicolumn{2}{r|}{\it Value of the Literal $x_k/\neg x_k$}&{\True}&{\False}&{\True}&{\False}\\
\hline
\hline
\multirow{3}{*}{TA Action: \bf Include Literal }&\multicolumn{1}{c|}{$P(\mathrm{Reward})$}&$0$&$\mathrm{NA}$&$0$&$0$\\
&\multicolumn{1}{c|}{$P(\mathrm{Inaction})$}&$1.0$&$\mathrm{NA}$&$1.0$&$1.0$\\
&\multicolumn{1}{c|}{$P(\mathrm{Penalty})$}&$0$&$\mathrm{NA}$&$0$&$0$\\
\hline
\multirow{3}{*}{TA Action: \bf Exclude Literal }&\multicolumn{1}{c|}{$P(\mathrm{Reward})$}&$0$&$0$&$0$&$0$\\
&\multicolumn{1}{c|}{$P(\mathrm{Inaction})$}&$1.0$&$0$&$1.0$ &$1.0$\\
&\multicolumn{1}{c|}{$P(\mathrm{Penalty})$}&$0$&$1.0$&$0$&$0$\\
\hline
\end{tabular}
\caption{Type II Feedback --- Feedback upon receiving a sample with label $y=0$, for a single TA to decide whether to Include or Exclude a given literal $x_k/\neg x_k$ into $C^i_j$ \cite{granmo2018tsetlin}.}
\label{table:type_ii_feedback}
\end{table}

To avoid the situation that a majority of the TA teams single in on only a subset of the patterns in the training data, forming an incomplete representation, we use a parameter $T$ as target for the summation $f_{\sum}$. If the votes for a certain sub-pattern accumulate to a total of $T$ or more, neither rewards or penalties are provided to the TAs when more training samples of this sub-pattern are given. In this way, we can ensure that only a few of the available clauses are utilized to capture each specific sub-pattern. In more details, the strategy works in the manner below:   

{\bf Generating Type I Feedback.} If the output from the training sample is $y^i=\True$, we generate \emph{Type I Feedback} for each clause $C^i_j \in \mathcal{C}^i$, where $\mathcal{C}^i$ is the set of clauses that are trained for pattern $i$, however, not every time. Instead, the decision to give feedback to a specific clause is random, according to a feedback probability. The probability of generating Type I Feedback is \cite{granmo2018tsetlin}:
\begin{equation}
u_1=\frac{T - \mathrm{max}(-T, \mathrm{min}(T, f_{\sum}(\mathcal{C}_i)))}{2T}.\label{u1}
\end{equation}


{\bf Generating Type II Feedback.} 
If the output of the training sample is $y^i = \False$, we generate \emph{Type II Feedback} to each clause $C^i_j \in \mathcal{C}^i$, again randomly. The probability of generating Type II Feedback is \cite{granmo2018tsetlin}:
\begin{equation}
u_2=\frac{T + \mathrm{max}(-T, \mathrm{min}(T, f_{\sum}(\mathcal{C}_i)))}{2T}.\label{u2}
\end{equation}


After Type I Feedback or Type II Feedback have been triggered for a clause, the individual TA within each clause is given reward/penalty/inaction according to the probability defined, and then the system is updated. 


\subsubsection{The Training Process in the XOR Case}
\label{sec:XORtraning}
We now introduce the special case of training TMs to capture XOR-patterns. We assume that the training samples shown in Table \ref{xorlogicfull} are provided without noise. In other words, we have $P(y=1|x_1=0, x_2=1)=1$, $P(y=1|x_1=1, x_2=0)=1$, $P(y=0|x_1=0, x_2=0)=1$, and $P(y=0|x_1=1, x_2=1)=1$. We also assume that $P(x_1=0, x_2=1)>0$, $P(x_1=1, x_2=1)>0$, $P(x_1=0, x_2=0)>0$, and $P(x_1=1, x_2=0)>0$. Clearly $P(x_1=0, x_2=1)$+$P(x_1=1, x_2=1)$+$P(x_1=0, x_2=0)$+$P(x_1=1, x_2=0)$ $=1$.
This guarantees that all types of possible input-output pairs will appear in the training samples. The aim is to show that after training, the TM can output $1$ for inputs $x_1=1, x_2=0$ or $x_1=0, x_2=1$, and $0$ otherwise. 

\begin{table}
\centering
\begin{tabular}{ |c|c|c| } 
\hline
$x_1$ & $x_2$ & y \\ 
0 & 0 & 0 \\ 
1 & 1 & 0 \\ 
0 & 1 & 1 \\
1& 0&1\\
\hline
\end{tabular}
\caption{The ``XOR'' logic.}
\label{xorlogicfull}
\end{table}

The above XOR-scenario leads to the following TM training process, described step-by-step:

\begin{enumerate} 
\item We initialize the TAs by assigning each of them a random state among the states associated with action Exclude.

\item \label{step:Clause} We obtain a new training sample $(x_1, x_2, y)$ and calculate the value of each single clause $C^i_j$ according to Eq. (\ref{eqn:clause1}).

\item The TA states for each clause are updated based on: (i) the label $y$; (ii) the clause value $C^i_j$; (iii) the value of each individual literal ($x_1$, $\neg x_1$, $x_2$, $\neg x_2$,); and (iv) the sum of the clause outputs for the class $i$, $f_{\sum}(\mathcal{C}^i(\bold{X}))$. Finally,  for each clause, the the states of the associated TAs are updated according to Table \ref{table:type_i_feedback} with  probability $u_1$ when $y=1$. If $y=0$, the TAs are updated according to Table \ref{table:type_ii_feedback}, with probability $u_2$. 


\item Repeat from Step \ref{step:Clause} until a given stopping criteria is met.

\end{enumerate}
Note that in the XOR case, there is only one class to be learnt, which is the XOR-relation. We therefore ignore the class index, i.e., $i$, in notation $C^i_j$ and $\mathrm{TA}^{i,j}_{2k}$ in the remainder of the paper.

\section{Proof of the Convergence for the XOR Operator}\label{XORstudy}

The XOR-relation is nonlinear and the inputs for the two to-be-learnt sub-patterns ($x_1=1, x_2=0$ or $x_1=0, x_2=1$) are the bit-wise inversion of the other. Therefore, the XOR-relation is challenging or even impossible to learn for many machine learning algorithms. In what follows, we will reveal, step-by-step, the convergence property of the TM for the XOR-relation. First, in Subsection \ref{simplecase}, we start from a special and simple case to show that there exists a TM configuration that can learn the XOR-relation. This establishes that TMs have the ability to learn such a relation. Thereafter, in Subsection \ref{complicated}, we analyze how a general TM can learn the XOR-relation, including the criteria for learning. Through these analyses, the dynamics of the learning process of the TAs, operating within the TM clauses, are elaborated. In particular, we investigate the self-organizing collaboration that happens among the clauses, to cast light on how a TM learns multiple patterns.

\subsection{The Simplest Structure for the XOR-relation}\label{simplecase}
\begin{mytheorem}
There exists a TM structure that can converge almost surely to the XOR-relation under an infinite time horizon.\label{theorem1}
\end{mytheorem}
\begin{proof}
To prove Theorem \ref{theorem1}, we use a TM with two clauses, $C_1$ and $C_2$. In $C_1$, there are four literals, i.e., $x_1$, $\neg x_1$, $x_2$, and $\neg x_2$, each of which corresponds to a TA, namely, $\mathrm{TA}_1^1$, $\mathrm{TA}_2^1$, $\mathrm{TA}_3^1$, and $\mathrm{TA}_4^1$. Similarly, in $C_2$, there are also four literals, i.e., $x_1$, $\neg x_1$, $x_2$, and $\neg x_2$, each of which corresponds to four other TA, namely, $\mathrm{TA}_1^2$, $\mathrm{TA}_2^2$, $\mathrm{TA}_3^2$, and $\mathrm{TA}_4^2$. Clearly, there are in total 8 TAs in the system. Considering the simplest structure for TA, we provide each TA with only two states, as shown in Figure \ref{figure:TAarchitecture_basic}. 

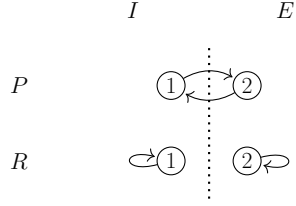
\begin{figure}[!h]
\centering
\begin{tikzpicture}[node distance = .35cm, font=\Huge]
\tikzstyle{every node}=[scale=0.35]
\node[state] (F) at (2,1) {1};
\node[state] (G) at (3,1) {2};

\node[state] (B) at (2,2) {1};
\node[state] (C) at (3,2) {2};
\node[thick] at (0,1) {$R$};
\node[thick] at (0,2) {$P$};
\node[thick] at (1.5,3) {$I$};
\node[thick] at (3.5,3) {$E$};
\draw[dotted, thick] (2.5,0.5) -- (2.5,2.5);
\draw[every loop]
(G) edge[loop right] node [scale=1.2, below=0.1 of G] {} (G);
\draw[every loop]
(B) edge[bend left] node [scale=1.2, above=0.1 of C] {} (C);
\draw[every loop]
(C) edge[bend left ] node [scale=1.2, below=0.1 of B] {} (B);
\draw[every loop]
(F) edge[loop left] node [scale=1.2, below=0.1 of F] {} (F);
\end{tikzpicture}
\caption{A simple TA with two states. In this figure, ``$P$'', ``$R$'', ``$I$'', and ``$E$'' means ``penalty", ``reward", ``include" and ``exclude" respectively. }
\label{figure:TAarchitecture_basic}
\end{figure}

The behavior of the above depicted TM can be modeled using a discrete time Markov chain (DTMC) with 8 elements, each of which represents the status of the corresponding TA. In more details, any state of the DTMC is represented by $\xv=(h_1, h_2, h_3, \ldots, h_8)$, where $h_i\in\{0,1\}$, $i\in\{1,\ldots,8\}$. Here, $h_1, h_2, h_3, \ldots, h_8$ correspond to $\mathrm{TA}_1^1$, $\mathrm{TA}_2^1$, $\mathrm{TA}_3^1$, $\mathrm{TA}_4^1$ $\mathrm{TA}_1^2$, $\mathrm{TA}_2^2$, $\mathrm{TA}_3^2$, and $\mathrm{TA}_4^2$. For example, $h_1$ represents the state for $\mathrm{TA}_1^1$, with state $0$ referring to ``Exclude'' and state $1$ referring to ``Include''. This is also how the other $h_i$ are organized. 
Clearly, the state space of the DTMC, $\stS$, includes $2^8=256$ states. If the system can capture the XOR-relation after training, the DTMC must have and only have two possible absorbing states, i.e., $(1,0,0,1,0,1,1,0)$ for $C_1=x_1 \wedge\neg x_2$ and $C_2=\neg x_1 \wedge x_2$, and $(0,1,1,0,1,0,0,1)$ for $C_1=\neg x_1 \wedge x_2$ and $C_2= x_1 \wedge \neg x_2$. Let us index the states from $(0,0,0,0,0,0,0,0)$ to $(1,1,1,1,1,1,1,1)$ as $1$ to $256$. Then the states  $(0,1,1,0,1,0,0,1)$ and $(1,0,0,1,0,1,1,0)$ correspond to the $106^{th}$ and the $151^{st}$ state, respectively. 


To determine whether the two states are absorbing, we can observe the transition matrix of the DTMC and see if there are any out going transitions from those two states. This can be easily checked and confirmed. To demonstrate that these two states are the only absorbing states, we also need to show that all the other states are recurrent. To demonstrate this point in a simple way, we calculate the limiting matrix of the DTMC. In more details, we first compose the transition matrix of the DTMC, $\boldsymbol{P}$, and then find the limiting matrix $\boldsymbol{A}=\boldsymbol{P}^\infty$. If the matrix $\boldsymbol{A}$ possesses the below properties, we can conclude that state $106$ and state $151$ are the only absorbing states. This, in turn, means that after infinite training samples, the system will learn the XOR-relation with probability 1. The properties are as follows:
\begin{itemize}
\item The transition probability from the $106^{th}$ state to the $106^{th}$ state is 1, and the same applies to the $151^{st}$ state. 
\item The transition probabilities from any state other than the two absorbing ones to the two absorbing ones sum to 1.
\item The transition probabilities from any state other than the two absorbing ones to a non-absorbing state are all zeros. 
\end{itemize}
The matrix $\boldsymbol{A}$ represents the probability of arriving at a destination state from any starting state after infinite time steps. 
The first bullet point shows that the $106^{th}$ and $151^{st}$ elements are indeed the absorbing states. This is because each of these states  returns to itself with probability 1. Similarly, the second and the third bullet points indicate that the other states are not absorbing states because starting from any other state, the system will end up in one of the absorbing states. 


In principle, $\boldsymbol{P}$ must be multiplied with itself an infinite number of times. In practice, however, we multiply $\boldsymbol{P}$ with itself a sufficiently large number of times, until the entries in $\boldsymbol{P}$ do not change. 

To validate the convergence, we use the hyper-parameters $s=10$ and $T=1$ as an example and use Algorithm \ref{alg:transition} in Appendix \ref{calculationDTMC} for the calculation\footnote{The Python code for the Algorithm can be obtained from \url{https://github.com/cair/TM-XOR-proof}.}. In this example, we assume the training samples (1,1,0), (1,0,1) (0,1,1) and (1,1,0) appear with the same probability, i.e., 25\% of the time each. From running the algorithm, we conclude that the $106^{th}$ and the $151^{st}$ are indeed the only absorbing states of the DTMC, which confirms that even the simplest configuration of the TM can converge almost surely to the XOR-relation. 
\end{proof}

\subsection{Structures with More Than Two TA States and/or More Than Two Clauses}\label{complicated}

Clearly, Theorem \ref {theorem1} confirms that TMs are capable of learning the XOR-relation. In the following, we study the cases where there are more than two clauses and/or more than two TA states. The purpose is to uncover how the XOR-relation is learnt by the TM in general. This also allows us to demonstrate the role that the $T$ hyper-parameter plays during learning. We look in particular at how the hyper-parameter governs reinforcement of the different clauses to learn the distinct sub-patterns associated with the XOR-relation. 

The flow of the analysis is given via the lemmas and theorem below:
\begin{mylemma} \label{full1} Any clause will converge almost surely to $\neg x_{1} \wedge x_{2}$ given the training samples indicated in Table \ref{xorlogichalf} in infinite time when $u_1>0$ and $u_2>0$. 
\end{mylemma}
\begin{mylemma} \label{full2} Any clause will converge almost surely to $ x_{1} \wedge \neg x_{2}$ given the training samples indicated in Table \ref{xorlogichalf1} in infinite time when $u_1>0$ and $u_2>0$. 
\end{mylemma}
\begin{mylemma} \label{full3}
The system for any clause is recurrent given the input and output pair indicated in Table \ref{xorlogicfull} for $u_1>0$ and $u_2>0$. 
\end{mylemma}
\begin{mylemma} \label{full4}
Given a number of clauses $m$ and a threshold value $T$, $T<m$, the event that the sum of the clause outputs, i.e., $f_{\sum}(\mathcal{C}_i)$, reaches $T$ appears almost surely in infinite time. 
\end{mylemma}
\begin{mylemma} \label{full5}
When the number of clauses that follow the same sub-pattern reaches $T$, other clauses will not see the input training samples from this particular sub-pattern.
\end{mylemma} 
\begin{mytheorem} \label{full6} The clauses can almost surely learn the sub-patterns of XOR in infinite time, when $T\leq m/2$. 
\end{mytheorem}

\begin{table}
\centering
\begin{tabular}{ |c|c|c| } 
\hline
$x_1$ & $x_2$ & Output \\ 
0 & 0 & 0 \\ 
1 & 1 & 0 \\ 
0 & 1 & 1 \\
\hline
\end{tabular}
\caption{A sub-pattern in ``XOR'' case.}
\label{xorlogichalf}
\end{table}

\begin{table}
\centering
\begin{tabular}{ |c|c|c| } 
\hline
$x_1$ & $x_2$ & Output \\ 
0 & 0 & 0 \\ 
1 & 1 & 0 \\ 
1 & 0 & 1 \\
\hline
\end{tabular}
\caption{A sub-pattern in ``XOR'' case.}
\label{xorlogichalf1}
\end{table}

The logical flow of the lemmas and the theorem is as follows. Lemma \ref{full1} and Lemma \ref{full2} confirm the fact that the TM can learn the intended sub-pattern if only one of the XOR sub-patterns appear in the training data. We assume non-negative $u_1$ and $u_2$ to guarantee that the training samples always trigger the feedback shown in Table \ref{table:type_i_feedback} or Table \ref{table:type_ii_feedback}. Note that these lemmas determine that the correct states are absorbing states as well as the uniqueness of these states.   Lemma~\ref{full3} establishes the fact that when both sub-patterns are present in the training samples, a TM will not converge to any one of these in probability 1, if both $u_1$ and $u_2$ are kept positive. Therefore, it is necessary to guide the convergence, which is done by the hyper-parameter $T$. We use $T$ to modify the probability of triggering the feedback events in Table \ref{table:type_i_feedback} or Table \ref{table:type_ii_feedback}. Lemma \ref{full4} then establishes that the number of clauses that learn a certain sub-pattern will reach $T$ at a certain time instant. Lemma \ref{full5} guarantees that when the event described by Lemma \ref{full4} happens, the corresponding training samples of the learnt sub-pattern will be blocked from the system. Accordingly, Lemmas \ref{full1}-\ref{full3} cover the system dynamics when $T$ (and thus $u_1$, $u_2$) is not involved in the learning, while Lemmas \ref{full4} and \ref{full5} shows how $T$ blocks the training samples of a learnt sub-pattern to make the system learn another sub-pattern. Based on Lemmas \ref{full1}-\ref{full5}, we prove that the TM can learn the XOR-relation and the conditions for learning, in terms of Theorem \ref{full6}. In the remaining subsections, we will prove the Lemmas one by one. 

\subsubsection{Proof of Lemma \ref{full1} and Lemma \ref{full2}}

Now let's study Lemma \ref{full1}. Here, we will confirm that the clauses in the TM will almost surely converge to the clause $\neg x_{1} \wedge x_{2}$ when the training samples shown in Table \ref{xorlogichalf} are given to the TM. Note that the functionality of $T$ is disabled in this lemma, and $u_1$ and $u_2$ are assumed to be positive constants.  


\begin{proof}
Without loss of generality, we study clause $C_3$, which has
$\mathrm{TA}^3_{1}$ with actions ``Include" $x_{1}$ or ``Exclude" it, $\mathrm{TA}^3_{2}$ with actions ``Include" $\neg x_{1}$ or ``Exclude" it, $\mathrm{TA}^3_{3} $ with actions ``Include" $x_{2}$ or ``Exclude" it, and $\mathrm{TA}^3_{4}$ with actions ``Include" $\neg x_{2}$ or ``Exclude" it. To analyze the convergence of those four TAs, we perform a quasi-stationary analysis, where we freeze the behavior of three of them, and then study the transitions of the remaining one. More specifically, the analysis is organized as follows:
\begin{enumerate}
\item We freeze $\mathrm{TA}^3_{1}$ and $\mathrm{TA}^3_{2}$ respectively at ``Exclude" and ``Include". In this case, the first bit becomes $\neg x_{1}$. There are four sub-cases for $\mathrm{TA}^3_{3}$ and $\mathrm{TA}^3_{4}$: 
\begin{enumerate}
\item We study the transition of $\mathrm{TA}^3_{3}$ when it has the action ``Include" as its current action, given different training samples shown in Table \ref{xorlogichalf} and different actions of $\mathrm{TA}^3_{4}$ (i.e., when the action of $\mathrm{TA}^3_{4}$ is frozen at ``Include" or ``Exclude".). \label {subcase1}
\item We study the transition of $\mathrm{TA}^3_{3}$ when it has ``Exclude" as its current action, given different training samples shown in Table \ref{xorlogichalf} and different actions of $\mathrm{TA}^3_{4}$ (i.e., when the action of $\mathrm{TA}^3_{4}$ is frozen at ``Include" or ``Exclude".). \label {subcase2}
\item We study the transition of $\mathrm{TA}^3_{4}$ when it has ``Include" as its current action, given different training samples shown in Table \ref{xorlogichalf} and different actions of $\mathrm{TA}^3_{3}$ (i.e., when the action of $\mathrm{TA}^3_{3}$ is frozen at ``Include" or ``Exclude".). \label {subcase3}
\item We study the transition of $\mathrm{TA}^3_{4}$ when it has ``Exclude" as its current action, given different training samples shown in Table \ref{xorlogichalf} and different actions of $\mathrm{TA}^3_{3}$ (i.e., when the action of $\mathrm{TA}^3_{3}$ is frozen as ``Include" or ``Exclude".). \label {subcase4}
\end{enumerate}
\item We freeze $\mathrm{TA}^3_{1}$ and $\mathrm{TA}^3_{2}$ respectively at ``Include" and ``Exclude". In this case, the first bit becomes  $x_{1}$. The sub-cases for $\mathrm{TA}^3_{3}$ and $\mathrm{TA}^3_{4}$ are identical to the sub-cases in the previous case. 

\item We freeze $\mathrm{TA}^3_{1}$ and $\mathrm{TA}^3_{2}$ at ``Exclude" and ``Exclude". In this case, the first bit is excluded and will not influence the final output. 
The sub-cases for $\mathrm{TA}^3_{3}$ and $\mathrm{TA}^3_{4}$ are identical to the sub-cases in the previous case. 
\item We freeze $\mathrm{TA}^3_{1}$ and $\mathrm{TA}^3_{2}$ at ``Include" and ``Include". In this case, we always have $C_3=0$ because the clause contains the contradiction $x_1 \land \lnot x_1$. The sub-cases for $\mathrm{TA}^3_{3}$ and $\mathrm{TA}^3_{4}$ are identical to the sub-cases in the previous case. 
\end{enumerate} 
In the analysis below, we will study each of the four cases, one by one.\\

\noindent {\bf Case 1} \\
We now analyze the first sub-case, i.e., Sub-case 1 (a). In this case, $\neg x_1$ is always included. We here study the transition of $\mathrm{TA}^3_{3}$ when its current action is ``Include''. Depending on different training samples and actions of $\mathrm{TA}^3_{4}$, we have the following possible transitions. Below, ``I'' and ``E'' mean ``Include'' and ``Exclude'', respectively.

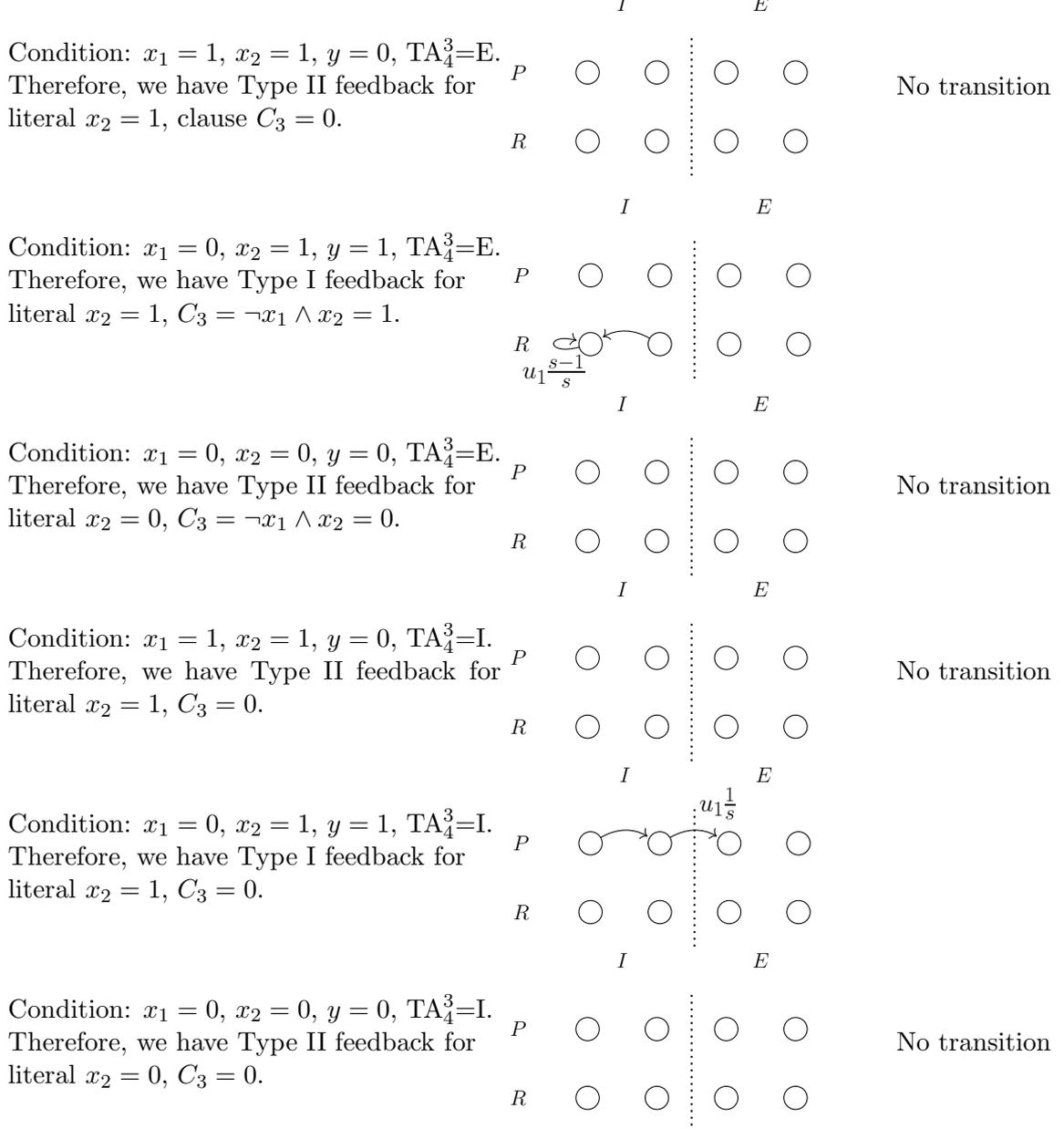
\begin{minipage}{0.45\textwidth}
Condition: $x_{1}=1$, $x_{2}=1$, $y=0$, $\mathrm{TA}^3_{4}$=E.\\
Therefore, we have Type II feedback for \\
literal $x_{2}=1$, clause $C_{3}=0$.
\end{minipage}
\begin{minipage}{0.35\textwidth}
\begin{tikzpicture}[node distance = .35cm, font=\Huge]
\tikzstyle{every node}=[scale=0.35]
\node[state] (E) at (1,1) {};
\node[state] (F) at (2,1) {};
\node[state] (G) at (3,1) {};
\node[state] (H) at (4,1) {};
\node[state] (A) at (1,2) {};
\node[state] (B) at (2,2) {};
\node[state] (C) at (3,2) {};
\node[state] (D) at (4,2) {};
\node[thick] at (0,1) {$R$};
\node[thick] at (0,2) {$P$};
\node[thick] at (1.5,3) {$I$};
\node[thick] at (3.5,3) {$E$};
\draw[dotted, thick] (2.5,0.5) -- (2.5,2.5);

\end{tikzpicture}
\end{minipage}
\begin{minipage}{.35\textwidth}
No transition
\end{minipage}

\vspace{.25cm}

\begin{minipage}{0.45\textwidth}
Condition: $x_{1}=0$, $x_{2}=1$, $y=1$, $\mathrm{TA}^3_{4}$=E.\\
Therefore, we have Type I feedback for\\
literal $x_{2}=1$, $C_{3}= \neg x_{1} \wedge x_{2} = 1$.\\ 
\end{minipage}
\begin{minipage}{0.35\textwidth}
\begin{tikzpicture}[node distance = .35cm, font=\Huge]
\tikzstyle{every node}=[scale=0.35]
\node[state] (E) at (1,1) {};
\node[state] (F) at (2,1) {};
\node[state] (G) at (3,1) {};
\node[state] (H) at (4,1) {};
\node[state] (A) at (1,2) {};
\node[state] (B) at (2,2) {};
\node[state] (C) at (3,2) {};
\node[state] (D) at (4,2) {};
\node[thick] at (0,1) {$R$};
\node[thick] at (0,2) {$P$};
\node[thick] at (1.5,3) {$I$};
\node[thick] at (3.5,3) {$E$};
\draw[dotted, thick] (2.5,0.5) -- (2.5,2.5);
\draw[every loop]
(F) edge[bend right] node [scale=1.2, above=0.1 of C] {} (E)
(E) edge[loop left = 45] node [scale=1.2, below=0.1 of E] {$u_1\frac{s-1}{s}$} (E);

\end{tikzpicture}
\end{minipage}

\begin{minipage}{0.45\textwidth}
Condition: $x_{1}=0$, $x_{2}=0$, $y=0$, $\mathrm{TA}^3_{4}$=E.\\
Therefore, we have Type II feedback for \\ 
literal $x_{2}=0$, $C_{3}= \neg x_{1} \wedge x_{2}=0$.
\end{minipage}
\begin{minipage}{0.35\textwidth}
\begin{tikzpicture}[node distance = .35cm, font=\Huge]
\tikzstyle{every node}=[scale=0.35]
\node[state] (E) at (1,1) {};
\node[state] (F) at (2,1) {};
\node[state] (G) at (3,1) {};
\node[state] (H) at (4,1) {};
\node[state] (A) at (1,2) {};
\node[state] (B) at (2,2) {};
\node[state] (C) at (3,2) {};
\node[state] (D) at (4,2) {};
\node[thick] at (0,1) {$R$};
\node[thick] at (0,2) {$P$};
\node[thick] at (1.5,3) {$I$};
\node[thick] at (3.5,3) {$E$};
\draw[dotted, thick] (2.5,0.5) -- (2.5,2.5);

\end{tikzpicture}
\end{minipage}
\begin{minipage}{.35\textwidth}
No transition
\end{minipage}

\begin{minipage}{0.45\textwidth}
Condition: $x_{1}=1$, $x_{2}=1$, $y=0$, $\mathrm{TA}^3_{4}$=I.\\
Therefore, we have Type II feedback for literal $x_{2}=1$, $C_{3}=0$.
\end{minipage}
\begin{minipage}{0.35\textwidth}
\begin{tikzpicture}[node distance = .35cm, font=\Huge]
\tikzstyle{every node}=[scale=0.35]
\node[state] (E) at (1,1) {};
\node[state] (F) at (2,1) {};
\node[state] (G) at (3,1) {};
\node[state] (H) at (4,1) {};
\node[state] (A) at (1,2) {};
\node[state] (B) at (2,2) {};
\node[state] (C) at (3,2) {};
\node[state] (D) at (4,2) {};
\node[thick] at (0,1) {$R$};
\node[thick] at (0,2) {$P$};
\node[thick] at (1.5,3) {$I$};
\node[thick] at (3.5,3) {$E$};
\draw[dotted, thick] (2.5,0.5) -- (2.5,2.5);

\end{tikzpicture}
\end{minipage}
\begin{minipage}{.35\textwidth}
No transition
\end{minipage}

\begin{minipage}{0.45\textwidth}
Condition: $x_{1}=0$, $x_{2}=1$, $y=1$, $\mathrm{TA}^3_{4}$=I.\\
Therefore, we have Type I feedback for \\
literal $x_{2}=1$, $C_{3}=0$.
\end{minipage}
\begin{minipage}{0.35\textwidth}
\begin{tikzpicture}[node distance = .35cm, font=\Huge]
\tikzstyle{every node}=[scale=0.35]
\node[state] (E) at (1,1) {};
\node[state] (F) at (2,1) {};
\node[state] (G) at (3,1) {};
\node[state] (H) at (4,1) {};
\node[state] (A) at (1,2) {};
\node[state] (B) at (2,2) {};
\node[state] (C) at (3,2) {};
\node[state] (D) at (4,2) {};
\node[thick] at (0,1) {$R$};
\node[thick] at (0,2) {$P$};
\node[thick] at (1.5,3) {$I$};
\node[thick] at (3.5,3) {$E$};
\draw[dotted, thick] (2.5,0.5) -- (2.5,2.5);
\draw[every loop]
(A) edge[bend left] node [scale=1.2, above=0.1 of C] {} (B)
(B) edge[bend left] node [scale=1.2, above=0.1 of B] {$~~~~~~u_1\frac{1}{s}$} (C);

\end{tikzpicture}
\end{minipage}

\begin{minipage}{0.45\textwidth}
Condition: $x_{1}=0$, $x_{2}=0$, $y=0$, $\mathrm{TA}^3_{4}$=I.\\
Therefore, we have Type II feedback for \\
literal $x_2=0$, $C_{3}=0$.
\end{minipage}
\begin{minipage}{0.35\textwidth}
\begin{tikzpicture}[node distance = .35cm, font=\Huge]
\tikzstyle{every node}=[scale=0.35]
\node[state] (E) at (1,1) {};
\node[state] (F) at (2,1) {};
\node[state] (G) at (3,1) {};
\node[state] (H) at (4,1) {};
\node[state] (A) at (1,2) {};
\node[state] (B) at (2,2) {};
\node[state] (C) at (3,2) {};
\node[state] (D) at (4,2) {};
\node[thick] at (0,1) {$R$};
\node[thick] at (0,2) {$P$};
\node[thick] at (1.5,3) {$I$};
\node[thick] at (3.5,3) {$E$};
\draw[dotted, thick] (2.5,0.5) -- (2.5,2.5);

\end{tikzpicture}
\end{minipage}
\begin{minipage}{.35\textwidth}
No transition
\end{minipage}

Clearly, the above analyzed sub-case has 6 instances, depending on the variations of the training samples and the status of $\mathrm{TA}^3_4$, where the first three correspond to the instances where $\mathrm{TA}^3_4=E$ while the last three represent the instances where $\mathrm{TA}^3_4=I$. We now investigate the first instance, which covers the training samples: $x_{1}=1$, $x_{2}=1$, $y=0$, and $T^3_{4}$=E. Clearly, the training sample will trigger Type II feedback because of $y=0$. Then the clause becomes $C_{3}= \neg x_{1} \wedge x_{2} = 0$ because the studied instance has $\mathrm{TA}^3_1=E$, $\mathrm{TA}^3_2=I$, $\mathrm{TA}^3_3=I$, and $\mathrm{TA}^3_4=E$. Because we now study $\mathrm{TA}^3_3$, the corresponding literal is $x_2=1$. Based on the above information, we can check from Table \ref{table:type_ii_feedback} that the probability of ``Inaction'' is 1. Therefore, the transition diagram does not have any arrow, indicating that there is ``No transition'' for $\mathrm{TA}^3_3$ in this circumstance.

To study a circumstance where transitions may happen, let us look at the second instance in the analyzed sub-case, with $x_{1}=0$, $x_{2}=1$, $y=1$, and $T^3_{4}$=E. Clearly, this training sample will trigger Type I feedback as $y=1$. Together with the current status of other TAs, the clause is determined to be $C_{3}= \neg x_{1} \wedge x_{2} = 1$ and the literal is $x_2=1$. From Table \ref{table:type_i_feedback}, we know that the reward probability is $\frac{s-1}{s}$ and the inaction probability is $1/s$. To indicate the transitions, we have plotted 
the diagram showing the reward probability. Note that the overall probability is $u_1\frac{s-1}{s}$, where $u_1$ is defined by Eq. (\ref{u1}). Understandably, $u_1\in[0,0.5]$ and $u_2\in[0.5,1]$ for any $f_{\sum}(\mathcal{C}^i)\geq 0$. For now, we assume we find a certain $T$ such that $u_1>0$ holds. The role of $T$ and $u_1$ will be analyzed later.



We now consider Sub-case 1 (b). The literal $\neg x_1$ is still included, and we study the transition of $\mathrm{TA}^3_{3}$ when its current action is ``Exclude''. The possible transitions are listed below.


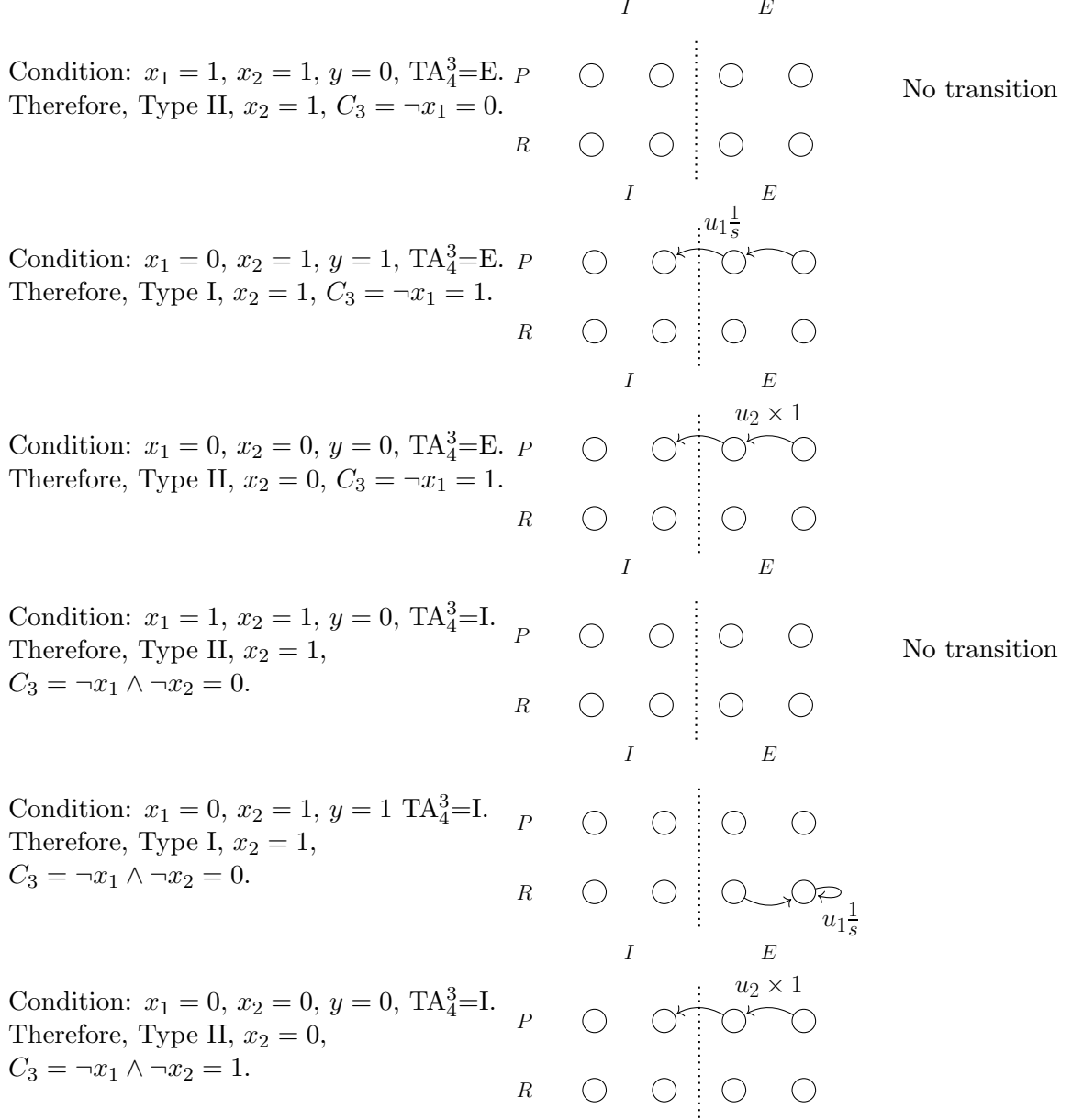
\begin{minipage}{0.45\textwidth}
Condition: $x_{1}=1$, $x_{2}=1$, $y=0$, $\mathrm{TA}^3_4$=E. \\
Therefore, Type II, $x_{2}=1$, $C_{3}=\neg x_{1}=0$.
\end{minipage}
\begin{minipage}{0.35\textwidth}
\begin{tikzpicture}[node distance = .35cm, font=\Huge]
\tikzstyle{every node}=[scale=0.35]
\node[state] (E) at (1,1) {};
\node[state] (F) at (2,1) {};
\node[state] (G) at (3,1) {};
\node[state] (H) at (4,1) {};
\node[state] (A) at (1,2) {};
\node[state] (B) at (2,2) {};
\node[state] (C) at (3,2) {};
\node[state] (D) at (4,2) {};
\node[thick] at (0,1) {$R$};
\node[thick] at (0,2) {$P$};
\node[thick] at (1.5,3) {$I$};
\node[thick] at (3.5,3) {$E$};
\draw[dotted, thick] (2.5,0.5) -- (2.5,2.5);

\end{tikzpicture}
\end{minipage}
\begin{minipage}{.35\textwidth}
No transition
\end{minipage}

\begin{minipage}{0.45\textwidth}
Condition: $x_{1}=0$, $x_{2}=1$, $y=1$, $\mathrm{TA}^3_{4}$=E.\\
Therefore, Type I, $x_{2}=1$, $C_{3}=\neg x_{1}=1$. 
\end{minipage}
\begin{minipage}{0.35\textwidth}
\begin{tikzpicture}[node distance = .35cm, font=\Huge]
\tikzstyle{every node}=[scale=0.35]
\node[state] (E) at (1,1) {};
\node[state] (F) at (2,1) {};
\node[state] (G) at (3,1) {};
\node[state] (H) at (4,1) {};
\node[state] (A) at (1,2) {};
\node[state] (B) at (2,2) {};
\node[state] (C) at (3,2) {};
\node[state] (D) at (4,2) {};
\node[thick] at (0,1) {$R$};
\node[thick] at (0,2) {$P$};
\node[thick] at (1.5,3) {$I$};
\node[thick] at (3.5,3) {$E$};
\draw[dotted, thick] (2.5,0.5) -- (2.5,2.5);
\draw[every loop]
(D) edge[bend right] node [scale=1.2, above=0.1 of C] {} (C)
(C) edge[bend right] node [scale=1.2, above=0.1 of B] {$~~~~~~u_1\frac{1}{s}$} (B);

\end{tikzpicture}
\end{minipage}

\begin{minipage}{0.45\textwidth}
Condition: $x_{1}=0$, $x_{2}=0$, $y=0$, $\mathrm{TA}^3_{4}$=E.\\
Therefore, Type II, $x_{2}=0$, $C_{3}=\neg x_{1}=1$. 
\end{minipage}
\begin{minipage}{0.35\textwidth}
\begin{tikzpicture}[node distance = .35cm, font=\Huge]
\tikzstyle{every node}=[scale=0.35]
\node[state] (E) at (1,1) {};
\node[state] (F) at (2,1) {};
\node[state] (G) at (3,1) {};
\node[state] (H) at (4,1) {};
\node[state] (A) at (1,2) {};
\node[state] (B) at (2,2) {};
\node[state] (C) at (3,2) {};
\node[state] (D) at (4,2) {};
\node[thick] at (0,1) {$R$};
\node[thick] at (0,2) {$P$};
\node[thick] at (1.5,3) {$I$};
\node[thick] at (3.5,3) {$E$};
\draw[dotted, thick] (2.5,0.5) -- (2.5,2.5);
\draw[every loop]
(D) edge[bend right] node [scale=1.2, above=0.1 of C] {$u_2\times 1$} (C)
(C) edge[bend right] node [scale=1.2, above=0.1 of B] {} (B);

\end{tikzpicture}
\end{minipage}

\begin{minipage}{0.45\textwidth}
Condition: $x_{1}=1$, $x_{2}=1$, $y=0$, $\mathrm{TA}^3_{4}$=I. \\
Therefore, Type II, $x_{2}=1$, \\
$C_{3}=\neg x_{1} \wedge \neg x_{2}=0$. 
\end{minipage}
\begin{minipage}{0.35\textwidth}
\begin{tikzpicture}[node distance = .35cm, font=\Huge]
\tikzstyle{every node}=[scale=0.35]
\node[state] (E) at (1,1) {};
\node[state] (F) at (2,1) {};
\node[state] (G) at (3,1) {};
\node[state] (H) at (4,1) {};
\node[state] (A) at (1,2) {};
\node[state] (B) at (2,2) {};
\node[state] (C) at (3,2) {};
\node[state] (D) at (4,2) {};
\node[thick] at (0,1) {$R$};
\node[thick] at (0,2) {$P$};
\node[thick] at (1.5,3) {$I$};
\node[thick] at (3.5,3) {$E$};
\draw[dotted, thick] (2.5,0.5) -- (2.5,2.5);

\end{tikzpicture}
\end{minipage}
\begin{minipage}{.35\textwidth}
No transition
\end{minipage}

\begin{minipage}{0.45\textwidth}
Condition: $x_{1}=0$, $x_{2}=1$, $y=1$ $\mathrm{TA}^3_{4}$=I.\\
Therefore, Type I, $x_{2}=1$, \\
$C_{3}=\neg x_{1} \wedge \neg x_{2}=0$.
\end{minipage}
\begin{minipage}{0.35\textwidth}
\begin{tikzpicture}[node distance = .35cm, font=\Huge]
\tikzstyle{every node}=[scale=0.35]
\node[state] (E) at (1,1) {};
\node[state] (F) at (2,1) {};
\node[state] (G) at (3,1) {};
\node[state] (H) at (4,1) {};
\node[state] (A) at (1,2) {};
\node[state] (B) at (2,2) {};
\node[state] (C) at (3,2) {};
\node[state] (D) at (4,2) {};
\node[thick] at (0,1) {$R$};
\node[thick] at (0,2) {$P$};
\node[thick] at (1.5,3) {$I$};
\node[thick] at (3.5,3) {$E$};
\draw[dotted, thick] (2.5,0.5) -- (2.5,2.5);
\draw[every loop]
(G) edge[bend right] node [scale=1.2, above=0.1 of C] {} (H)
(H) edge[loop right] node [scale=1.2, below=0.1 of H] {$u_1\frac{1}{s}$} (H);

\end{tikzpicture}
\end{minipage}

\begin{minipage}{0.45\textwidth}
Condition: $x_{1}=0$, $x_{2}=0$, $y=0$, $\mathrm{TA}^3_4$=I.\\
Therefore, Type II, $x_{2}=0$, \\
$C_{3}=\neg x_{1} \wedge \neg x_{2}=1$.
\end{minipage}
\begin{minipage}{0.35\textwidth}
\begin{tikzpicture}[node distance = .35cm, font=\Huge]
\tikzstyle{every node}=[scale=0.35]
\node[state] (E) at (1,1) {};
\node[state] (F) at (2,1) {};
\node[state] (G) at (3,1) {};
\node[state] (H) at (4,1) {};
\node[state] (A) at (1,2) {};
\node[state] (B) at (2,2) {};
\node[state] (C) at (3,2) {};
\node[state] (D) at (4,2) {};
\node[thick] at (0,1) {$R$};
\node[thick] at (0,2) {$P$};
\node[thick] at (1.5,3) {$I$};
\node[thick] at (3.5,3) {$E$};
\draw[dotted, thick] (2.5,0.5) -- (2.5,2.5);
\draw[every loop]
(D) edge[bend right] node [scale=1.2, above=0.1 of C] {$u_2\times 1$} (C)
(C) edge[bend right] node [scale=1.2, above=0.1 of B] {$$} (B);

\end{tikzpicture}
\end{minipage}

Now let us move onto the third sub-case in Case 1, i.e., Sub-case 1 (c). The literal $\neg x_1$ is still included, and we study the transition of $\mathrm{TA}^3_{4}$ when its current action is ``Include''. To save space, we remove the instances where no transition happens in the remainder of the paper. Note that we are now studying $\mathrm{TA}^3_4$ that corresponds to $\neg x_{2}$ rather than $x_{2}$. Therefore, the literal in Tables \ref{table:type_i_feedback} and \ref{table:type_ii_feedback} becomes $\neg x_{2}$. 


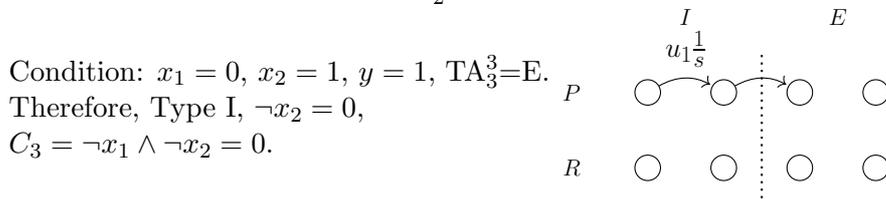
\begin{minipage}{0.45\textwidth}
Condition: $x_{1}=0$, $x_{2}=1$, $y=1$, $\mathrm{TA}^3_{3}$=E. \\
Therefore, Type I, $\neg x_{2}=0$,\\
 $C_{3}=\neg x_{1} \wedge \neg x_{2}=0$.
\end{minipage}
\begin{minipage}{0.35\textwidth}
\begin{tikzpicture}[node distance = .35cm, font=\Huge]
\tikzstyle{every node}=[scale=0.35]
\node[state] (E) at (1,1) {};
\node[state] (F) at (2,1) {};
\node[state] (G) at (3,1) {};
\node[state] (H) at (4,1) {};
\node[state] (A) at (1,2) {};
\node[state] (B) at (2,2) {};
\node[state] (C) at (3,2) {};
\node[state] (D) at (4,2) {};
\node[thick] at (0,1) {$R$};
\node[thick] at (0,2) {$P$};
\node[thick] at (1.5,3) {$I$};
\node[thick] at (3.5,3) {$E$};
\draw[dotted, thick] (2.5,0.5) -- (2.5,2.5);
\draw[every loop]
(A) edge[bend left] node [scale=1.2, above=0.1 of C] {$u_1\frac{1}{s}$} (B)
(B) edge[bend left] node [scale=1.2, above=0.1 of B] {} (C);

\end{tikzpicture}
\end{minipage}

\begin{minipage}{0.45\textwidth}
Condition: $x_{1}=0$, $x_{2}=1$, $y=1$, $\mathrm{TA}^3_{3}$=I\\
Therefore, Type I, $\neg x_{2}=0$,\\
$C_{3}=0$.
\end{minipage}
\begin{minipage}{0.35\textwidth}
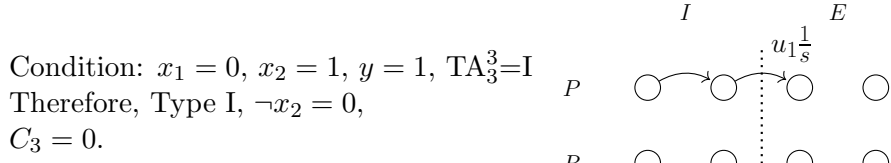
\begin{tikzpicture}[node distance = .35cm, font=\Huge]
\tikzstyle{every node}=[scale=0.35]
\node[state] (E) at (1,1) {};
\node[state] (F) at (2,1) {};
\node[state] (G) at (3,1) {};
\node[state] (H) at (4,1) {};
\node[state] (A) at (1,2) {};
\node[state] (B) at (2,2) {};
\node[state] (C) at (3,2) {};
\node[state] (D) at (4,2) {};
\node[thick] at (0,1) {$R$};
\node[thick] at (0,2) {$P$};
\node[thick] at (1.5,3) {$I$};
\node[thick] at (3.5,3) {$E$};
\draw[dotted, thick] (2.5,0.5) -- (2.5,2.5);
\draw[every loop]
(A) edge[bend left] node [scale=1.2, above=0.1 of C] {} (B)
(B) edge[bend left] node [scale=1.2, above=0.1 of B] {$~~~~~~~u_1\frac{1}{s}$} (C);

\end{tikzpicture}
\end{minipage}

For the Sub-case 1 (d), we study the transition of $\mathrm{TA}^3_{4}$ when it has the current action ``Exclude''. 

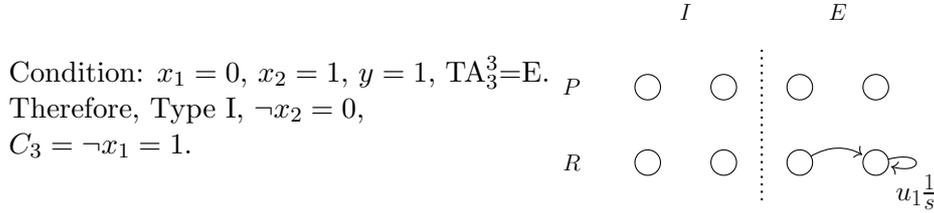
\begin{minipage}{0.45\textwidth}
Condition: $x_{1}=0$, $x_{2}=1$, $y=1$, $\mathrm{TA}^3_{3}$=E.\\
Therefore, Type I, $\neg x_{2}=0$,\\
$C_{3}=\neg x_{1}=1$.
\end{minipage}
\begin{minipage}{0.35\textwidth}
\begin{tikzpicture}[node distance = .35cm, font=\Huge]
\tikzstyle{every node}=[scale=0.35]
\node[state] (E) at (1,1) {};
\node[state] (F) at (2,1) {};
\node[state] (G) at (3,1) {};
\node[state] (H) at (4,1) {};
\node[state] (A) at (1,2) {};
\node[state] (B) at (2,2) {};
\node[state] (C) at (3,2) {};
\node[state] (D) at (4,2) {};
\node[thick] at (0,1) {$R$};
\node[thick] at (0,2) {$P$};
\node[thick] at (1.5,3) {$I$};
\node[thick] at (3.5,3) {$E$};
\draw[dotted, thick] (2.5,0.5) -- (2.5,2.5);
\draw[every loop]
(G) edge[bend left] node [scale=1.2, above=0.1 of C] {} (H)
(H) edge[loop right] node [scale=1.2, below=0.1 of H] {$u_1\frac{1}{s}$} (H);

\end{tikzpicture}
\end{minipage}

\begin{minipage}{0.45\textwidth}
Condition: $x_{1}=0$, $x_{2}=1$, $y=1$, $\mathrm{TA}^3_{3}$=I. \\
Therefore, Type I, $\neg x_{2}=0$,\\
$C_{3}=\neg x_{1} \wedge x_{2}=1$.
\end{minipage}
\begin{minipage}{0.35\textwidth}
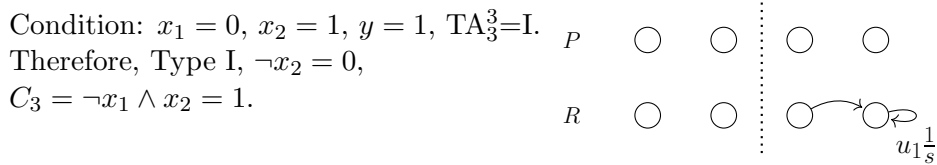
\begin{tikzpicture}[node distance = .35cm, font=\Huge]
\tikzstyle{every node}=[scale=0.35]
\node[state] (E) at (1,1) {};
\node[state] (F) at (2,1) {};
\node[state] (G) at (3,1) {};
\node[state] (H) at (4,1) {};
\node[state] (A) at (1,2) {};
\node[state] (B) at (2,2) {};
\node[state] (C) at (3,2) {};
\node[state] (D) at (4,2) {};
\node[thick] at (0,1) {$R$};
\node[thick] at (0,2) {$P$};
\node[thick] at (1.5,3) {$I$};
\node[thick] at (3.5,3) {$E$};
\draw[dotted, thick] (2.5,0.5) -- (2.5,2.5);
\draw[every loop]
(G) edge[bend left] node [scale=1.2, above=0.1 of C] {} (H)
(H) edge[loop right] node [scale=1.2, below=0.1 of H] {$u_1\frac{1}{s}$} (H);

\end{tikzpicture}
\end{minipage}

So far, we have gone through all sub-cases in Case 1. We are now ready to sum up Case 1 by looking at the transitions of $\mathrm{TA}^3_{3}$ and $\mathrm{TA}^3_4$ in different scenarios. Clearly, $\mathrm{TA}^3_{4}$ will become ``\textsl{Exclude}" in the long run because it has only one direction of transition, i.e., towards action ``Exclude''. 
Given $\mathrm{TA}^3_{4}$ is ``Exclude'', action ``Include'' of $\mathrm{TA}^3_{3}$ is an absorbing state.
Therefore, if $\mathrm{TA}^3_{1}$ and $\mathrm{TA}^3_{2}$ are ``Exclude" and ``Include", respectively, $\mathrm{TA}^3_{3}$ will become ``Include", and $\mathrm{TA}^3_{4}$ will eventually be ``Exclude". In other words, $C_{3}$ will converge to $\neg x_{1} \wedge x_{2}$ in Case 1. \\

\noindent {\bf Case 2} \\
Case 2 studies the behavior of $\mathrm{TA}^3_3$ and $\mathrm{TA}^3_4$ when $\mathrm{TA}^3_1$ and $\mathrm{TA}^3_2$ select ``Include" and ``Exclude", respectively. There are here also four sub-cases and we will detail them presently. 

We first study $\mathrm{TA}^3_3$ with action ``Include", providing the below transitions. 


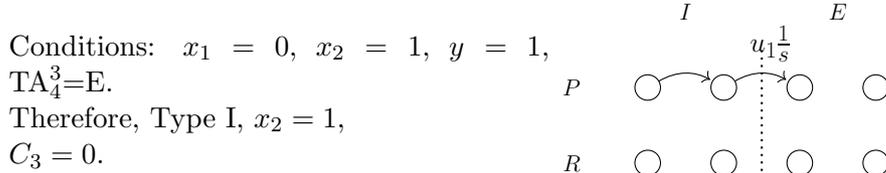
\begin{minipage}{0.45\textwidth}
Conditions: $x_{1}=0$, $x_{2}=1$, ${y}=1$, $\mathrm{TA}^3_{4}$=E. \\
Therefore, Type I, $ x_{2} = 1$, \\$C_{3} = 0$.
\end{minipage}
\begin{minipage}{0.35\textwidth}
\begin{tikzpicture}[node distance = .35cm, font=\Huge]
\tikzstyle{every node}=[scale=0.35]
\node[state] (E) at (1,1) {};
\node[state] (F) at (2,1) {};
\node[state] (G) at (3,1) {};
\node[state] (H) at (4,1) {};
\node[state] (A) at (1,2) {};
\node[state] (B) at (2,2) {};
\node[state] (C) at (3,2) {};
\node[state] (D) at (4,2) {};
\node[thick] at (0,1) {$R$};
\node[thick] at (0,2) {$P$};
\node[thick] at (1.5,3) {$I$};
\node[thick] at (3.5,3) {$E$};
\draw[dotted, thick] (2.5,0.5) -- (2.5,2.5);
\draw[every loop]
(A) edge[bend left] node [scale=1.2, above=0.1 of C] {} (B)
(B) edge[bend left] node [scale=1.2, above=0.1 of C] {~~$u_1\frac{1}{s}$} (C);

\end{tikzpicture}
\end{minipage}

\begin{minipage}{0.45\textwidth}
Conditions: $x_{1}=0$, $x_{2}=1$, $y=1$, $\mathrm{TA}^3_{4}$=I. \\
Therefore, Type I, $x_{2}=1$, $C_{3} = 0$. 
\end{minipage}
\begin{minipage}{0.35\textwidth}
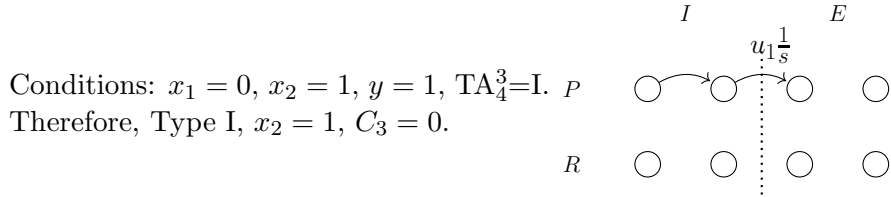
\begin{tikzpicture}[node distance = .35cm, font=\Huge]
\tikzstyle{every node}=[scale=0.35]
\node[state] (E) at (1,1) {};
\node[state] (F) at (2,1) {};
\node[state] (G) at (3,1) {};
\node[state] (H) at (4,1) {};
\node[state] (A) at (1,2) {};
\node[state] (B) at (2,2) {};
\node[state] (C) at (3,2) {};
\node[state] (D) at (4,2) {};
\node[thick] at (0,1) {$R$};
\node[thick] at (0,2) {$P$};
\node[thick] at (1.5,3) {$I$};
\node[thick] at (3.5,3) {$E$};
\draw[dotted, thick] (2.5,0.5) -- (2.5,2.5);
\draw[every loop]
(A) edge[bend left] node [scale=1.2, above=0.1 of C] {} (B)
(B) edge[bend left] node [scale=1.2, above=0.1 of C] {~~$u_1\frac{1}{s}$} (C);

\end{tikzpicture}
\end{minipage}

We then study $\mathrm{TA}^3_3$ with action ``Exclude", transitions shown below. 

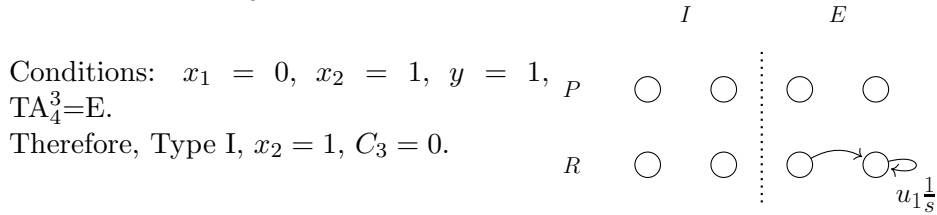
\begin{minipage}{0.45\textwidth}
Conditions: $x_{1}=0$, $x_{2}=1$, $y=1$, $\mathrm{TA}^3_{4}$=E. \\
Therefore, Type I, $x_{2}=1$, $C_3 = 0$.
\end{minipage}
\begin{minipage}{0.35\textwidth}
\begin{tikzpicture}[node distance = .35cm, font=\Huge]
\tikzstyle{every node}=[scale=0.35]
\node[state] (E) at (1,1) {};
\node[state] (F) at (2,1) {};
\node[state] (G) at (3,1) {};
\node[state] (H) at (4,1) {};
\node[state] (A) at (1,2) {};
\node[state] (B) at (2,2) {};
\node[state] (C) at (3,2) {};
\node[state] (D) at (4,2) {};
\node[thick] at (0,1) {$R$};
\node[thick] at (0,2) {$P$};
\node[thick] at (1.5,3) {$I$};
\node[thick] at (3.5,3) {$E$};
\draw[dotted, thick] (2.5,0.5) -- (2.5,2.5);
\draw[every loop]
(G) edge[bend left] node [scale=1.2, above=0.1 of C] {} (H)
(H) edge[loop right] node [scale=1.2, below=0.1 of H] {$u_1\frac{1}{s}$} (H);

\end{tikzpicture}
\end{minipage}

\begin{minipage}{0.45\textwidth}
Conditions: $x_{1}=0$, $x_{2}=1$, $y=1$, $\mathrm{TA}^3_{4}$=I. \\
Therefore, Type I, $x_2=1$, $C_3 = 0$.
\end{minipage}
\begin{minipage}{0.35\textwidth}
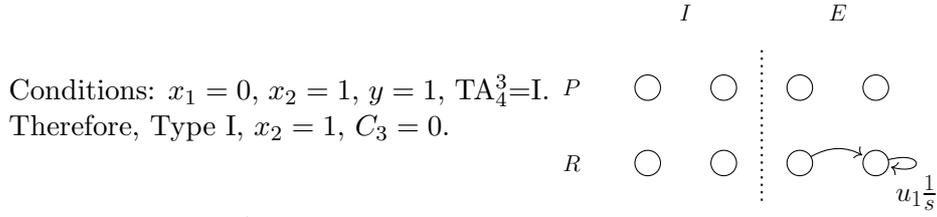
\begin{tikzpicture}[node distance = .35cm, font=\Huge]
\tikzstyle{every node}=[scale=0.35]
\node[state] (E) at (1,1) {};
\node[state] (F) at (2,1) {};
\node[state] (G) at (3,1) {};
\node[state] (H) at (4,1) {};
\node[state] (A) at (1,2) {};
\node[state] (B) at (2,2) {};
\node[state] (C) at (3,2) {};
\node[state] (D) at (4,2) {};
\node[thick] at (0,1) {$R$};
\node[thick] at (0,2) {$P$};
\node[thick] at (1.5,3) {$I$};
\node[thick] at (3.5,3) {$E$};
\draw[dotted, thick] (2.5,0.5) -- (2.5,2.5);
\draw[every loop]
(G) edge[bend left] node [scale=1.2, above=0.1 of C] {} (H)
(H) edge[loop right] node [scale=1.2, below=0.1 of H] {$u_1\frac{1}{s}$} (H);

\end{tikzpicture}
\end{minipage}

We now study $\mathrm{TA}^3_4$ with action ``Include" and the transitions are presented below. 

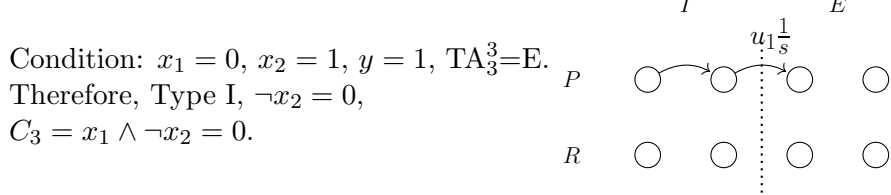
\begin{minipage}{0.45\textwidth}
Condition: $x_{1}=0$, $x_{2}=1$, $y=1$, $\mathrm{TA}^3_{3}$=E. \\
Therefore, Type I, $\neg x_{2} = 0$,\\ $C_3 =x_{1} \wedge \neg x_{2}= 0$.
\end{minipage}
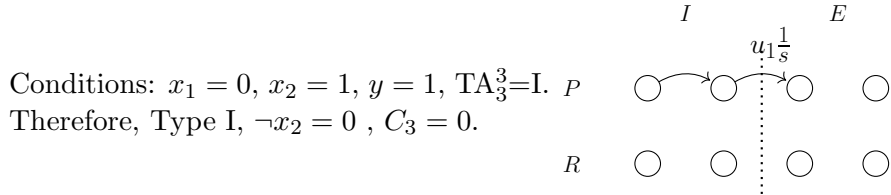
\begin{minipage}{0.35\textwidth}
\begin{tikzpicture}[node distance = .35cm, font=\Huge]
\tikzstyle{every node}=[scale=0.35]
\node[state] (E) at (1,1) {};
\node[state] (F) at (2,1) {};
\node[state] (G) at (3,1) {};
\node[state] (H) at (4,1) {};
\node[state] (A) at (1,2) {};
\node[state] (B) at (2,2) {};
\node[state] (C) at (3,2) {};
\node[state] (D) at (4,2) {};
\node[thick] at (0,1) {$R$};
\node[thick] at (0,2) {$P$};
\node[thick] at (1.5,3) {$I$};
\node[thick] at (3.5,3) {$E$};
\draw[dotted, thick] (2.5,0.5) -- (2.5,2.5);
\draw[every loop]
(A) edge[bend left] node [scale=1.2, above=0.1 of C] {} (B)
(B) edge[bend left] node [scale=1.2, above=0.1 of C] {~~$u_1\frac{1}{s}$} (C);

\end{tikzpicture}
\end{minipage}

\begin{minipage}{0.45\textwidth}
Conditions: $x_{1}=0$, $x_{2}=1$, $y=1$, $\mathrm{TA}^3_{3}$=I. \\
Therefore, Type I, $\neg x_{2} = 0$ , $ C_3= 0$.
\end{minipage}
\begin{minipage}{0.35\textwidth}
\begin{tikzpicture}[node distance = .35cm, font=\Huge]
\tikzstyle{every node}=[scale=0.35]
\node[state] (E) at (1,1) {};
\node[state] (F) at (2,1) {};
\node[state] (G) at (3,1) {};
\node[state] (H) at (4,1) {};
\node[state] (A) at (1,2) {};
\node[state] (B) at (2,2) {};
\node[state] (C) at (3,2) {};
\node[state] (D) at (4,2) {};
\node[thick] at (0,1) {$R$};
\node[thick] at (0,2) {$P$};
\node[thick] at (1.5,3) {$I$};
\node[thick] at (3.5,3) {$E$};
\draw[dotted, thick] (2.5,0.5) -- (2.5,2.5);
\draw[every loop]
(A) edge[bend left] node [scale=1.2, above=0.1 of C] {} (B)
(B) edge[bend left] node [scale=1.2, above=0.1 of C] {~~$u_1\frac{1}{s}$} (C);

\end{tikzpicture}
\end{minipage}

We study lastly $\mathrm{TA}^3_4$ with action ``Exclude", leading to the following transitions. 

\begin{minipage}{0.45\textwidth}
Conditions: $x_{1}=1$, $x_{2}=1$, $y=0$, $\mathrm{TA}^3_{3}$=E. \\
Therefore, Type II, $\neg x_{2} = 0$, $C_3 =x_{1}= 1$.
\end{minipage}
\begin{minipage}{0.35\textwidth}
\begin{tikzpicture}[node distance = .35cm, font=\Huge]
\tikzstyle{every node}=[scale=0.35]
\node[state] (E) at (1,1) {};
\node[state] (F) at (2,1) {};
\node[state] (G) at (3,1) {};
\node[state] (H) at (4,1) {};
\node[state] (A) at (1,2) {};
\node[state] (B) at (2,2) {};
\node[state] (C) at (3,2) {};
\node[state] (D) at (4,2) {};
\node[thick] at (0,1) {$R$};
\node[thick] at (0,2) {$P$};
\node[thick] at (1.5,3) {$I$};
\node[thick] at (3.5,3) {$E$};
\draw[dotted, thick] (2.5,0.5) -- (2.5,2.5);
\draw[every loop]
(D) edge[bend right] node [scale=1.2, above=0.1 of C] {$u_2\times1$} (C)
(C) edge[bend right] node [scale=1.2, above=0.1 of B] {} (B);

\end{tikzpicture}
\end{minipage}

\begin{minipage}{0.45\textwidth}
Conditions: $x_{1}=0$, $x_{2}=1$, $y=1$, $\mathrm{TA}^3_{3}$=E. \\
Therefore, Type I, $\neg x_{2}=0$, $C_3 =0$.
\end{minipage}
\begin{minipage}{0.35\textwidth}
\begin{tikzpicture}[node distance = .35cm, font=\Huge]
\tikzstyle{every node}=[scale=0.35]
\node[state] (E) at (1,1) {};
\node[state] (F) at (2,1) {};
\node[state] (G) at (3,1) {};
\node[state] (H) at (4,1) {};
\node[state] (A) at (1,2) {};
\node[state] (B) at (2,2) {};
\node[state] (C) at (3,2) {};
\node[state] (D) at (4,2) {};
\node[thick] at (0,1) {$R$};
\node[thick] at (0,2) {$P$};
\node[thick] at (1.5,3) {$I$};
\node[thick] at (3.5,3) {$E$};
\draw[dotted, thick] (2.5,0.5) -- (2.5,2.5);
\draw[every loop]
(G) edge[bend left] node [scale=1.2, above=0.1 of C] {} (H)
(H) edge[loop right] node [scale=1.2, below=0.1 of H] {$u_1\frac{1}{s}$} (H);

\end{tikzpicture}
\end{minipage}

\begin{minipage}{0.45\textwidth}
Conditions: $x_{1}=1$, $x_{2}=1$, $y=0$, $\mathrm{TA}^3_{3}$=I. \\
Therefore, Type II, $\neg x_{2} = 0$,
\\$C_3 =x_{1} \wedge x_{2}= 1$. 
\end{minipage}
\begin{minipage}{0.35\textwidth}
\begin{tikzpicture}[node distance = .35cm, font=\Huge]
\tikzstyle{every node}=[scale=0.35]
\node[state] (E) at (1,1) {};
\node[state] (F) at (2,1) {};
\node[state] (G) at (3,1) {};
\node[state] (H) at (4,1) {};
\node[state] (A) at (1,2) {};
\node[state] (B) at (2,2) {};
\node[state] (C) at (3,2) {};
\node[state] (D) at (4,2) {};
\node[thick] at (0,1) {$R$};
\node[thick] at (0,2) {$P$};
\node[thick] at (1.5,3) {$I$};
\node[thick] at (3.5,3) {$E$};
\draw[dotted, thick] (2.5,0.5) -- (2.5,2.5);
\draw[every loop]
(H) edge[bend right] node [scale=1.2, above=0.1 of C] {$u_2\times1$} (G)
(G) edge[bend right] node [scale=1.2, below=0.1 of B] { } (F);

\end{tikzpicture}
\end{minipage}

\begin{minipage}{0.45\textwidth}
Conditions: $x_{1}=0$, $x_{2}=1$, $y=1$, $\mathrm{TA}^3_{3}$=I. \\
Therefore, Type I, $\neg x_{2} = 0$, \\$C_3 =x_{1} \wedge x_{2}= 0$.
\end{minipage}
\begin{minipage}{0.35\textwidth}
\begin{tikzpicture}[node distance = .35cm, font=\Huge]
\tikzstyle{every node}=[scale=0.35]
\node[state] (E) at (1,1) {};
\node[state] (F) at (2,1) {};
\node[state] (G) at (3,1) {};
\node[state] (H) at (4,1) {};
\node[state] (A) at (1,2) {};
\node[state] (B) at (2,2) {};
\node[state] (C) at (3,2) {};
\node[state] (D) at (4,2) {};
\node[thick] at (0,1) {$R$};
\node[thick] at (0,2) {$P$};
\node[thick] at (1.5,3) {$I$};
\node[thick] at (3.5,3) {$E$};
\draw[dotted, thick] (2.5,0.5) -- (2.5,2.5);
\draw[every loop]
(G) edge[bend left]node [scale=1.2, above=0.1 of C] {} (H)
(H) edge[loop right] node [scale=1.2, below=0.1 of H] {$u_1\frac{1}{s}$} (H);

\end{tikzpicture}
\end{minipage}

To sum up Case 2, we understand that $\mathrm{TA}^3_{3}$ will select ``Exclude", and $T^3_{4}$ will switch between ``Include" or ``Exclude", depending on the training samples and system status.\\


\noindent {\bf Case 3} \\
Now we move onto Case 3, where $\mathrm{TA}^3_{1}$ and $\mathrm{TA}^3_{2}$ both select ``Exclude". We study the behavior of $\mathrm{TA}^3_{3}$ and $\mathrm{TA}^3_{4}$ for different sub-cases. Clearly, in this case, the first bit $x_1$ does not play any role for the output. 

We first examine $\mathrm{TA}^3_3$ with action ``Include'', providing the transitions below. 

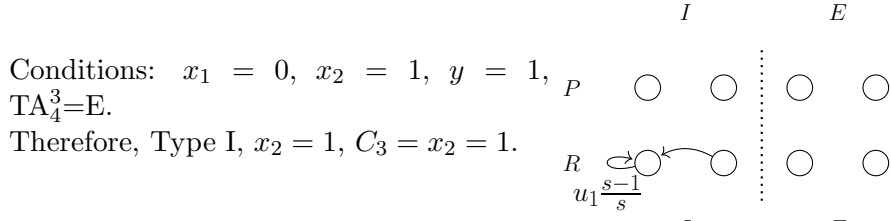
\begin{minipage}{0.45\textwidth}
{Conditions}: $x_{1}=0$, $x_{2}=1$, $y=1$, $\mathrm{TA}^3_4$=E. \\
Therefore, Type I, $x_{2} = 1$, $C_{3}=x_{2} = 1$.
\end{minipage}
\begin{minipage}{0.35\textwidth}
\begin{tikzpicture}[node distance = .35cm, font=\Huge]
\tikzstyle{every node}=[scale=0.35]
\node[state] (E) at (1,1) {};
\node[state] (F) at (2,1) {};
\node[state] (G) at (3,1) {};
\node[state] (H) at (4,1) {};
\node[state] (A) at (1,2) {};
\node[state] (B) at (2,2) {};
\node[state] (C) at (3,2) {};
\node[state] (D) at (4,2) {};
\node[thick] at (0,1) {$R$};
\node[thick] at (0,2) {$P$};
\node[thick] at (1.5,3) {$I$};
\node[thick] at (3.5,3) {$E$};
\draw[dotted, thick] (2.5,0.5) -- (2.5,2.5);
\draw[every loop]
(F) edge[bend right] node [scale=1.2, above=0.1 of C] {} (E)
(E) edge[loop left = 45] node [scale=1.2, below=0.1 of E] {$u_1\frac{s-1}{s}$} (E);

\end{tikzpicture}
\end{minipage}

\begin{minipage}{0.45\textwidth}
Conditions: $x_{1}=0$, $x_{2}=1$, $y=1$, $\mathrm{TA}^3_{4}$=I. \\
Therefore, Type I, $x_{2} = 1$, $C_3= 0$. 
\end{minipage}
\begin{minipage}{0.35\textwidth}
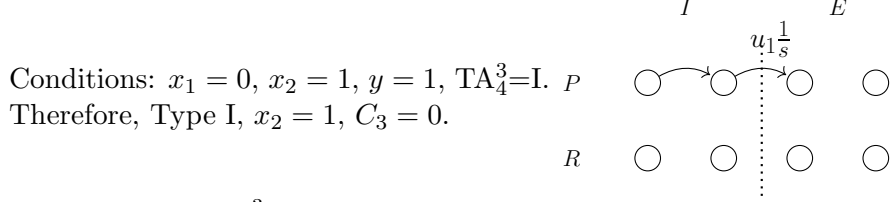
\begin{tikzpicture}[node distance = .35cm, font=\Huge]
\tikzstyle{every node}=[scale=0.35]
\node[state] (E) at (1,1) {};
\node[state] (F) at (2,1) {};
\node[state] (G) at (3,1) {};
\node[state] (H) at (4,1) {};
\node[state] (A) at (1,2) {};
\node[state] (B) at (2,2) {};
\node[state] (C) at (3,2) {};
\node[state] (D) at (4,2) {};
\node[thick] at (0,1) {$R$};
\node[thick] at (0,2) {$P$};
\node[thick] at (1.5,3) {$I$};
\node[thick] at (3.5,3) {$E$};
\draw[dotted, thick] (2.5,0.5) -- (2.5,2.5);
\draw[every loop]
(A) edge[bend left] node [scale=1.2, above=0.1 of C] {} (B)
(B) edge[bend left] node [scale=1.2, above=0.1 of C] {~~$u_1\frac{1}{s}$} (C);

\end{tikzpicture}
\end{minipage}

We then study $\mathrm{TA}^3_3$ with action ``Exclude'', transitions shown below. 
In this situation, if $T_{3,4}$ is also excluded, $C_3$ is ``empty'' since all its associated TA select action ``Exclude''. To make the training proceed, according to the training rule of TM, we assign $C_3=1$ in this situation. 

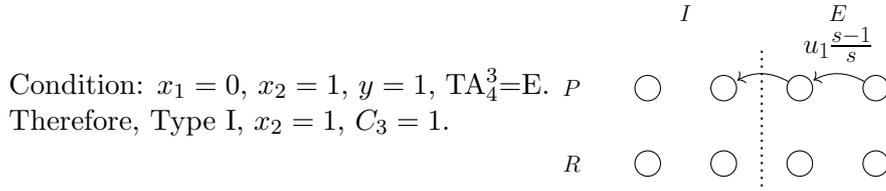
\begin{minipage}{0.45\textwidth}
{Condition:} $x_{1}=0$, $x_{2}=1$, $y=1$, $\mathrm{TA}^3_{4}$=E. \\
Therefore, Type I, $x_{2} = 1$, $C_{3}=1$.
\end{minipage}
\begin{minipage}{0.35\textwidth}
\begin{tikzpicture}[node distance = .35cm, font=\Huge]
\tikzstyle{every node}=[scale=0.35]
\node[state] (E) at (1,1) {};
\node[state] (F) at (2,1) {};
\node[state] (G) at (3,1) {};
\node[state] (H) at (4,1) {};
\node[state] (A) at (1,2) {};
\node[state] (B) at (2,2) {};
\node[state] (C) at (3,2) {};
\node[state] (D) at (4,2) {};
\node[thick] at (0,1) {$R$};
\node[thick] at (0,2) {$P$};
\node[thick] at (1.5,3) {$I$};
\node[thick] at (3.5,3) {$E$};
\draw[dotted, thick] (2.5,0.5) -- (2.5,2.5);
\draw[every loop]
(D) edge[bend right] node [scale=1.2, above=0.1 of C] {$u_1\frac{s-1}{s}$} (C)
(C) edge[bend right] node [scale=1.2, above=0.1 of C] {} (B);

\end{tikzpicture}
\end{minipage}

\begin{minipage}{0.45\textwidth}
{Condition:} $x_{1}=0$, $x_{2}=0$, $y=0$, $\mathrm{TA}^3_{4}$=E. \\
Therefore, Type II, $x_{2} = 0$, $C_{3}=1$.
\end{minipage}
\begin{minipage}{0.35\textwidth}
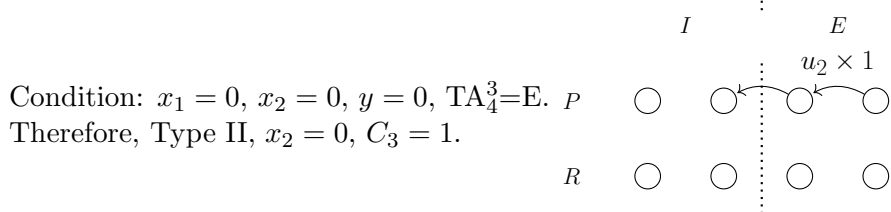
\begin{tikzpicture}[node distance = .35cm, font=\Huge]
\tikzstyle{every node}=[scale=0.35]
\node[state] (E) at (1,1) {};
\node[state] (F) at (2,1) {};
\node[state] (G) at (3,1) {};
\node[state] (H) at (4,1) {};
\node[state] (A) at (1,2) {};
\node[state] (B) at (2,2) {};
\node[state] (C) at (3,2) {};
\node[state] (D) at (4,2) {};
\node[thick] at (0,1) {$R$};
\node[thick] at (0,2) {$P$};
\node[thick] at (1.5,3) {$I$};
\node[thick] at (3.5,3) {$E$};
\draw[dotted, thick] (2.5,0.5) -- (2.5,2.5);
\draw[every loop]
(D) edge[bend right] node [scale=1.2, above=0.1 of C] {$u_2\times 1$} (C)
(C) edge[bend right] node [scale=1.2, above=0.1 of C] {} (B);

\end{tikzpicture}
\end{minipage}

\begin{minipage}{0.45\textwidth}
Condition: $x_{1}=0$, $x_{2}=1$, $y=1$, $\mathrm{TA}^3_{4}$=I. \\
Therefore, Type I, $x_{2} = 1$,  $C_{3}=\neg x_{2} = 0$.
\end{minipage}
\begin{minipage}{0.35\textwidth}
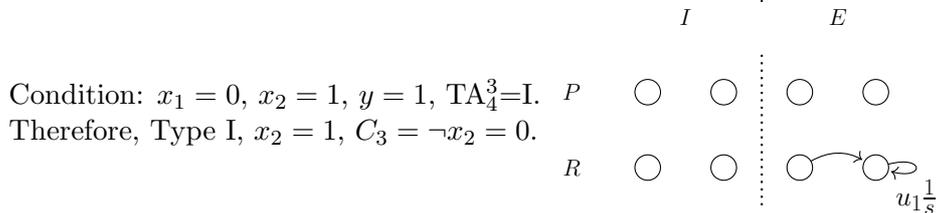
\begin{tikzpicture}[node distance = .35cm, font=\Huge]
\tikzstyle{every node}=[scale=0.35]
\node[state] (E) at (1,1) {};
\node[state] (F) at (2,1) {};
\node[state] (G) at (3,1) {};
\node[state] (H) at (4,1) {};
\node[state] (A) at (1,2) {};
\node[state] (B) at (2,2) {};
\node[state] (C) at (3,2) {};
\node[state] (D) at (4,2) {};
\node[thick] at (0,1) {$R$};
\node[thick] at (0,2) {$P$};
\node[thick] at (1.5,3) {$I$};
\node[thick] at (3.5,3) {$E$};
\draw[dotted, thick] (2.5,0.5) -- (2.5,2.5);
\draw[every loop]
(G) edge[bend left] node [scale=1.2, above=0.1 of C] {} (H)
(H) edge[loop right] node [scale=1.2, below=0.1 of H] {$u_1\frac{1}{s}$} (H);

\end{tikzpicture}
\end{minipage}

\begin{minipage}{0.45\textwidth}
Condition: $x_{1}=0$, $x_{2}=0$, $y=0$, $\mathrm{TA}^3_{4}$=I. \\
Therefore, Type II, $x_{2} = 0$, $C_{3}=\neg x_{2} = 1$.
\end{minipage}
\begin{minipage}{0.35\textwidth}
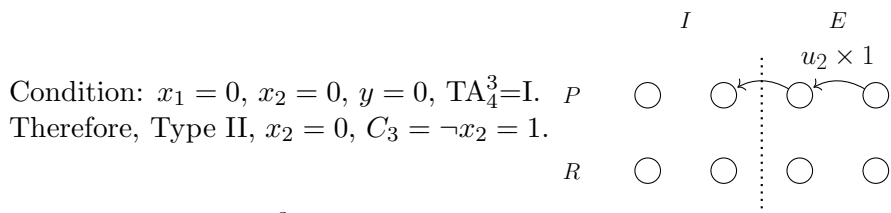
\begin{tikzpicture}[node distance = .35cm, font=\Huge]
\tikzstyle{every node}=[scale=0.35]
\node[state] (E) at (1,1) {};
\node[state] (F) at (2,1) {};
\node[state] (G) at (3,1) {};
\node[state] (H) at (4,1) {};
\node[state] (A) at (1,2) {};
\node[state] (B) at (2,2) {};
\node[state] (C) at (3,2) {};
\node[state] (D) at (4,2) {};
\node[thick] at (0,1) {$R$};
\node[thick] at (0,2) {$P$};
\node[thick] at (1.5,3) {$I$};
\node[thick] at (3.5,3) {$E$};
\draw[dotted, thick] (2.5,0.5) -- (2.5,2.5);
\draw[every loop]
(D) edge[bend right] node [scale=1.2, above=0.1 of C] {$u_2\times 1$} (C)
(C) edge[bend right] node [scale=1.2, above=0.1 of C] {} (B);

\end{tikzpicture}
\end{minipage}

We thirdly study $\mathrm{TA}^3_4$ with action ``Include'', covering the transitions shown below. 

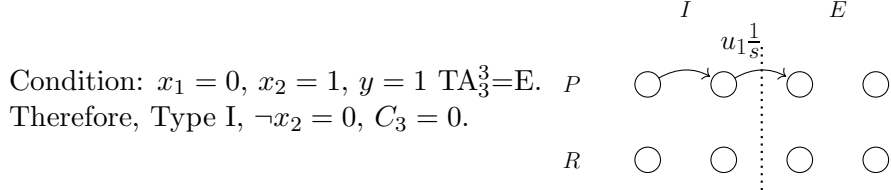
\begin{minipage}{0.45\textwidth}
Condition: $x_{1}=0$, $x_{2}=1$, $y=1$ $\mathrm{TA}^3_{3}$=E. \\
Therefore, Type I, $\neg x_{2} = 0$, $C_3=0$.
\end{minipage}
\begin{minipage}{0.35\textwidth}
\begin{tikzpicture}[node distance = .35cm, font=\Huge]
\tikzstyle{every node}=[scale=0.35]
\node[state] (E) at (1,1) {};
\node[state] (F) at (2,1) {};
\node[state] (G) at (3,1) {};
\node[state] (H) at (4,1) {};
\node[state] (A) at (1,2) {};
\node[state] (B) at (2,2) {};
\node[state] (C) at (3,2) {};
\node[state] (D) at (4,2) {};
\node[thick] at (0,1) {$R$};
\node[thick] at (0,2) {$P$};
\node[thick] at (1.5,3) {$I$};
\node[thick] at (3.5,3) {$E$};
\draw[dotted, thick] (2.5,0.5) -- (2.5,2.5);
\draw[every loop]
(A) edge[bend left] node [scale=1.2, above=0.1 of C] {} (B)
(B) edge[bend left] node [scale=1.2, above=0.1 of C] {$u_1\frac{1}{s}~~~~~$} (C);

\end{tikzpicture}
\end{minipage}

\begin{minipage}{0.45\textwidth}
Condition: $x_{1}=0$, $x_{2}=1$, $y=1$, $\mathrm{TA}^3_{3}$=I.\\
Therefore, Type I, $\neg x_{2} = 0$, $C_3=0$. 
\end{minipage}
\begin{minipage}{0.35\textwidth}
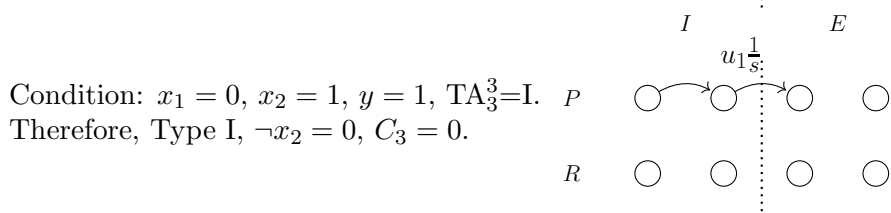
\begin{tikzpicture}[node distance = .35cm, font=\Huge]
\tikzstyle{every node}=[scale=0.35]
\node[state] (E) at (1,1) {};
\node[state] (F) at (2,1) {};
\node[state] (G) at (3,1) {};
\node[state] (H) at (4,1) {};
\node[state] (A) at (1,2) {};
\node[state] (B) at (2,2) {};
\node[state] (C) at (3,2) {};
\node[state] (D) at (4,2) {};
\node[thick] at (0,1) {$R$};
\node[thick] at (0,2) {$P$};
\node[thick] at (1.5,3) {$I$};
\node[thick] at (3.5,3) {$E$};
\draw[dotted, thick] (2.5,0.5) -- (2.5,2.5);
\draw[every loop]
(A) edge[bend left] node [scale=1.2, above=0.1 of C] {} (B)
(B) edge[bend left] node [scale=1.2, above=0.1 of C] {$u_1\frac{1}{s}~~~~~$} (C);

\end{tikzpicture}
\end{minipage}

Lastly, we study $\mathrm{TA}^3_4$ with action ``Exclude'', transitions shown below. 
Similarly, in this situation, when $\mathrm{TA}^3_{3}$ is also excluded, $C_3$ becomes ``empty'' again, as all its associated TAs select action ``Exclude''. Following the training rule of TM, we assign $C_3=1$.

\begin{minipage}{0.45\textwidth}
Conditions: $x_{1}=1$, $x_{2}=1$, $y=0$, $\mathrm{TA}^3_{3}$=E. \\
Therefore, Type II, $\neg x_{2} = 0$, $C_3=1$.
\end{minipage}
\begin{minipage}{0.35\textwidth}
\begin{tikzpicture}[node distance = .35cm, font=\Huge]
\tikzstyle{every node}=[scale=0.35]
\node[state] (E) at (1,1) {};
\node[state] (F) at (2,1) {};
\node[state] (G) at (3,1) {};
\node[state] (H) at (4,1) {};
\node[state] (A) at (1,2) {};
\node[state] (B) at (2,2) {};
\node[state] (C) at (3,2) {};
\node[state] (D) at (4,2) {};
\node[thick] at (0,1) {$R$};
\node[thick] at (0,2) {$P$};
\node[thick] at (1.5,3) {$I$};
\node[thick] at (3.5,3) {$E$};
\draw[dotted, thick] (2.5,0.5) -- (2.5,2.5);
\draw[every loop]
(D) edge[bend right] node [scale=1.2, above=0.1 of C] {} (C)
(C) edge[bend right] node [scale=1.2, above=0.1 of B] {$u_2\times1$} (B);

\end{tikzpicture}
\end{minipage}

\begin{minipage}{0.45\textwidth}
Conditions: $x_{1}=0$, $x_{2}=1$, $y=1$, $\mathrm{TA}^3_{3}$=E. \\
Therefore, Type I, $\neg x_{2} = 0$, $C_3=1$.
\end{minipage}
\begin{minipage}{0.35\textwidth}
\begin{tikzpicture}[node distance = .35cm, font=\Huge]
\tikzstyle{every node}=[scale=0.35]
\node[state] (E) at (1,1) {};
\node[state] (F) at (2,1) {};
\node[state] (G) at (3,1) {};
\node[state] (H) at (4,1) {};
\node[state] (A) at (1,2) {};
\node[state] (B) at (2,2) {};
\node[state] (C) at (3,2) {};
\node[state] (D) at (4,2) {};
\node[thick] at (0,1) {$R$};
\node[thick] at (0,2) {$P$};
\node[thick] at (1.5,3) {$I$};
\node[thick] at (3.5,3) {$E$};
\draw[dotted, thick] (2.5,0.5) -- (2.5,2.5);
\draw[every loop]
(G) edge[bend left] node [scale=1.2, above=0.1 of C] {} (H)
(H) edge[loop right] node [scale=1.2, below=0.1 of H] {$u_1\frac{1}{s}$} (H);

\end{tikzpicture}
\end{minipage}

\begin{minipage}{0.45\textwidth}
Conditions: $x_{1}=1$, $x_{2}=1$, $y=0$, $\mathrm{TA}^3_{3}$=I. \\
Therefore, Type II, $\neg x_{2} = 0$, $C_3=1$.
\end{minipage}
\begin{minipage}{0.35\textwidth}
\begin{tikzpicture}[node distance = .35cm, font=\Huge]
\tikzstyle{every node}=[scale=0.35]
\node[state] (E) at (1,1) {};
\node[state] (F) at (2,1) {};
\node[state] (G) at (3,1) {};
\node[state] (H) at (4,1) {};
\node[state] (A) at (1,2) {};
\node[state] (B) at (2,2) {};
\node[state] (C) at (3,2) {};
\node[state] (D) at (4,2) {};
\node[thick] at (0,1) {$R$};
\node[thick] at (0,2) {$P$};
\node[thick] at (1.5,3) {$I$};
\node[thick] at (3.5,3) {$E$};
\draw[dotted, thick] (2.5,0.5) -- (2.5,2.5);
\draw[every loop]
(D) edge[bend right] node [scale=1.2, above=0.1 of C] {} (C)
(C) edge[bend right] node [scale=1.2, above=0.1 of B] {$u_2\times1$} (B);

\end{tikzpicture}
\end{minipage}

\begin{minipage}{0.45\textwidth}
Conditions: $x_{1}=0$, $x_{2}=1$, $y=1$, $\mathrm{TA}^3_{3}$=I. \\
Therefore, Type I, $\neg x_{2} = 0$, $C_3=1$.
\end{minipage}
\begin{minipage}{0.35\textwidth}
\begin{tikzpicture}[node distance = .35cm, font=\Huge]
\tikzstyle{every node}=[scale=0.35]
\node[state] (E) at (1,1) {};
\node[state] (F) at (2,1) {};
\node[state] (G) at (3,1) {};
\node[state] (H) at (4,1) {};
\node[state] (A) at (1,2) {};
\node[state] (B) at (2,2) {};
\node[state] (C) at (3,2) {};
\node[state] (D) at (4,2) {};
\node[thick] at (0,1) {$R$};
\node[thick] at (0,2) {$P$};
\node[thick] at (1.5,3) {$I$};
\node[thick] at (3.5,3) {$E$};
\draw[dotted, thick] (2.5,0.5) -- (2.5,2.5);
\draw[every loop]
(G) edge[bend left] node [scale=1.2, above=0.1 of C] {} (H)
(H) edge[loop right] node [scale=1.2, below=0.1 of H] {$u_1\frac{1}{s}$} (H);

\end{tikzpicture}
\end{minipage}

Clearly, in Case 3, there is no absorbing state. \\


\noindent {\bf Case 4} \\
Now, we study Case 4, where $\neg x_{1}$ and $x_{1}$ both select ``Include''. For this reason, in this case, we always have $C_{3}=0$. We study firstly $\mathrm{TA}^3_3$ with action ``Include'' and the transitions are shown below. 

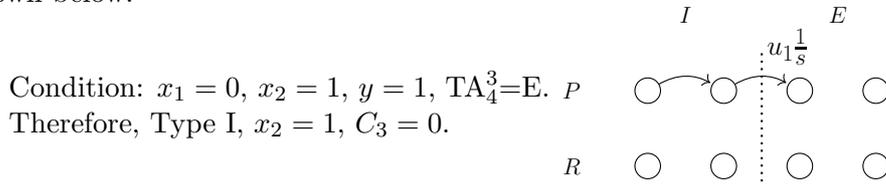
\begin{minipage}{0.45\textwidth}
Condition: $x_{1}=0$, $x_{2}=1$, $y=1$, $\mathrm{TA}^3_{4}$=E. \\
Therefore, Type I,  $x_{2} = 1$, $C_{3}= 0$.
\end{minipage}
\begin{minipage}{0.35\textwidth}
\begin{tikzpicture}[node distance = .35cm, font=\Huge]
\tikzstyle{every node}=[scale=0.35]
\node[state] (E) at (1,1) {};
\node[state] (F) at (2,1) {};
\node[state] (G) at (3,1) {};
\node[state] (H) at (4,1) {};
\node[state] (A) at (1,2) {};
\node[state] (B) at (2,2) {};
\node[state] (C) at (3,2) {};
\node[state] (D) at (4,2) {};
\node[thick] at (0,1) {$R$};
\node[thick] at (0,2) {$P$};
\node[thick] at (1.5,3) {$I$};
\node[thick] at (3.5,3) {$E$};
\draw[dotted, thick] (2.5,0.5) -- (2.5,2.5);
\draw[every loop]
(A) edge[bend left] node [scale=1.2, above=0.1 of C] {} (B)
(B) edge[bend left] node [scale=1.2, above=0.1 of C] {$~~~~~~u_1\frac{1}{s}$} (C);

\end{tikzpicture}
\end{minipage}

\begin{minipage}{0.45\textwidth}
Condition: $x_{1}=0$, $x_{2}=1$, $y=1$, $\mathrm{TA}^3_{4}$=I. \\
Therefore, Type I, $x_{2} = 1$, $C_{3}= 0$.
\end{minipage}
\begin{minipage}{0.35\textwidth}
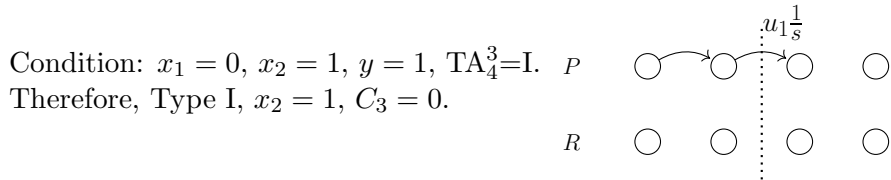
\begin{tikzpicture}[node distance = .35cm, font=\Huge]
\tikzstyle{every node}=[scale=0.35]
\node[state] (E) at (1,1) {};
\node[state] (F) at (2,1) {};
\node[state] (G) at (3,1) {};
\node[state] (H) at (4,1) {};
\node[state] (A) at (1,2) {};
\node[state] (B) at (2,2) {};
\node[state] (C) at (3,2) {};
\node[state] (D) at (4,2) {};
\node[thick] at (0,1) {$R$};
\node[thick] at (0,2) {$P$};
\node[thick] at (1.5,3) {$I$};
\node[thick] at (3.5,3) {$E$};
\draw[dotted, thick] (2.5,0.5) -- (2.5,2.5);
\draw[every loop]
(A) edge[bend left] node [scale=1.2, above=0.1 of C] {} (B)
(B) edge[bend left] node [scale=1.2, above=0.1 of C] {$~~~~~u_1\frac{1}{s}$} (C);

\end{tikzpicture}
\end{minipage}

We secondly study $\mathrm{TA}^3_{3}$ with action ``Exclude''.

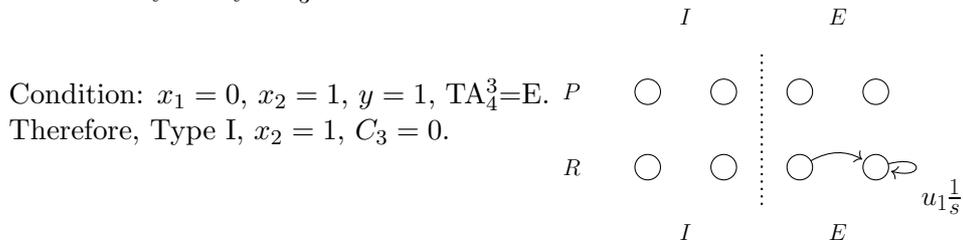
\begin{minipage}{0.45\textwidth}
Condition: $x_{1}=0$, $x_{2}=1$, $y=1$, $\mathrm{TA}^3_{4}$=E. \\
Therefore, Type I, $x_{2} = 1$, $C_{3}= 0$.
\end{minipage}
\begin{minipage}{0.35\textwidth}
\begin{tikzpicture}[node distance = .35cm, font=\Huge]
\tikzstyle{every node}=[scale=0.35]
\node[state] (E) at (1,1) {};
\node[state] (F) at (2,1) {};
\node[state] (G) at (3,1) {};
\node[state] (H) at (4,1) {};
\node[state] (A) at (1,2) {};
\node[state] (B) at (2,2) {};
\node[state] (C) at (3,2) {};
\node[state] (D) at (4,2) {};
\node[thick] at (0,1) {$R$};
\node[thick] at (0,2) {$P$};
\node[thick] at (1.5,3) {$I$};
\node[thick] at (3.5,3) {$E$};
\draw[dotted, thick] (2.5,0.5) -- (2.5,2.5);
\draw[every loop]
(G) edge[bend left] node [scale=1.2, above=0.1 of C] {} (H)
(H) edge[loop right] node [scale=1.2, below=0.1 of H] {$~~~~~~u_1\frac{1}{s}$} (H);

\end{tikzpicture}
\end{minipage}

\begin{minipage}{0.45\textwidth}
Condition: $x_{1}=0$, $x_{2}=1$, $y=1$, $\mathrm{TA}^3_{4}$=I. \\
Therefore, Type I, $x_{2} = 1$, $C_{3}= 0$.
\end{minipage}
\begin{minipage}{0.35\textwidth}
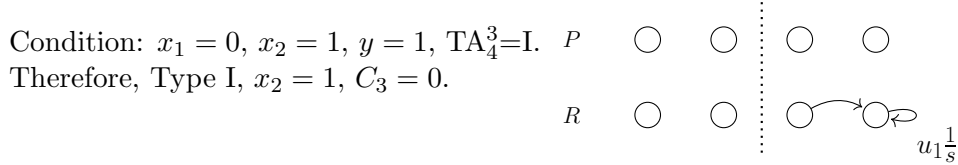
\begin{tikzpicture}[node distance = .35cm, font=\Huge]
\tikzstyle{every node}=[scale=0.35]
\node[state] (E) at (1,1) {};
\node[state] (F) at (2,1) {};
\node[state] (G) at (3,1) {};
\node[state] (H) at (4,1) {};
\node[state] (A) at (1,2) {};
\node[state] (B) at (2,2) {};
\node[state] (C) at (3,2) {};
\node[state] (D) at (4,2) {};
\node[thick] at (0,1) {$R$};
\node[thick] at (0,2) {$P$};
\node[thick] at (1.5,3) {$I$};
\node[thick] at (3.5,3) {$E$};
\draw[dotted, thick] (2.5,0.5) -- (2.5,2.5);
\draw[every loop]
(G) edge[bend left] node [scale=1.2, above=0.1 of C] {} (H)
(H) edge[loop right] node [scale=1.2, below=0.1 of H] {$~~~~~u_1\frac{1}{s}$} (H);

\end{tikzpicture}
\end{minipage}

Now, we study $\mathrm{TA}^3_{4}$ with action ``Include''.

\begin{minipage}{0.45\textwidth}
Condition: $x_{1}=0$, $x_{2}=1$, $y=1$, $\mathrm{TA}^3_{3}$=E. \\
Therefore, Type I, $\neg x_{2} = 0$, $C_{3}= 0$.
\end{minipage}
\begin{minipage}{0.35\textwidth}
\begin{tikzpicture}[node distance = .35cm, font=\Huge]
\tikzstyle{every node}=[scale=0.35]
\node[state] (E) at (1,1) {};
\node[state] (F) at (2,1) {};
\node[state] (G) at (3,1) {};
\node[state] (H) at (4,1) {};
\node[state] (A) at (1,2) {};
\node[state] (B) at (2,2) {};
\node[state] (C) at (3,2) {};
\node[state] (D) at (4,2) {};
\node[thick] at (0,1) {$R$};
\node[thick] at (0,2) {$P$};
\node[thick] at (1.5,3) {$I$};
\node[thick] at (3.5,3) {$E$};
\draw[dotted, thick] (2.5,0.5) -- (2.5,2.5);
\draw[every loop]
(A) edge[bend left] node [scale=1.2, above=0.1 of C] {} (B)
(B) edge[bend left] node [scale=1.2, above=0.1 of C] {$~~~~u_1\frac{1}{s}$} (C);

\end{tikzpicture}
\end{minipage}

\begin{minipage}{0.45\textwidth}
Condition: $x_{1}=0$, $x_{2}=1$, $y=1$, $\mathrm{TA}^3_{3}$=I. \\
Therefore, Type I,  $\neg x_{2} = 0$, $C_{3}= 0$.
\end{minipage}
\begin{minipage}{0.35\textwidth}
\begin{tikzpicture}[node distance = .35cm, font=\Huge]
\tikzstyle{every node}=[scale=0.35]
\node[state] (E) at (1,1) {};
\node[state] (F) at (2,1) {};
\node[state] (G) at (3,1) {};
\node[state] (H) at (4,1) {};
\node[state] (A) at (1,2) {};
\node[state] (B) at (2,2) {};
\node[state] (C) at (3,2) {};
\node[state] (D) at (4,2) {};
\node[thick] at (0,1) {$R$};
\node[thick] at (0,2) {$P$};
\node[thick] at (1.5,3) {$I$};
\node[thick] at (3.5,3) {$E$};
\draw[dotted, thick] (2.5,0.5) -- (2.5,2.5);
\draw[every loop]
(A) edge[bend left] node [scale=1.2, above=0.1 of C] {} (B)
(B) edge[bend left] node [scale=1.2, above=0.1 of C] {~~~~$u_1\frac{1}{s}$} (C);

\end{tikzpicture}
\end{minipage}

We lastly study $\mathrm{TA}^3_{4}$ with action ``Exclude''.

\begin{minipage}{0.45\textwidth}
Condition: $x_1=0$, $x_2=1$, $y=1$, $\mathrm{TA}^3_{3}$=E.\\
Therefore, Type I, $\neg x_{2} = 0$, $C_{3}= 1$.
\end{minipage}
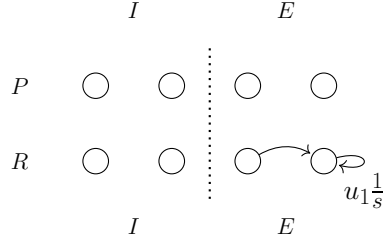
\begin{minipage}{0.35\textwidth}
\begin{tikzpicture}[node distance = .35cm, font=\Huge]
\tikzstyle{every node}=[scale=0.35]
\node[state] (E) at (1,1) {};
\node[state] (F) at (2,1) {};
\node[state] (G) at (3,1) {};
\node[state] (H) at (4,1) {};
\node[state] (A) at (1,2) {};
\node[state] (B) at (2,2) {};
\node[state] (C) at (3,2) {};
\node[state] (D) at (4,2) {};
\node[thick] at (0,1) {$R$};
\node[thick] at (0,2) {$P$};
\node[thick] at (1.5,3) {$I$};
\node[thick] at (3.5,3) {$E$};
\draw[dotted, thick] (2.5,0.5) -- (2.5,2.5);
\draw[every loop]
(G) edge[bend left] node [scale=1.2, above=0.1 of C] {} (H)
(H) edge[loop right] node [scale=1.2, below=0.1 of H] {$u_1\frac{1}{s}$} (H);

\end{tikzpicture}
\end{minipage}

\begin{minipage}{0.45\textwidth}
Condition: $x_1=0$, $x_2=1$, $y=1$, $\mathrm{TA}^3_{3}$=I.\\
Therefore, Type I, $\neg x_{2} = 0$, $C_{3}= 0$.
\end{minipage}
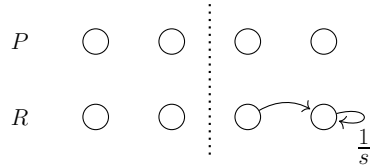
\begin{minipage}{0.35\textwidth}
\begin{tikzpicture}[node distance = .35cm, font=\Huge]
\tikzstyle{every node}=[scale=0.35]
\node[state] (E) at (1,1) {};
\node[state] (F) at (2,1) {};
\node[state] (G) at (3,1) {};
\node[state] (H) at (4,1) {};
\node[state] (A) at (1,2) {};
\node[state] (B) at (2,2) {};
\node[state] (C) at (3,2) {};
\node[state] (D) at (4,2) {};
\node[thick] at (0,1) {$R$};
\node[thick] at (0,2) {$P$};
\node[thick] at (1.5,3) {$I$};
\node[thick] at (3.5,3) {$E$};
\draw[dotted, thick] (2.5,0.5) -- (2.5,2.5);
\draw[every loop]
(G) edge[bend left] node [scale=1.2, above=0.1 of C] {} (H)
(H) edge[loop right] node [scale=1.2, below=0.1 of H] {$\frac{1}{s}$} (H);

\end{tikzpicture}
\end{minipage}

To summarize Case 4, we realize that both $\mathrm{TA}^3_{3}$ and $\mathrm{TA}^3_{4}$ will converge to ``Exclude''.

Based on the above analyses, we can summarize the transitions of $\mathrm{TA}^3_{3}$ and $\mathrm{TA}^3_{4}$, given different configurations of $\mathrm{TA}^3_{1}$ and $\mathrm{TA}^3_{2}$ in Case 1-Case 4 (i.e., given four different combinations of $x_{1}$ and $\neg x_{1}$). The arrow shown below means the direction of transitions. 

\textbf{Scenario 1:} Study $\mathrm{TA}^3_{3}$ = I and $\mathrm{TA}^3_{4}$ = I.

\begin{minipage}{0.5\textwidth}
\textbf{Case 1:} we can see that \\
$\mathrm{TA}^3_{3}$ $\rightarrow$ E \\
$\mathrm{TA}^3_{4}$ $\rightarrow$ E 
\end{minipage}
\begin{minipage}{0.5\textwidth}
\textbf{Case 2:} we can see that \\
$\mathrm{TA}^3_{3}$ $\rightarrow$ E \\
$\mathrm{TA}^3_{4}$ $\rightarrow$ E 
\end{minipage}

\vspace{.5cm}

\begin{minipage}{0.5\textwidth}
\textbf{Case 3:} we can see that \\
$\mathrm{TA}^3_{3}$ $\rightarrow$ E \\
$\mathrm{TA}^3_{4}$ $\rightarrow$ E 
\end{minipage}
\begin{minipage}{0.5\textwidth}
\textbf{Case 4:} we can see that \\
$\mathrm{TA}^3_{3}$ $\rightarrow$ E \\
$\mathrm{TA}^3_{4}$ $\rightarrow$ E 
\end{minipage}

From the facts presented above, it is confirmed that regardless of the state of $\mathrm{TA}^3_{1}$ and $\mathrm{TA}^3_{2}$,
if $\mathrm{TA}^3_{3}$=I and $\mathrm{TA}^3_{4}$=I, they ($\mathrm{TA}^3_{3}$ and $\mathrm{TA}^3_{4}$) will move towards the opposite half of the state space (i.e., towards ``Exclude" ), away from the current state. So, the state with $\mathrm{TA}^3_{3}$=I and $\mathrm{TA}^3_{4}$=I is not absorbing. 

\textbf{Scenario 2:} Study $\mathrm{TA}^3_{3}$ = I and $\mathrm{TA}^3_{4}$= E.

\begin{minipage}{0.5\textwidth}
\textbf{Case 1:} we can see that \\
$\mathrm{TA}^3_{3}$ $\rightarrow$ I \\
$\mathrm{TA}^3_{4}$ $\rightarrow$ E 
\end{minipage}
\begin{minipage}{0.5\textwidth}
\textbf{Case 2:} we can see that \\
$\mathrm{TA}^3_{3}$ $\rightarrow$ E \\
$\mathrm{TA}^3_{4}$ $\rightarrow$ I, E 
\end{minipage}

\vspace{.5cm}

\begin{minipage}{0.5\textwidth}
\textbf{Case 3:} we can see that \\
$\mathrm{TA}^3_{3}$ $\rightarrow$ I \\
$\mathrm{TA}^3_{4}$ $\rightarrow$ I, E 
\end{minipage}
\begin{minipage}{0.5\textwidth}
\textbf{Case 4:} we can see that \\
$\mathrm{TA}^3_{3}$ $\rightarrow$ E \\
$\mathrm{TA}^3_{4}$ $\rightarrow$ E 
\end{minipage}

Clearly, in this scenario, the starting point of $\mathrm{TA}^3_{3}$ is ``Include" and that of $\mathrm{TA}^3_{4}$ is ``Exclude". In Case 1, where $\mathrm{TA}^3_{1}$ = E and $\mathrm{TA}^3_{2}$ = I, $\mathrm{TA}^3_{3}$ will move towards ``Include" and $\mathrm{TA}^3_{4}$ will move towards ``Exclude". Therefore, given $\mathrm{TA}^3_{1}$ = E and $\mathrm{TA}^3_{2}$ = I hold, $\mathrm{TA}^3_{3}$ in ``Include" and $\mathrm{TA}^3_{4}$ in ``Exclude" are absorbing actions, while in other cases (i.e., in other configurations of $\mathrm{TA}^3_{1}$ and $\mathrm{TA}^3_{2}$), actions ``Include" and ``Exclude" for $\mathrm{TA}^3_{3}$ and $\mathrm{TA}^3_{4}$ are not absorbing. 

\textbf{Scenario 3:} Study $\mathrm{TA}^3_{3}$ = E and $\mathrm{TA}^3_{4}$ = I.

\begin{minipage}{0.5\textwidth}
\textbf{Case 1:} we can see that \\
$\mathrm{TA}^3_{3}$ $\rightarrow$ I, E \\
$\mathrm{TA}^3_{4}$ $\rightarrow$ E 
\end{minipage}
\begin{minipage}{0.5\textwidth}
\textbf{Case 2:} we can see that \\
$\mathrm{TA}^3_{3}$ $\rightarrow$ E \\
$\mathrm{TA}^3_{4}$ $\rightarrow$ E 
\end{minipage}

\vspace{.5cm}

\begin{minipage}{0.5\textwidth}
\textbf{Case 3:} we can see that \\
$\mathrm{TA}^3_{3}$ $\rightarrow$ I, E \\
$\mathrm{TA}^3_{4}$ $\rightarrow$ E 
\end{minipage}
\begin{minipage}{0.5\textwidth}
\textbf{Case 4:} we can see that \\
$\mathrm{TA}^3_{3}$ $\rightarrow$ E \\
$\mathrm{TA}^3_{4}$ $\rightarrow$ E 
\end{minipage}
From the transitions of $\mathrm{TA}^3_{3}$ and $\mathrm{TA}^3_{4}$ in Scenario 3, we can conclude that the state with $\mathrm{TA}^3_{3}$ = E and $\mathrm{TA}^3_{4}$ = I is not absorbing.

\textbf{Scenario 4:} Study $\mathrm{TA}^3_{3}$ = E and $\mathrm{TA}^3_{4}$ = E.

\begin{minipage}{0.5\textwidth}
\textbf{Case 1:} we can see that \\
$\mathrm{TA}^3_{3}$ $\rightarrow$ I \\
$\mathrm{TA}^3_{4}$ $\rightarrow$ E 
\end{minipage}
\begin{minipage}{0.5\textwidth}
\textbf{Case 2:} we can see that \\
$\mathrm{TA}^3_{3}$ $\rightarrow$ E \\
$\mathrm{TA}^3_{4}$ $\rightarrow$ I, E 
\end{minipage}

\vspace{.5cm}

\begin{minipage}{0.5\textwidth}
\textbf{Case 3:} we can see that \\
$\mathrm{TA}^3_{3}$ $\rightarrow$I \\
$\mathrm{TA}^3_{4}$ $\rightarrow$I, E
\end{minipage}
\begin{minipage}{0.5\textwidth}
\textbf{Case 4:} we can see that \\
$\mathrm{TA}^3_{3}$ $\rightarrow$ E \\
$\mathrm{TA}^3_{4}$ $\rightarrow$ E 
\end{minipage}

From the transitions of $\mathrm{TA}^3_{3}$ and $\mathrm{TA}^3_{4}$ in Scenario 4, we can conclude that the state with $\mathrm{TA}^3_{3}$ = E and $\mathrm{TA}^3_{4}$ = E is also absorbing in Case 4, when $\mathrm{TA}^3_{1}$ and $\mathrm{TA}^3_{2}$ have both actions as Include. 

From the above analysis, we can conclude that when we freeze $\mathrm{TA}^3_{1}$ and $\mathrm{TA}^3_{2}$ with certain actions, there are altogether two absorbing cases. (1) Given that $\mathrm{TA}^3_{1}$ selects ``Exclude" and $\mathrm{TA}^3_{2}$ selects ``Include", $\mathrm{TA}^3_{3}$ selects ``Include" and $\mathrm{TA}^3_{4}$ selects ``Exclude". (2) Given that $\mathrm{TA}^3_{1}$ selects ``Include" and $\mathrm{TA}^3_{2}$ selects ``Include", $\mathrm{TA}^3_{3}$ selects ``Exclude" and $\mathrm{TA}^3_{4}$ selects ``Exclude".

So far, we have finished half of the proof. More specifically,  we have studied the case when we freeze the transitions of TAs for the first input bit ($\mathrm{TA}^3_{1}$ and $\mathrm{TA}^3_{2}$) and examine the transitions of TAs for the second input bit (study $\mathrm{TA}^3_{3}$ and $\mathrm{TA}^3_{4}$ by frozen one of them and illustrate the transitions of the other one). In the following paragraphs, we will move on to the second half, i.e., we freeze $\mathrm{TA}^3_{4}$ and $\mathrm{TA}^3_{4}$ and study the transition of $\mathrm{TA}^3_{1}$ and $\mathrm{TA}^3_{2}$. 
The analysis procedure is similar to the one that has been done in the above paragraphs, as seen in the following.
\begin{enumerate}
\item We freeze $\mathrm{TA}^3_{3}$ and $\mathrm{TA}^3_{4}$ as ``Exclude" and ``Include". In this case, the second bit becomes $\neg x_{2}$. There are four sub-cases for $\mathrm{TA}^3_{1}$ and $\mathrm{TA}^3_{2}$. 
\begin{itemize}
\item We study the transition of $\mathrm{TA}^3_{1}$ when it has the action ``Include" as its current action, given different input training samples shown in Table \ref{xorlogichalf} and different actions of $\mathrm{TA}^3_{2}$ (i.e., when the action of $\mathrm{TA}^3_{2}$ is frozen as ``Include" or ``Exclude"). 
\item We study the transition of $\mathrm{TA}^3_{1}$ when it has the action ``Exclude" as its current action, given different input training samples shown in Table \ref{xorlogichalf} and different actions of $\mathrm{TA}^3_{2}$ (i.e., when the action of $\mathrm{TA}^3_{2}$ is frozen as ``Include" or ``Exclude"). 
\item We study the transition of $\mathrm{TA}^3_{2}$ when it has the action ``Include" as its current action, given different input training samples shown in Table \ref{xorlogichalf} and different actions of $\mathrm{TA}^3_{1}$ (i.e., when the action of $\mathrm{TA}^3_{1}$ is frozen as ``Include" or ``Exclude"). 
\item We study the transition of $\mathrm{TA}^3_{2}$ when it has the action ``Exclude" as its current action, given different input training samples shown in Table \ref{xorlogichalf} and different actions of $\mathrm{TA}^3_{1}$ (i.e., when the action of $\mathrm{TA}^3_{1}$ is frozen as ``Include" or ``Exclude"). 
\end{itemize}
\item We freeze $\mathrm{TA}^3_{3}$ and $\mathrm{TA}^3_{4}$ as ``Include" and ``Exclude". In this case, the second literal becomes $x_{2}$. The sub-cases for $\mathrm{TA}^3_{1}$ and $\mathrm{TA}^3_{2}$ are identical to the sub-cases in the previous case. 

\item We freeze $\mathrm{TA}^3_{3}$ and $\mathrm{TA}^3_{3}$ as ``Exclude" and ``Exclude". In this case, the second bit is excluded and will not influence the output. 
The sub-cases for $\mathrm{TA}^3_{1}$ and $\mathrm{TA}^3_{2}$ are identical to the sub-cases in the previous case. 
\item We freeze $\mathrm{TA}^3_{3}$ and $\mathrm{TA}^3_{4}$ as ``Include" and ``Include". In this case, the second bit will always be 0 and therefore $C_3=0$. The sub-cases for $\mathrm{TA}^3_{1}$ and $\mathrm{TA}^3_{2}$ are identical to the sub-cases in the previous case. 
\end{enumerate} 

When we go through all the possible transitions, we can conclude that (1) when $\mathrm{TA}^3_{3}$ and $\mathrm{TA}^3_{4}$ are frozen as ``Include'' and ``Exclude'', $\mathrm{TA}^3_{1}$=E and $\mathrm{TA}^3_{2}$=I are absorbing. (2) When $\mathrm{TA}^3_{3}$ and $\mathrm{TA}^3_{4}$ are frozen as ``Include'' and ``Include'', $\mathrm{TA}^3_{1}$=E and $\mathrm{TA}^3_{2}$=E are also absorbing. The detailed proof of this statement can be found in Appendix \ref{halfproof11}. 

When we look at the absorbing cases that are conditioned upon the frozen actions of both ``Include'' (i.e., when $\mathrm{TA}^3_{3}$ and $\mathrm{TA}^3_{4}$ are frozen as Include and Include, and when $\mathrm{TA}^3_{1}$ and $\mathrm{TA}^3_{2}$ are frozen as Include and Include), we can easily conclude that those conditions cannot be fulfilled and thus the corresponding system states are not absorbing. The main reason is that the condition of both ``Include'' is surely not absorbing and such state cannot be frozen as a stable state. Differently, for the state with $\mathrm{TA}^3_{1}$=E, $\mathrm{TA}^3_{2}$=I, $\mathrm{TA}^3_{3}$=I and $\mathrm{TA}^3_{4}$=E, the learning mechanism together with the training samples will reinforce the individual TA to move to a deeper state for the selected actions. Therefore, the state is indeed an absorbing state and it is the only absorbing state in the system. 
Therefore, given infinite time horizon, the TM will converge to the expected logic, which is half of the XOR-relation. We thus prove Lemma \ref{full1}. 
\end{proof}

Following the same strategy used for the proof in Lemma \ref{full1}, we can prove Lemma \ref{full2}. We do not detail the proof for the sake of brevity. 

\subsubsection{Proof of Lemma \ref{full3}}
\begin{proof}
To prove that the system is recurrent, we just need to show that the only absorbing state in the TM based on the training samples from Table \ref{xorlogichalf1} disappears when the training samples in Table \ref{xorlogicfull} is given. More specifically, once the TM is trained based on Table \ref{xorlogichalf1}, we will show that the absorbing state disappears when training sample $x_1=1$, $x_2=0$, and $y=1$ is given in addition. To validate this point, we can simply show that one of the TA, i.e., $\mathrm{TA}^3_3$ with an ``Including'' action will not move only towards ``Include'' when the output of the first
literal is $\neg x_1$, and when $\mathrm{TA}^3_4$ is ``Exclude'', given the newly added training sample. The diagram below indicates the transition of $\mathrm{TA}^3_3$ in the above mentioned condition when the new training sample is given. 


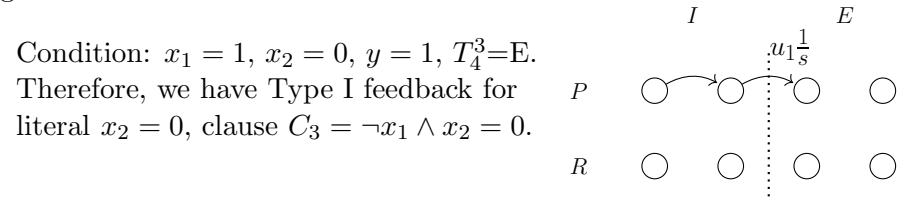
\begin{minipage}{0.45\textwidth}
Condition: $x_{1}=1$, $x_{2}=0$, $y=1$, $T^3_{4}$=E.\\
Therefore, we have Type I feedback for\\
literal $x_{2}=0$, clause $C_{3}= \neg x_{1} \wedge x_{2} = 0$.\\ 
\end{minipage}
\begin{minipage}{0.35\textwidth}
\begin{tikzpicture}[node distance = .35cm, font=\Huge]
\tikzstyle{every node}=[scale=0.35]
\node[state] (E) at (1,1) {};
\node[state] (F) at (2,1) {};
\node[state] (G) at (3,1) {};
\node[state] (H) at (4,1) {};
\node[state] (A) at (1,2) {};
\node[state] (B) at (2,2) {};
\node[state] (C) at (3,2) {};
\node[state] (D) at (4,2) {};
\node[thick] at (0,1) {$R$};
\node[thick] at (0,2) {$P$};
\node[thick] at (1.5,3) {$I$};
\node[thick] at (3.5,3) {$E$};
\draw[dotted, thick] (2.5,0.5) -- (2.5,2.5);
\draw[every loop]
(A) edge[bend left] node [scale=1.2, above=0.1 of C] {} (B)
(B) edge[bend left] node [scale=1.2, above=0.1 of C] {$~~~~~u_1\frac{1}{s}$} (C);

\end{tikzpicture}
\end{minipage}

Clearly, when $x_1=1$, $x_2=0$, and $y=1$ is given in addition, $\mathrm{TA}^3_3$ has a non-zero probability to move towards ``Exclude''. Therefore, ``Include'' is not the only direction that $\mathrm{TA}^3_3$ moves to upon the input, and this will make the state not absorbing any longer. For other states, the newly added training sample will not remove any transition from the previous case. Therefore, the system will not have any new absorbing state. Given the non-zero probability of returning II, IE, EI, EE, these system states are recurrent. 
\end{proof}
To summarize so far, from Lemma \ref{full1} and Lemma \ref{full2}, we understand that each individual clause is able to learn any sub-pattern from the XOR-relation. However, when the full logic of XOR is given, as shown in Table \ref{xorlogicfull}, the system state II, IE, EI and EE becomes recurrent. In other words, each clause may stay in any of the above states in probability that is less than 1. For this reason, it is necessary to have a parameter to guide the learning process of different clauses so that they can converge or cover different sub-patterns. The parameter is $T$, and the analysis is given presently. 

\subsubsection{Proof of Lemma \ref{full4} and Lemma \ref{full5}}
\begin{proof}
Consider $m$ clauses in the TM, and $m>T$.
Let's study parameter $u_1$ first. When $f_{\sum}(\mathcal{C}_i)$ is zero, $u_1=1/2$. When $0<f_{\sum}(\mathcal{C}_i)< T$, $u_1=\frac{T - \mathrm{max}(-T, \mathrm{min}(T, f_{\sum}(\mathcal{C}_i)))}{2T}=\frac{T - f_{\sum}(\mathcal{C}_i)}{2T}$, and $u_1$ monotonically decreases as $f_{\sum}(\mathcal{C}_i)$ increases. When $T\geq m$, $u_1=0$ holds. For $u_2$, when $f_{\sum}(\mathcal{C}_i)=0$, $u_2=1/2$. When $0<f_{\sum}(\mathcal{C}_i)< T$, 
$u_2=\frac{T + f_{\sum}(\mathcal{C}_i)}{2T}$, and it monotonically increases when $f_{\sum}(\mathcal{C}_i)$ grows. When $T\geq m$, $u_2=1$ holds. 

Clearly, $u_2>0$ always holds regardless the value of $f_{\sum}(\mathcal{C}_i)$. Therefore, the variations of $f_{\sum}(\mathcal{C}_i)$ does not change the directions of system transitions upon Type II feedback. 

To guarantee $u_1>0$, it is required $0\leq f_{\sum}(\mathcal{C}_i)<T$. When this condition fulfills, the variations of $f_{\sum}(\mathcal{C}_i)$ does not change the directions of system transitions upon Type I feedback.

According to the system updating rule of TM, once the Type I or Type II feedback is triggered, the clauses are updated independently. Due to the recurrent property, each clause will transit among II, IE, EI and EE, as long as $0\leq f_{\sum}(\mathcal{C}_i)<T$. In other words, each clause will move among those four status, until $T$ clauses follow the same sub-pattern. Therefore, consider infinite time horizon, the event that $T$, $T<m$, clauses appear in the same sub-pattern will almost surely happen.  

\end{proof}

When $f_{\sum}(\mathcal{C}_i)=T$, it means there are $T$ clauses that have followed the same sub-pattern. Once this happens, it means that there are certain number of clauses that have learnt a certain sub-pattern already and we would like to encourage the other clauses to learn the other sub-pattern. This is to be shown in Lemma \ref{full5}. 

\begin{proof}
Lemma \ref{full5} is self-evident. Clearly, when $f_{\sum}(\mathcal{C}_i)=T$, $u_1=0$ holds and thus Type I feedback will not be generated to the TM for any updates when the same training sample is given. 
For example, without loss of generalization, we assume there are $T$ clauses that have converged to $c_{3}=\neg x_{1} \wedge x_{2}=1$. When another training sample $x_1=0$, $x_2=1$, $y=1$ is given, Type I feedback will not be given any longer because $u_1=0$. Therefore, such input training sample is filtered out and the system will only update for Type I feedback when training sample $x_1=1$, $x_2=0$, $y=1$ is given (i.e., the TM will update based on the samples shown in Table~\ref{xorlogichalf1}, guiding the TM to learn the other sub-pattern. 
\end{proof}

\subsubsection{Proof of Theorem \ref{full6}}
\begin{proof} 
Based on Lemmas \ref{full1}-\ref{full5}, we can prove Theorem \ref{full6}. 


Clearly, $u_1$ monotonically decreases as $f_{\sum}(\mathcal{C}_i)$ increases. When the number of clauses that follow a certain sub-pattern increases, due to the monotonicity of $u_1$, the impact of such training samples becomes less and less to the system. 
Ultimately, when $f_{\sum}(\mathcal{C}_i)=T$ holds, the system will not be updated for the learnt sub-pattern. Therefore, at this particular time, only the other sub-pattern will be used for system training. This behavior can avoid the situation that many clauses learn one sub-pattern but the other sub-pattern is not learnt. 

Now let's consider the case where the selected $T$ is less than or equal to half of the number of the clauses, i.e., $T\leq m/2$. 
From Lemma \ref{full4}, we know that the system will eventually have $T$ clauses that follows one sub-pattern. Once this happens, due to Lemmas \ref{full1} and \ref{full2}, we understand that all clauses will move towards the only absorbing state of the system corresponds to the other sub-pattern. As soon as the number of the clauses that follow each sub-pattern reaches $T$, the system will not be updated any longer for any training input. In this situation, the system have been absorbed to the point where both sub-patterns of XOR have been learnt.


We thus complete the proof. 
\end{proof}
Note that the clauses that already follow a sub-pattern may get out of the pattern when training samples from the other sub-pattern are given. So even if there are $T$ clauses that are have learnt a sub-pattern, the number of learnt clauses may decrease in front of training samples from the other sub-pattern. However, as soon as the sum of the clauses for the same sub-pattern is less than $T$, the corresponding Type I feedback can be triggered again for this sub-pattern, leading the clauses to possibly move back to the sub-pattern again. Nevertheless, when the number of the clauses that follow each sub-pattern reaches $T$, the system is converged to the intended pattern. 

Note also that when $T$ is less than half of the number of the clauses, there are $m-2T$ clauses that do not follow to any of the sub-patterns when the system stops updating. Therefore, those clauses that do not follow the correct XOR sub-patterns may involve incorrect output if they all happen to follow a certain incorrect logic and the sum of them happens to be greater than or equal to $T$. For this reason, even if the absorbing states exist, the number of clauses, the threshold value $T$ need to be carefully chosen. 

\begin{remark}
\label{remark1}
The system configuration described in Theorem~\ref{theorem1} is a special case of Theorem~\ref{full6}. 
\end{remark}

\begin{remark}
\label{remark2}
When $T$ is greater than half of the number of the clauses, i.e., $T>m/2$, the system will not have any absorbing state. We conjuncture that the system can still learn the two sub-patterns in a balanced manner, as long as $T$ is not configured too close to the total number of clauses $m$ and when $s$ is sufficiently large. 
\end{remark}

To address the conjecture in Remark \ref{remark2}, we now study the system behavior when $T>m/2$. According to Lemma \ref{full3}, at a certain time slot, there will be $T$ clauses following a certain sub-pattern, named sub-pattern 1. In this situation, the corresponding training samples for sub-pattern~1 will not trigger any Type I feedback to the system and therefore the other training samples will guide the remaining clauses to learn towards the other sub-pattern, named sub-pattern 2. As the training process continues, all the clauses (including those $T$ clauses who have learned sub-pattern 1, although the action probability ($u_1/s$) is low) will lean towards sub-pattern 2. 

Because $m-T<T$, the clauses following sub-pattern 2 will not block the training samples for sub-pattern 2 even if there are $m-T$ clauses that follow this sub-pattern. Therefore, as more training samples are given, the remaining clauses will eventually move out of sub-pattern 1 or their current states and then move towards sub-pattern 2, until $T$ clauses follow sub-pattern~2 before the training samples for the sub-pattern 2 are completely blocked. Then the clauses will again move towards sub-pattern 1. 


The system will thus oscillate and will not be absorbed to a certain state. Nevertheless, with high probability, the system will have at least $m-T$ clauses that follow each sub-pattern, especially when $s$ is large. According to the updating rule of Type I feedback, the probability for an included literal in a clause that has learnt a certain sub-pattern to change towards the other sub-pattern is $u_1/s$, which only happens when a training sample of the other sub-pattern is given. On the other hand, the reward is $u_1\frac{s-1}{s}$ if a training sample of the same sub-pattern is received. Therefore, when $s$ is large, the clause is less likely to get out of the learnt sub-pattern due to a training sample from the conflicting sub-pattern. In other words, the system will most probably have at least $m-T$ clauses for each sub-pattern after training, and in the worst case, $m-2(m-T)$ clauses will appear in states other than any intended sub-pattern, depending on the stop time of training. To summarize, if we select the threshold as $m-T$, the two sub-patterns of the XOR-relation can still be followed with high probability. 

The statement of Remark \ref{remark2} has been validated via simulations\footnote{The code for XOR-relation can be found at \url{https://github.com/cair/TM-XOR-proof}.}. In the simulations, we configure 5 clauses, $s=10$, and $T=3$, and we train the TM for sufficient large number of samples. From the simulation results, we have confirmed that from all observations, each sub-pattern of XOR has been covered by at least two clauses. Not surprisingly, there are indeed a few observed cases that two clauses have followed the two distinct sub-patterns respectively and one clause is in a non-intended pattern, mostly excluding the two input bits. Nevertheless, because all-exclude clauses are ignored after training (Eq.  (\ref{eqn:clause1})),  the trained TM can still give the correct output of the XOR-relation in the latter case. 
\section{Conclusions}\label{conclusions}
In this paper, we complete the proof on the convergence of the XOR-relation. Firstly, we demonstrate that TM can almost surely learn the XOR-relation with the simplest structure. Thereafter, we analyze the dynamics of the system and reveal the relationship between the number of clauses and the threshold parameter $T$ when multiple sub-patterns exist. The analytical results not only confirm the convergence property of TM in XOR-relation, they also illustrate the role of the threshold parameter $T$ when multiple sub-patterns exist. 

\bibliographystyle{IEEEtran}
\bibliography{main}
\setcounter{section}{0}
\section{Appendix 1}\label{calculationDTMC}

\begin{algorithm}

\begin{algorithmic}[1]
\State \textbf{Input:} Training data (XOR: $\bold{X}$, $y$), $m = 2$, $o = 2$, Target $T$, Precision $s$
\State \textbf{Compute:} Transition Probability Matrix $M$
\State \textbf{Output:} Limiting Matrix 
\State \textbf{Initialize:}\\
\tab - $M$ \Comment{$M$ requires $2^8$ by $2^8$ space} \\
\tab - $\mathrm{TA}^t$ \Comment{TA action combinations of both clauses at time $t$} \\
\tab - $\mathrm{TA}^{t+1}$ \Comment{TA action combinations of both clauses at time $t+1$}
\State \textbf{Function:}
\For{$j = 1, ...,2^8$}
\For{$i = 1, ...,2^8$}
\For{$input = 1, ...,4$} \Comment{XOR contains 4 training samples}
\State $C_1 \leftarrow $ \textbf{Compute clause output} \Comment{Based on $input$ and $\mathrm{TA}^t$ actions in $C_1$}
\State $C_2 \leftarrow $ \textbf{Compute clause output} \Comment{Based on $input$ and $\mathrm{TA}^t$ actions in $C_2$}
\If{y = 1}
\State $FeedbackType$ = I
\State $P_{act}$ $\leftarrow$ $\frac{T-max(-T,min(T,C_1+C_2))}{2T}$
\Else
\State $FeedbackType$ = II
\State $P_{act}$ $\leftarrow$ $\frac{T+max(-T,min(T,C_1+C_2))}{2T}$
\EndIf\\
\If{$\mathrm{TA}^t_{C_1} = TA^{t+1}_{C_1}$} \Comment{No change in clause 1 from $t$ to $t+1$}
\State $Change$ = False
\For{$literal = 1, ..., 4$} \Comment{Each clause has four TAs}
\State $P_{feed} \leftarrow$ \textbf{Update} \Comment{Table \ref{feedbackprobabilities}}
\EndFor
\State $P_{transC_1} \leftarrow $ \textbf{Compute} \Comment{Eq.~(\ref{ptransnochange})}
\Else \Comment{There is a change in clause 1 from $t$ to $t+1$}
\For{$literal = 1, ..., 4$} \Comment{Each clause has four TAs}
\If{$\mathrm{TA}^t_{C_1}(literal) = TA^{t+1}_{C_1}(literal)$}
\State $Change$ = False
\Else
\State $Change$ = True
\EndIf
\State $P_{feed} \leftarrow$ \textbf{Update} \Comment{Table~\ref{feedbackprobabilities}}
\EndFor
\State $P_{transC_1} \leftarrow $ \textbf{Compute} \Comment{Eq.~(\ref{ptranschange})}
\EndIf\\
\If{$\mathrm{TA}^t_{C_2} = TA^{t+1}_{C_2}$} \Comment{No change in clause 2 from $t$ to $t+1$}
\State $Change$ = False
\For{$literal = 1, ..., 4$} \Comment{Each clause has four TAs}
\State $P_{feed} \leftarrow$ \textbf{Update} \Comment{Table~\ref{feedbackprobabilities}}
\EndFor
\State $P_{transC_2} \leftarrow $ \textbf{Compute} \Comment{Eq.~(\ref{ptransnochange})}
\Else \Comment{There is a change in clause 2 from $t$ to $t+1$}
\For{$literal = 1, ..., 4$} \Comment{Each clause has four TAs}
\If{$\mathrm{TA}^t_{C_2}(literal) = TA^{t+1}_{C_2}(literal)$}
\State $Change$ = False
\Else
\State $Change$ = True
\EndIf
\caption{Algorithm for calculating the transitions of the DTMC.} 
\algstore{myalg}
\end{algorithmic}
\end{algorithm}
\setcounter{algorithm}{0} 
\begin{algorithm} 
\begin{algorithmic} [1] 
\algrestore{myalg}
\State $P_{feed} \leftarrow$ \textbf{Update} \Comment{Table~\ref{feedbackprobabilities}}
\EndFor
\State $P_{transC_2} \leftarrow $ \textbf{Compute} \Comment{Eq.~(\ref{ptranschange})}
\EndIf\\
\State $P_{TotalTrans} \leftarrow$ \textbf{Compute} \Comment{Eq.~(\ref{ptota})}
\EndFor 
\EndFor 
\State $M[i,j] \leftarrow$ \textbf{Update} \Comment{$M[i,j] = P_{TotalTrans}$}
\EndFor\\ 
\State \textbf{End Function}
\State \textbf{Return: [Transpose of $M]^\infty$ }
\Comment{Compute the transpose of $M$ and return the result of power infinity}
\end{algorithmic}
\caption{Algorithm for calculating the transitions of the DTMC.}
\label{alg:transition}
\end{algorithm}

The step-by-step procedure for calculating the limiting matrix of the DTMC for the XOR-relation can be found in Algorithm \ref{alg:transition}. For ease of observation, we summarize Type I and Type II feedback in one table, as shown in Table~\ref{feedbackprobabilities}. 

\textbf{Line 1:} The algorithm takes the set of training examples ($\bold{X}$, $y$). The hyper parameters, i.e., the number of clauses, $m$,  the number of features, $o$, the target parameter, $T$, and the precision parameter $s$ have to be set at the start of the algorithm.

\textbf{Lines 2-3:} The goal of the algorithm is to compute transition probabilities for all possible transitions of system states from time $t$ to $t+1$, i.e., for 1-step, and store them in matrix $M$. Then the limiting matrix, which is the infinite power of the transpose of $M$, will be returned. 

\textbf{Lines 4-7:} Matrices $M$, $TA^t$, and $TA^{t+1}$ are initialized. $M$ is the transition probability matrix, with size $2^8\times2^8$.  $TA^t$ is the TA action combination at time $t$, i.e., the system state at time $t$. $TA^{t+1}$ represents the system state at time $t+1$.

\textbf{Lines 8-62:} Each clause contains four TAs. Hence, there are eight TAs in total in two clauses. An action of a TA can be either $include$ or $exclude$. Therefore, there are $2^8$ possible action combinations at time $t$. There is a possibility of changing from the TA action combination (the system state) at time $t$ to a TA action combination from the set of $2^8$ possible action combinations at time $t+1$. Here, we compute the transition probability of moving from any TA action combination at time $t$ to another possible TA action combination at time $t+1$, and the probability is called $P_{TotalTrans}$. Accordingly, the matrix $M$, which is composed by $P_{TotalTrans}$,  is updated.

\begin{itemize}
    \item \textbf{Line 9:} $j$ represents the index of a certain TA action combination at time $t$.
    \item \textbf{Line 10:} $i$ represents the index of a certain TA action combination at time $t+1$.
    \item \textbf{Line 11:} In order to calculate the transition probabilities at each possible transition, the TM receives all possible input samples of XOR, which is four.
    \item \textbf{Lines 12-13:} Based on inputs and the system state at time $t$, clause outputs of clauses $C_1$ and $C_2$ are calculated.
    \item \textbf{Lines 14-20:} The type of the feedback, i.e., Type I or Type II and the activation probability, $P_{act}$ for receiving a feedback by the clause are determined. 
    \item \textbf{Lines 22-27:} Those lines calculate the probability of transitions when there is no change of TA actions from time $t$ to time $t+1$ for clause 1.  Here, $\mathrm{TA}^t_{C_1}$ is the TA combination of clause 1 at time $t$. The feedback probability that each TA in clause 1 receives, i.e., $P_{feed(TA_k)}$, $k\in\{1, \ldots 4\}$, is selected from Table \ref{feedbackprobabilities}. When there is no transition, possible feedback options are \textit{reward} or \textit{inaction}. For example, for $\mathrm{TA}^1_1$, if the $FeedbackType$ is Type I, the output of $C_1$ is 1, the literal of the considered TA is 1, and the current TA decision is to \textit{include} the corresponding literal in the clause. From the first column in the probability section of the Table \ref{feedbackprobabilities}, we can find the feedback probability for this particular TA, $P_{feed(TA_1)}$ as $\frac{(s-1)}{s} + \frac{1}{s}$. Following the same concept, if the TA decision is to \textit{exclude} the corresponding literal from the clause, the feedback probability is therefore $\frac{1}{s}$.  The $P_{feed}$ in \textbf{Line 25} is the multiplication of probabilities received by all TAs in clause 1, as the feedback is given independently to each TA following a certain probability. 
    
    \begin{equation}\label{pfeedback}
        P_{feed} = P_{feed(TA_1)} \times P_{feed(TA_2)} \times P_{feed(TA_3)} \times P_{feed(TA_4)}.
    \end{equation}
    
    $P_{transC_1}$ is then calculated as,
    \begin{equation}\label{ptransnochange}
        P_{transC_1} = (P_{act} \times P_{feed}) + (1-P_{act}),
    \end{equation}
    \noindent where $P_{act}$ represents the probability that the feedback is given. It is $u_1$ for Type I feedback and $u_2$ for Type II feedback. 
  
   \item \textbf{Lines 28-38:} Those lines calculate the probability of transitions when there is a change of TA actions from time $t$ to time $t+1$ for clause 1.   
   Here, in order to calculate the correct transition probability, we need to check for each literal if there is any change, and calculate feedback probability for each of them accordingly. Then $P_{feed}$ can be calculated in the same way to $P_{feed}$ in Eq.(\ref{pfeedback}). However, the calculation of $P_{transC_1}$ is slightly different from Eq.~(\ref{ptransnochange}),  as shown in Eq.~(\ref{ptranschange}). Here $P_{act}$ is multiplied by $P_{feed}$ as any feedback must be initiated for any change in clause 1. 
    \begin{equation}\label{ptranschange}
        P_{transC_1} = P_{act} \times P_{feed}.
    \end{equation}
    \item \textbf{Lines 40-56:} Similar calculations on clause 2 are performed to find $P_{transC_2}$
    \item \textbf{Line 58:} The total transition probability of moving from the current TA state combination, i.e., at time $t$, to the next time instant, i.e., at time $t+1$, is then calculated, as,
    \begin{equation}\label{ptota}
        P_{TotalTrans} = 0.25 \times P_{transC_1} \times P_{transC_2}.
    \end{equation}
    \noindent where $0.25$ means that the probability of any type of training sample is equal.
    \item \textbf{Line 61:} At the end of each $i$ in the loop of $j$, the transition probability matrix, $M$ is updated, i.e., $M[i,j] = P_{TotalTrans}$.
\end{itemize}
 \textbf{Line 65:} The algorithm returns the transition probability matrix after it is transposed and multiplied by itself with infinity number of times, 
 which is indeed the limiting matrix.

\begin{table}[t]
\centering
\newcolumntype{P}[1]{>{\centering\arraybackslash}p{#1}}
\begin{tabular}{|c|c|c|c|c|c|c|c|c|c|c|}
\hline
\multicolumn{3}{|c|}{Feedback Type} & \multicolumn{4}{c|}{I} & \multicolumn{4}{c|}{II}  \\ \hline
\multicolumn{3}{|c|}{Clause Output} & \multicolumn{2}{c|}{1} & \multicolumn{2}{c|}{0} & \multicolumn{2}{c|}{1} & \multicolumn{2}{c|}{0} \\ \hline
\multicolumn{3}{|c|}{Literal Value} & 1 & 0 & 1 & 0 & 1 & 0 & 1 & 0 \\ \hline
\hline
\multirow{6}{*}{\rotatebox[origin=c]{90}{Current State}} & \multirow{3}{*}{Include} & Reward Probability & (s-1)/s & NA & 0 & 0 & 0 & NA & 0 & 0 \\
                        &                     & Inaction Probability & 1/s & NA & (s-1)/s & (s-1)/s & 1 & NA & 1 & 1\\
                        &                     & Penalty Probability & 0 & NA & 1/s & 1/s & 0 & NA & 0 & 0 \\\cline{2-11}
                        & \multirow{3}{*}{Exclude} & Reward Probability & 0 & 1/s & 1/s & 1/s & 0 & 0 & 0 & 0\\
                        &                     & Inaction Probability & 1/s & (s-1)/s & (s-1)/s & (s-1)/s & 1 & 0 & 1 & 1\\
                        &                     & Penalty Probability & (s-1)/s & 0 & 0 & 0 & 0 & 1 & 0 & 0       \\ 
\hline
\end{tabular}
\caption{Type I and Type II feedback.}\label{feedbackprobabilities}
\end{table}

\section{Appendix 2}\label{halfproof11}

In this appendix, we freeze the actions of $\mathrm{TA}^3_3$ and $\mathrm{TA}^3_4$ and study the transitions of $\mathrm{TA}^3_1$ and $\mathrm{TA}^3_2$. \\
\\
\noindent {\bf Case 1} \\

Here $\mathrm{TA}^3_3$ is frozen as ``Exclude" and $\mathrm{TA}^3_4$ is ``Include". In this situation, the output of $\mathrm{TA}^3_3$ and $\mathrm{TA}^3_4$ is $\neg x_{2}$.

{\bf We firstly study $\mathrm{TA}^3_{1}$ with action ``Include''.}

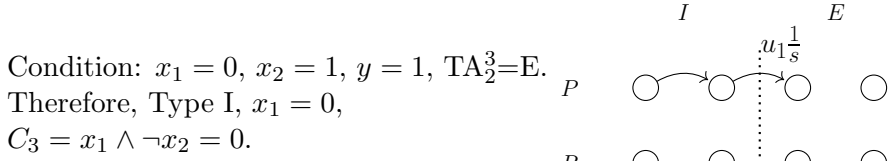
\begin{minipage}{0.45\textwidth}
Condition: $x_{1}=0$, $x_{2}=1$, $y=1$, $\mathrm{TA}^3_{2}$=E. \\
Therefore, Type I, $x_{1} = 0$, \\$C_3=x_{1} \wedge \neg x_{2}=0$.
\end{minipage}
\begin{minipage}{0.35\textwidth}
\begin{tikzpicture}[node distance = .35cm, font=\Huge]
\tikzstyle{every node}=[scale=0.35]
\node[state] (E) at (1,1) {};
\node[state] (F) at (2,1) {};
\node[state] (G) at (3,1) {};
\node[state] (H) at (4,1) {};
\node[state] (A) at (1,2) {};
\node[state] (B) at (2,2) {};
\node[state] (C) at (3,2) {};
\node[state] (D) at (4,2) {};
\node[thick] at (0,1) {$R$};
\node[thick] at (0,2) {$P$};
\node[thick] at (1.5,3) {$I$};
\node[thick] at (3.5,3) {$E$};
\draw[dotted, thick] (2.5,0.5) -- (2.5,2.5);
\draw[every loop]
(A) edge[bend left] node [scale=1.2, above=0.1 of C] {} (B)
(B) edge[bend left] node [scale=1.2, above=0.1 of C] {$~~~~~u_1\frac{1}{s}$} (C);

\end{tikzpicture}
\end{minipage}

\begin{minipage}{0.45\textwidth}
Condition: $x_{1}=0$, $x_{2}=1$, $y=1$, $\mathrm{TA}^3_{2}$=I. \\
Therefore, Type I, $x_{1} = 0$, $C_3=0$. 
\end{minipage}
\begin{minipage}{0.35\textwidth}
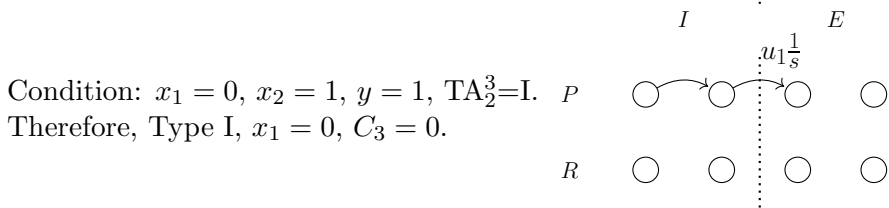
\begin{tikzpicture}[node distance = .35cm, font=\Huge]
\tikzstyle{every node}=[scale=0.35]
\node[state] (E) at (1,1) {};
\node[state] (F) at (2,1) {};
\node[state] (G) at (3,1) {};
\node[state] (H) at (4,1) {};
\node[state] (A) at (1,2) {};
\node[state] (B) at (2,2) {};
\node[state] (C) at (3,2) {};
\node[state] (D) at (4,2) {};
\node[thick] at (0,1) {$R$};
\node[thick] at (0,2) {$P$};
\node[thick] at (1.5,3) {$I$};
\node[thick] at (3.5,3) {$E$};
\draw[dotted, thick] (2.5,0.5) -- (2.5,2.5);
\draw[every loop]
(A) edge[bend left] node [scale=1.2, above=0.1 of C] {} (B)
(B) edge[bend left] node [scale=1.2, above=0.1 of C] {$~~~~~u_1\frac{1}{s}$} (C);

\end{tikzpicture}
\end{minipage}

{\bf We now study $\mathrm{TA}^3_{1}$ with action ``Exclude''.}

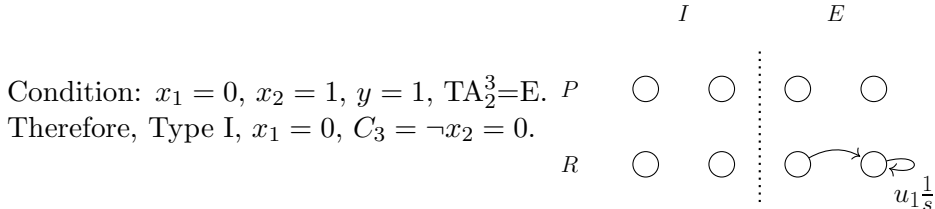
\begin{minipage}{0.45\textwidth}
Condition: $x_{1}=0$, $x_{2}=1$, $y=1$, $\mathrm{TA}^3_{2}$=E. \\
Therefore, Type I, $x_{1} = 0$, $C_3=\neg x_{2}=0$. 
\end{minipage}
\begin{minipage}{0.35\textwidth}
\begin{tikzpicture}[node distance = .35cm, font=\Huge]
\tikzstyle{every node}=[scale=0.35]
\node[state] (E) at (1,1) {};
\node[state] (F) at (2,1) {};
\node[state] (G) at (3,1) {};
\node[state] (H) at (4,1) {};
\node[state] (A) at (1,2) {};
\node[state] (B) at (2,2) {};
\node[state] (C) at (3,2) {};
\node[state] (D) at (4,2) {};
\node[thick] at (0,1) {$R$};
\node[thick] at (0,2) {$P$};
\node[thick] at (1.5,3) {$I$};
\node[thick] at (3.5,3) {$E$};
\draw[dotted, thick] (2.5,0.5) -- (2.5,2.5);
\draw[every loop]
(G) edge[bend left] node [scale=1.2, above=0.1 of C] {} (H)
(H) edge[loop right] node [scale=1.2, below=0.1 of H] {$u_1\frac{1}{s}$} (H);

\end{tikzpicture}
\end{minipage}

\begin{minipage}{0.45\textwidth}
Condition: $x_{1}=0$, $x_{2}=0$, $y=0$, $\mathrm{TA}^3_{2}$=E. \\
Therefore, Type II, $x_{1} =0$, $C_3=\neg x_{2}=1$. 
\end{minipage}
\begin{minipage}{0.35\textwidth}
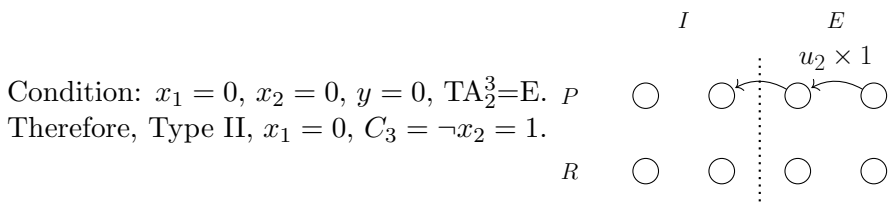
\begin{tikzpicture}[node distance = .35cm, font=\Huge]
\tikzstyle{every node}=[scale=0.35]
\node[state] (E) at (1,1) {};
\node[state] (F) at (2,1) {};
\node[state] (G) at (3,1) {};
\node[state] (H) at (4,1) {};
\node[state] (A) at (1,2) {};
\node[state] (B) at (2,2) {};
\node[state] (C) at (3,2) {};
\node[state] (D) at (4,2) {};
\node[thick] at (0,1) {$R$};
\node[thick] at (0,2) {$P$};
\node[thick] at (1.5,3) {$I$};
\node[thick] at (3.5,3) {$E$};
\draw[dotted, thick] (2.5,0.5) -- (2.5,2.5);
\draw[every loop]
(D) edge[bend right] node [scale=1.2, above=0.1 of C] {$u_2\times 1$} (C)
(C) edge[bend right] node [scale=1.2, above=0.1 of B] {} (B);

\end{tikzpicture}
\end{minipage}

\begin{minipage}{0.45\textwidth}
Condition: $x_{1}=0$, $x_{2}=1$, $y=1$, $\mathrm{TA}^3_{2}$=I. \\
Therefore, Type I, $x_{1}=0$, \\$C_3=\neg x_{1} \wedge \neg x_{2}=0$. 
\end{minipage}
\begin{minipage}{0.35\textwidth}
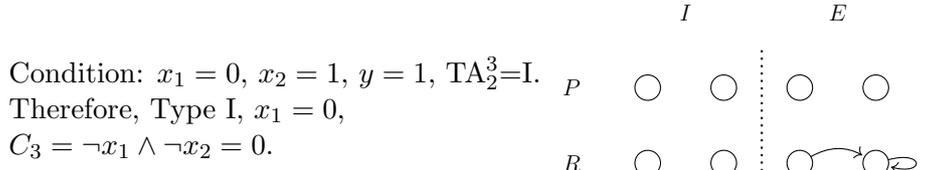
\begin{tikzpicture}[node distance = .35cm, font=\Huge]
\tikzstyle{every node}=[scale=0.35]
\node[state] (E) at (1,1) {};
\node[state] (F) at (2,1) {};
\node[state] (G) at (3,1) {};
\node[state] (H) at (4,1) {};
\node[state] (A) at (1,2) {};
\node[state] (B) at (2,2) {};
\node[state] (C) at (3,2) {};
\node[state] (D) at (4,2) {};
\node[thick] at (0,1) {$R$};
\node[thick] at (0,2) {$P$};
\node[thick] at (1.5,3) {$I$};
\node[thick] at (3.5,3) {$E$};
\draw[dotted, thick] (2.5,0.5) -- (2.5,2.5);
\draw[every loop]
(G) edge[bend left] node [scale=1.2, above=0.1 of C] {} (H)
(H) edge[loop right] node [scale=1.2, below=0.1 of H] {$u_1\frac{1}{s}$} (H);

\end{tikzpicture}
\end{minipage}

\begin{minipage}{0.45\textwidth}
Condition: $x_{1}=0$, $x_{2}=0$, $y=0$, $\mathrm{TA}^3_{2}$=I. \\
Therefore, Type II, $x_{1} = 0$, \\$C_3=\neg x_{1} \wedge \neg x_{2}=1$.
\end{minipage}
\begin{minipage}{0.35\textwidth}
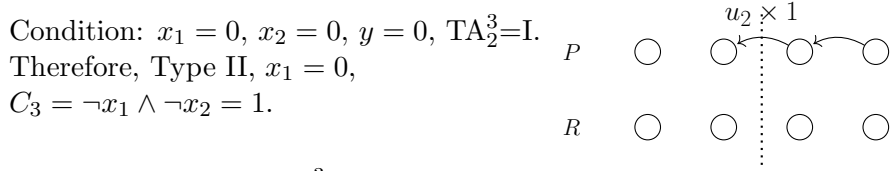
\begin{tikzpicture}[node distance = .35cm, font=\Huge]
\tikzstyle{every node}=[scale=0.35]
\node[state] (E) at (1,1) {};
\node[state] (F) at (2,1) {};
\node[state] (G) at (3,1) {};
\node[state] (H) at (4,1) {};
\node[state] (A) at (1,2) {};
\node[state] (B) at (2,2) {};
\node[state] (C) at (3,2) {};
\node[state] (D) at (4,2) {};
\node[thick] at (0,1) {$R$};
\node[thick] at (0,2) {$P$};
\node[thick] at (1.5,3) {$I$};
\node[thick] at (3.5,3) {$E$};
\draw[dotted, thick] (2.5,0.5) -- (2.5,2.5);
\draw[every loop]
(D) edge[bend right] node [scale=1.2, above=0.1 of C] {} (C)
(C) edge[bend right] node [scale=1.2, above=0.1 of B] {$u_2\times1$} (B);
\end{tikzpicture}
\end{minipage}


{\bf We thirdly study $\mathrm{TA}^3_{2}$ with action ``Include''.} 

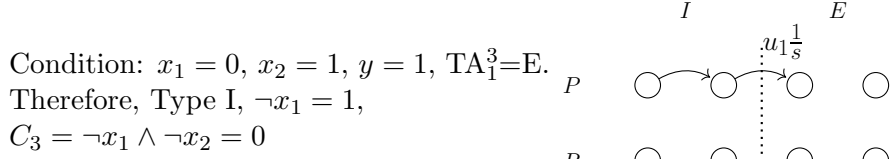
\begin{minipage}{0.45\textwidth}
Condition: $x_{1}=0$, $x_{2}=1$, $y=1$, $\mathrm{TA}^3_{1}$=E. \\
Therefore, Type I, $\neg x_{1}=1$, \\$C_3=\neg x_{1} \wedge \neg x_{2}=0$
\end{minipage}
\begin{minipage}{0.35\textwidth}
\begin{tikzpicture}[node distance = .35cm, font=\Huge]
\tikzstyle{every node}=[scale=0.35]
\node[state] (E) at (1,1) {};
\node[state] (F) at (2,1) {};
\node[state] (G) at (3,1) {};
\node[state] (H) at (4,1) {};
\node[state] (A) at (1,2) {};
\node[state] (B) at (2,2) {};
\node[state] (C) at (3,2) {};
\node[state] (D) at (4,2) {};
\node[thick] at (0,1) {$R$};
\node[thick] at (0,2) {$P$};
\node[thick] at (1.5,3) {$I$};
\node[thick] at (3.5,3) {$E$};
\draw[dotted, thick] (2.5,0.5) -- (2.5,2.5);
\draw[every loop]
(A) edge[bend left] node [scale=1.2, above=0.1 of C] {} (B)
(B) edge[bend left] node [scale=1.2, above=0.1 of C] {$~~~~~u_1\frac{1}{s}$} (C);

\end{tikzpicture}
\end{minipage}

\begin{minipage}{0.45\textwidth}
Condition: $x_{1}=0$, $x_{2}=1$, $y=1$, $\mathrm{TA}^3_1$=I. \\
Therefore, Type I, $\neg x_{1}=1$, $C_3=0$
\end{minipage}
\begin{minipage}{0.35\textwidth}
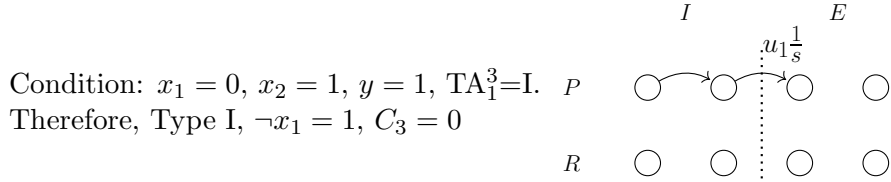
\begin{tikzpicture}[node distance = .35cm, font=\Huge]
\tikzstyle{every node}=[scale=0.35]
\node[state] (E) at (1,1) {};
\node[state] (F) at (2,1) {};
\node[state] (G) at (3,1) {};
\node[state] (H) at (4,1) {};
\node[state] (A) at (1,2) {};
\node[state] (B) at (2,2) {};
\node[state] (C) at (3,2) {};
\node[state] (D) at (4,2) {};
\node[thick] at (0,1) {$R$};
\node[thick] at (0,2) {$P$};
\node[thick] at (1.5,3) {$I$};
\node[thick] at (3.5,3) {$E$};
\draw[dotted, thick] (2.5,0.5) -- (2.5,2.5);
\draw[every loop]
(A) edge[bend left] node [scale=1.2, above=0.1 of C] {} (B)
(B) edge[bend left] node [scale=1.2, above=0.1 of C] {$~~~~~u_1\frac{1}{s}$} (C);
\end{tikzpicture}
\end{minipage}


{\bf We finally study $\mathrm{TA}^3_{2}$ with action ``Exclude''.}

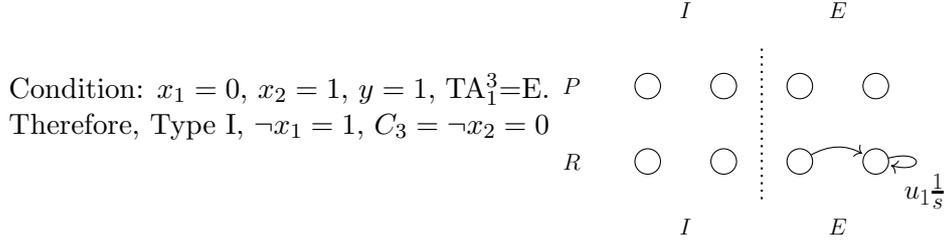
\begin{minipage}{0.45\textwidth}
Condition: $x_{1}=0$, $x_{2}=1$, $y=1$, $\mathrm{TA}^3_{1}$=E. \\
Therefore, Type I, $\neg x_{1}=1$, $C_3=\neg x_{2}=0$
\end{minipage}
\begin{minipage}{0.35\textwidth}
\begin{tikzpicture}[node distance = .35cm, font=\Huge]
\tikzstyle{every node}=[scale=0.35]
\node[state] (E) at (1,1) {};
\node[state] (F) at (2,1) {};
\node[state] (G) at (3,1) {};
\node[state] (H) at (4,1) {};
\node[state] (A) at (1,2) {};
\node[state] (B) at (2,2) {};
\node[state] (C) at (3,2) {};
\node[state] (D) at (4,2) {};
\node[thick] at (0,1) {$R$};
\node[thick] at (0,2) {$P$};
\node[thick] at (1.5,3) {$I$};
\node[thick] at (3.5,3) {$E$};
\draw[dotted, thick] (2.5,0.5) -- (2.5,2.5);
\draw[every loop]
(G) edge[bend left] node [scale=1.2, above=0.1 of C] {} (H)
(H) edge[loop right] node [scale=1.2, below=0.1 of H] {$~~u_1\frac{1}{s}$} (H);

\end{tikzpicture}
\end{minipage}

\begin{minipage}{0.45\textwidth}
Condition: $x_{1}=0$, $x_{2}=1$, $y=1$, $\mathrm{TA}^3_{1}$=I. \\
Therefore, Type I, $\neg x_{1}=1$, \\$C_3= x_{1} \wedge \neg x_{2}=0$. 
\end{minipage}
\begin{minipage}{0.35\textwidth}
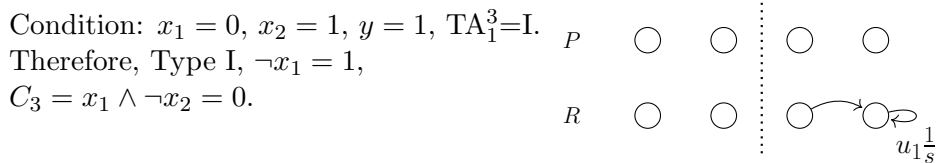
\begin{tikzpicture}[node distance = .35cm, font=\Huge]
\tikzstyle{every node}=[scale=0.35]
\node[state] (E) at (1,1) {};
\node[state] (F) at (2,1) {};
\node[state] (G) at (3,1) {};
\node[state] (H) at (4,1) {};
\node[state] (A) at (1,2) {};
\node[state] (B) at (2,2) {};
\node[state] (C) at (3,2) {};
\node[state] (D) at (4,2) {};
\node[thick] at (0,1) {$R$};
\node[thick] at (0,2) {$P$};
\node[thick] at (1.5,3) {$I$};
\node[thick] at (3.5,3) {$E$};
\draw[dotted, thick] (2.5,0.5) -- (2.5,2.5);
\draw[every loop]
(G) edge[bend left] node [scale=1.2, above=0.1 of C] {} (H)
(H) edge[loop right] node [scale=1.2, below=0.1 of H] {$u_1\frac{1}{s}$} (H);

\end{tikzpicture}
\end{minipage}



\noindent {\bf Case 2} \\
Here $\mathrm{TA}^3_3$ is frozen as ``Include" and $\mathrm{TA}^3_4$ is as ``Exclude". In this situation, the output of $\mathrm{TA}^3_3$ and $\mathrm{TA}^3_4$ is $x_{2}$.

{\bf We now study $\mathrm{TA}^3_{1}$ with action ``Include''.}

\begin{minipage}{0.45\textwidth}
Condition: $x_{1}=0$, $x_{2}=1$, $y=1$, $\mathrm{TA}^3_{2}$=E. \\
Therefore, Type I, $x_{1} = 0$, $x_{1} = 0$, $C_{3}=x_{1} \wedge x_{2}=0$. 
\end{minipage}
\begin{minipage}{0.35\textwidth}
\begin{tikzpicture}[node distance = .35cm, font=\Huge]
\tikzstyle{every node}=[scale=0.35]
\node[state] (E) at (1,1) {};
\node[state] (F) at (2,1) {};
\node[state] (G) at (3,1) {};
\node[state] (H) at (4,1) {};
\node[state] (A) at (1,2) {};
\node[state] (B) at (2,2) {};
\node[state] (C) at (3,2) {};
\node[state] (D) at (4,2) {};
\node[thick] at (0,1) {$R$};
\node[thick] at (0,2) {$P$};
\node[thick] at (1.5,3) {$I$};
\node[thick] at (3.5,3) {$E$};
\draw[dotted, thick] (2.5,0.5) -- (2.5,2.5);
\draw[every loop]
(A) edge[bend left] node [scale=1.2, above=0.1 of C] {} (B)
(B) edge[bend left] node [scale=1.2, above=0.1 of C] {$~~~~~u_1\frac{1}{s}$} (C);

\end{tikzpicture}
\end{minipage}

\begin{minipage}{0.45\textwidth}
Condition: $x_{1}=0$, $x_{2}=1$, $y=1$, $\mathrm{TA}^3_{2}$=I. \\
Therefore, Type I, $x_{1} = 0$, \\$C_{3}=\neg x_{1} \wedge x_{1}\wedge x_{2}=1$. 
\end{minipage}
\begin{minipage}{0.35\textwidth}
\begin{tikzpicture}[node distance = .35cm, font=\Huge]
\tikzstyle{every node}=[scale=0.35]
\node[state] (E) at (1,1) {};
\node[state] (F) at (2,1) {};
\node[state] (G) at (3,1) {};
\node[state] (H) at (4,1) {};
\node[state] (A) at (1,2) {};
\node[state] (B) at (2,2) {};
\node[state] (C) at (3,2) {};
\node[state] (D) at (4,2) {};
\node[thick] at (0,1) {$R$};
\node[thick] at (0,2) {$P$};
\node[thick] at (1.5,3) {$I$};
\node[thick] at (3.5,3) {$E$};
\draw[dotted, thick] (2.5,0.5) -- (2.5,2.5);
\draw[every loop]
(A) edge[bend left] node [scale=1.2, above=0.1 of C] {} (B)
(B) edge[bend left] node [scale=1.2, above=0.1 of C] {$~~~~~u_1\frac{1}{s}$} (C);

\end{tikzpicture}
\end{minipage}


\newpage
{\bf We now study $\mathrm{TA}^3_{1}$ with action ``Exclude''.}

\begin{minipage}{0.45\textwidth}
Condition: 
$x_{1}=0$, $x_{2}=1$, $y=1$, $\mathrm{TA}^3_{2}$=E. \\
Therefore, Type I, $x_{1} = 0$, $C_{3}=x_{1}=1$. 
\end{minipage}
\begin{minipage}{0.35\textwidth}
\begin{tikzpicture}[node distance = .35cm, font=\Huge]
\tikzstyle{every node}=[scale=0.35]
\node[state] (E) at (1,1) {};
\node[state] (F) at (2,1) {};
\node[state] (G) at (3,1) {};
\node[state] (H) at (4,1) {};
\node[state] (A) at (1,2) {};
\node[state] (B) at (2,2) {};
\node[state] (C) at (3,2) {};
\node[state] (D) at (4,2) {};
\node[thick] at (0,1) {$R$};
\node[thick] at (0,2) {$P$};
\node[thick] at (1.5,3) {$I$};
\node[thick] at (3.5,3) {$E$};
\draw[dotted, thick] (2.5,0.5) -- (2.5,2.5);
\draw[every loop]
(G) edge[bend left] node [scale=1.2, above=0.1 of C] {} (H)
(H) edge[loop right] node [scale=1.2, below=0.1 of H] {$u_1\frac{1}{s}$} (H);

\end{tikzpicture}
\end{minipage}

\begin{minipage}{0.45\textwidth}
Condition: $x_{1}=0$, $x_{2}=1$, $y=1$, $\mathrm{TA}^3_{2}$=I. \\
Therefore, Type I, $x_{1} = 0$, \\$C_{3}=\neg x_{1}\wedge x_{2}=1$. 
\end{minipage}
\begin{minipage}{0.35\textwidth}
\begin{tikzpicture}[node distance = .35cm, font=\Huge]
\tikzstyle{every node}=[scale=0.35]
\node[state] (E) at (1,1) {};
\node[state] (F) at (2,1) {};
\node[state] (G) at (3,1) {};
\node[state] (H) at (4,1) {};
\node[state] (A) at (1,2) {};
\node[state] (B) at (2,2) {};
\node[state] (C) at (3,2) {};
\node[state] (D) at (4,2) {};
\node[thick] at (0,1) {$R$};
\node[thick] at (0,2) {$P$};
\node[thick] at (1.5,3) {$I$};
\node[thick] at (3.5,3) {$E$};
\draw[dotted, thick] (2.5,0.5) -- (2.5,2.5);
\draw[every loop]
(G) edge[bend left] node [scale=1.2, above=0.1 of C] {} (H)
(H) edge[loop right] node [scale=1.2, below=0.1 of H] {$u_1\frac{1}{s}$} (H);

\end{tikzpicture}
\end{minipage}



{\bf We now study $\mathrm{TA}^3_{2}$ with action ``Include''.}

\begin{minipage}{0.45\textwidth}
Condition: 
$x_{1}=0$, $x_{2}=1$, $y=1$, $\mathrm{TA}^3_{1}$=E. \\
Therefore, Type I, $\neg x_{1} = 1$, \\$C_{3}=\neg x_{1}\wedge x_{2}=1$.
\end{minipage}
\begin{minipage}{0.35\textwidth}
\begin{tikzpicture}[node distance = .35cm, font=\Huge]
\tikzstyle{every node}=[scale=0.35]
\node[state] (E) at (1,1) {};
\node[state] (F) at (2,1) {};
\node[state] (G) at (3,1) {};
\node[state] (H) at (4,1) {};
\node[state] (A) at (1,2) {};
\node[state] (B) at (2,2) {};
\node[state] (C) at (3,2) {};
\node[state] (D) at (4,2) {};
\node[thick] at (0,1) {$R$};
\node[thick] at (0,2) {$P$};
\node[thick] at (1.5,3) {$I$};
\node[thick] at (3.5,3) {$E$};
\draw[dotted, thick] (2.5,0.5) -- (2.5,2.5);
\draw[every loop]
(F) edge[bend right] node [scale=1.2, above=0.1 of C] {} (E)
(E) edge[loop left] node [scale=1.2, below=0.1 of C] {$u_1\frac{s-1}{s}$} (E);

\end{tikzpicture}
\end{minipage}

\begin{minipage}{0.45\textwidth}
Condition: $x_{1}=0$, $x_{2}=1$, $y=1$, $\mathrm{TA}^3_{1}$=I. \\
Therefore, Type I, $\neg x_{1} = 1$, $C_{3}=0$. 
\end{minipage}
\begin{minipage}{0.35\textwidth}
\begin{tikzpicture}[node distance = .35cm, font=\Huge]
\tikzstyle{every node}=[scale=0.35]
\node[state] (E) at (1,1) {};
\node[state] (F) at (2,1) {};
\node[state] (G) at (3,1) {};
\node[state] (H) at (4,1) {};
\node[state] (A) at (1,2) {};
\node[state] (B) at (2,2) {};
\node[state] (C) at (3,2) {};
\node[state] (D) at (4,2) {};
\node[thick] at (0,1) {$R$};
\node[thick] at (0,2) {$P$};
\node[thick] at (1.5,3) {$I$};
\node[thick] at (3.5,3) {$E$};
\draw[dotted, thick] (2.5,0.5) -- (2.5,2.5);
\draw[every loop]
(A) edge[bend left] node [scale=1.2, above=0.1 of C] {} (B)
(B) edge[bend left] node [scale=1.2, above=0.1 of C] {~~$u_1\frac{1}{s}$} (C);

\end{tikzpicture}
\end{minipage}


{\bf We now study $\mathrm{TA}^3_{2}$ with action ``Exclude''.} 

\begin{minipage}{0.45\textwidth}
Condition: $x_{1}=1$, $x_{2}=1$, $y=0$, $\mathrm{TA}^3_{1}$=E. \\
Therefore, Type II, $\neg x_{1} = 0$, $C_3= x_{2}=1$. 
\end{minipage}
\begin{minipage}{0.35\textwidth}
\begin{tikzpicture}[node distance = .35cm, font=\Huge]
\tikzstyle{every node}=[scale=0.35]
\node[state] (E) at (1,1) {};
\node[state] (F) at (2,1) {};
\node[state] (G) at (3,1) {};
\node[state] (H) at (4,1) {};
\node[state] (A) at (1,2) {};
\node[state] (B) at (2,2) {};
\node[state] (C) at (3,2) {};
\node[state] (D) at (4,2) {};
\node[thick] at (0,1) {$R$};
\node[thick] at (0,2) {$P$};
\node[thick] at (1.5,3) {$I$};
\node[thick] at (3.5,3) {$E$};
\draw[dotted, thick] (2.5,0.5) -- (2.5,2.5);
\draw[every loop]
(D) edge[bend right] node [scale=1.2, above=0.1 of C] {} (C)
(C) edge[bend right] node [scale=1.2, above=0.1 of B] {$u_2\times1$} (B);

\end{tikzpicture}
\end{minipage}

\begin{minipage}{0.45\textwidth}
Condition: 
$x_{1}=0$, $x_{2}=1$, $y=1$, $\mathrm{TA}^3_{1}$=E \\
Therefore, Type I, $\neg x_{1} = 1$, $C_3= x_{2}=1$.
\end{minipage}
\begin{minipage}{0.35\textwidth}
\begin{tikzpicture}[node distance = .35cm, font=\Huge]
\tikzstyle{every node}=[scale=0.35]
\node[state] (E) at (1,1) {};
\node[state] (F) at (2,1) {};
\node[state] (G) at (3,1) {};
\node[state] (H) at (4,1) {};
\node[state] (A) at (1,2) {};
\node[state] (B) at (2,2) {};
\node[state] (C) at (3,2) {};
\node[state] (D) at (4,2) {};
\node[thick] at (0,1) {$R$};
\node[thick] at (0,2) {$P$};
\node[thick] at (1.5,3) {$I$};
\node[thick] at (3.5,3) {$E$};
\draw[dotted, thick] (2.5,0.5) -- (2.5,2.5);
\draw[every loop]
(D) edge[bend right] node [scale=1.2, above=0.1 of C] {} (C)
(C) edge[bend right] node [scale=1.2, above=0.1 of C] {$~~~~~~u_1\frac{s-1}{s}$} (B);

\end{tikzpicture}
\end{minipage}

\begin{minipage}{0.45\textwidth}
Condition: 
$x_{1}=1$, $x_{2}=1$, $y=0$, $\mathrm{TA}^3_{1}$=I.\\
Therefore, Type II, $\neg x_{1} = 0$, \\$C_3= x_{1} \wedge x_{2}=1$. 
\end{minipage}
\begin{minipage}{0.35\textwidth}
\begin{tikzpicture}[node distance = .35cm, font=\Huge]
\tikzstyle{every node}=[scale=0.35]
\node[state] (E) at (1,1) {};
\node[state] (F) at (2,1) {};
\node[state] (G) at (3,1) {};
\node[state] (H) at (4,1) {};
\node[state] (A) at (1,2) {};
\node[state] (B) at (2,2) {};
\node[state] (C) at (3,2) {};
\node[state] (D) at (4,2) {};
\node[thick] at (0,1) {$R$};
\node[thick] at (0,2) {$P$};
\node[thick] at (1.5,3) {$I$};
\node[thick] at (3.5,3) {$E$};
\draw[dotted, thick] (2.5,0.5) -- (2.5,2.5);
\draw[every loop]
(D) edge[bend right] node [scale=1.2, above=0.1 of C] {} (C)
(C) edge[bend right] node [scale=1.2, above=0.1 of C] {$u_2\times1$} (B);

\end{tikzpicture}
\end{minipage}

\begin{minipage}{0.45\textwidth}
Condition: 
$x_{1}=0$, $x_{2}=1$, $y=1$, $\mathrm{TA}^3_{1}$=I.\\
Therefore, Type I, $\neg x_{1} = 1$, \\$C_3= x_{1} \wedge x_{2}=0$. 
\end{minipage}
\begin{minipage}{0.35\textwidth}
\begin{tikzpicture}[node distance = .35cm, font=\Huge]
\tikzstyle{every node}=[scale=0.35]
\node[state] (E) at (1,1) {};
\node[state] (F) at (2,1) {};
\node[state] (G) at (3,1) {};
\node[state] (H) at (4,1) {};
\node[state] (A) at (1,2) {};
\node[state] (B) at (2,2) {};
\node[state] (C) at (3,2) {};
\node[state] (D) at (4,2) {};
\node[thick] at (0,1) {$R$};
\node[thick] at (0,2) {$P$};
\node[thick] at (1.5,3) {$I$};
\node[thick] at (3.5,3) {$E$};
\draw[dotted, thick] (2.5,0.5) -- (2.5,2.5);
\draw[every loop]
(G) edge[bend left] node [scale=1.2, above=0.1 of C] {} (H)
(H) edge[loop right] node [scale=1.2, below=0.1 of H] {$u_1\frac{1}{s}$} (H);

\end{tikzpicture}
\end{minipage}

Clearly $\mathrm{TA}^3_{1}$ will only move to ``Exclude". 
In this situation $\mathrm{TA}^3_{2}$ will become ``Include".


\newpage
\noindent {\bf Case 3} \\
Here $\mathrm{TA}^3_3$ is frozen as ``Exclude" and $\mathrm{TA}^3_4$ is as ``Exclude". 

{\bf We now study $\mathrm{TA}^3_{1}$ with action ``Include''.}

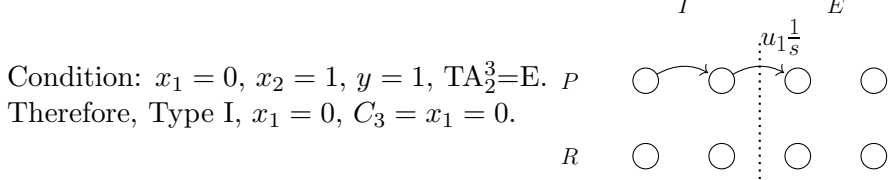
\begin{minipage}{0.45\textwidth}
Condition: $x_{1}=0$, $x_{2}=1$, $y=1$, $\mathrm{TA}^3_{2}$=E. \\
Therefore, Type I, $x_1=0$, $C_{3}= x_{1}=0$. 
\end{minipage}
\begin{minipage}{0.35\textwidth}
\begin{tikzpicture}[node distance = .35cm, font=\Huge]
\tikzstyle{every node}=[scale=0.35]
\node[state] (E) at (1,1) {};
\node[state] (F) at (2,1) {};
\node[state] (G) at (3,1) {};
\node[state] (H) at (4,1) {};
\node[state] (A) at (1,2) {};
\node[state] (B) at (2,2) {};
\node[state] (C) at (3,2) {};
\node[state] (D) at (4,2) {};
\node[thick] at (0,1) {$R$};
\node[thick] at (0,2) {$P$};
\node[thick] at (1.5,3) {$I$};
\node[thick] at (3.5,3) {$E$};
\draw[dotted, thick] (2.5,0.5) -- (2.5,2.5);
\draw[every loop]
(A) edge[bend left] node [scale=1.2, above=0.1 of C] {} (B)
(B) edge[bend left] node [scale=1.2, above=0.1 of C] {$~~~~~u_1\frac{1}{s}$} (C);
\end{tikzpicture}
\end{minipage}

\begin{minipage}{0.45\textwidth}
Condition: $x_{1}=0$, $x_{2}=1$, $y=1$, $\mathrm{TA}^3_{2}$=I. \\
Therefore, Type I, $x_{1} = 0$, \\$C_{3}= x_{1} \wedge \neg x_{1}=0$. 
\end{minipage}
\begin{minipage}{0.35\textwidth}
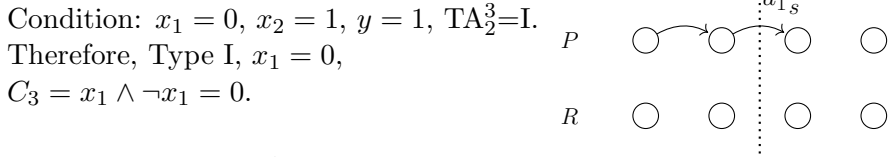
\begin{tikzpicture}[node distance = .35cm, font=\Huge]
\tikzstyle{every node}=[scale=0.35]
\node[state] (E) at (1,1) {};
\node[state] (F) at (2,1) {};
\node[state] (G) at (3,1) {};
\node[state] (H) at (4,1) {};
\node[state] (A) at (1,2) {};
\node[state] (B) at (2,2) {};
\node[state] (C) at (3,2) {};
\node[state] (D) at (4,2) {};
\node[thick] at (0,1) {$R$};
\node[thick] at (0,2) {$P$};
\node[thick] at (1.5,3) {$I$};
\node[thick] at (3.5,3) {$E$};
\draw[dotted, thick] (2.5,0.5) -- (2.5,2.5);
\draw[every loop]
(A) edge[bend left] node [scale=1.2, above=0.1 of C] {} (B)
(B) edge[bend left] node [scale=1.2, above=0.1 of C] {$~~~~~u_1\frac{1}{s}$} (C);

\end{tikzpicture}
\end{minipage}

{\bf We now study $\mathrm{TA}^3_{1}$ with action ``Exclude''.}



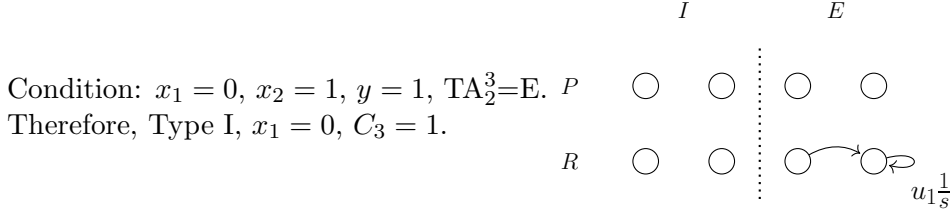
\begin{minipage}{0.45\textwidth}
Condition: $x_{1}=0$, $x_{2}=1$, $y=1$, $\mathrm{TA}^3_{2}$=E. \\
Therefore, Type I, $x_{1} = 0$, $C_{3}=1$. 
\end{minipage}
\begin{minipage}{0.35\textwidth}
\begin{tikzpicture}[node distance = .35cm, font=\Huge]
\tikzstyle{every node}=[scale=0.35]
\node[state] (E) at (1,1) {};
\node[state] (F) at (2,1) {};
\node[state] (G) at (3,1) {};
\node[state] (H) at (4,1) {};
\node[state] (A) at (1,2) {};
\node[state] (B) at (2,2) {};
\node[state] (C) at (3,2) {};
\node[state] (D) at (4,2) {};
\node[thick] at (0,1) {$R$};
\node[thick] at (0,2) {$P$};
\node[thick] at (1.5,3) {$I$};
\node[thick] at (3.5,3) {$E$};
\draw[dotted, thick] (2.5,0.5) -- (2.5,2.5);
\draw[every loop]
(G) edge[bend left] node [scale=1.2, above=0.1 of C] {} (H)
(H) edge[loop right] node [scale=1.2, below=0.1 of H] {$~~~~u_1\frac{1}{s}$} (H);

\end{tikzpicture}
\end{minipage}

\begin{minipage}{0.45\textwidth}
Condition: $x_{1}=0$, $y=0$, $x_{2}=0$, $\mathrm{TA}^3_{2}$=I. \\
Therefore, Type II, $x_{1} = 0$, $C_{3}=1$ 
\end{minipage}
\begin{minipage}{0.35\textwidth}
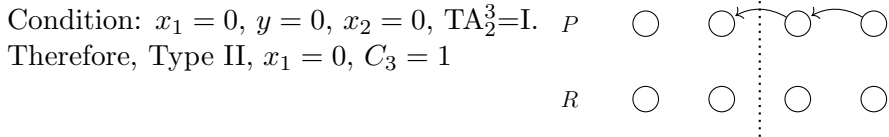
\begin{tikzpicture}[node distance = .35cm, font=\Huge]
\tikzstyle{every node}=[scale=0.35]
\node[state] (E) at (1,1) {};
\node[state] (F) at (2,1) {};
\node[state] (G) at (3,1) {};
\node[state] (H) at (4,1) {};
\node[state] (A) at (1,2) {};
\node[state] (B) at (2,2) {};
\node[state] (C) at (3,2) {};
\node[state] (D) at (4,2) {};
\node[thick] at (0,1) {$R$};
\node[thick] at (0,2) {$P$};
\node[thick] at (1.5,3) {$I$};
\node[thick] at (3.5,3) {$E$};
\draw[dotted, thick] (2.5,0.5) -- (2.5,2.5);
\draw[every loop]
(D) edge[bend right] node [scale=1.2, above=0.1 of C] {} (C)
(C) edge[bend right] node [scale=1.2, above=0.1 of B] {$u_2\times1$} (B);

\end{tikzpicture}
\end{minipage}

{\bf We study $\mathrm{TA}_{3,1}$ with action ``Exclude".}


\begin{minipage}{0.45\textwidth}
Condition; $x_{1}=0$, $x_{2}=1$, $y=1$, $\mathrm{TA}^3_{2}$=I. \\
Therefore, Type I, $x_{1} = 0$, $C_{3}=\neg x_{1}=1$. 
\end{minipage}
\begin{minipage}{0.35\textwidth}
\begin{tikzpicture}[node distance = .35cm, font=\Huge]
\tikzstyle{every node}=[scale=0.35]
\node[state] (E) at (1,1) {};
\node[state] (F) at (2,1) {};
\node[state] (G) at (3,1) {};
\node[state] (H) at (4,1) {};
\node[state] (A) at (1,2) {};
\node[state] (B) at (2,2) {};
\node[state] (C) at (3,2) {};
\node[state] (D) at (4,2) {};
\node[thick] at (0,1) {$R$};
\node[thick] at (0,2) {$P$};
\node[thick] at (1.5,3) {$I$};
\node[thick] at (3.5,3) {$E$};
\draw[dotted, thick] (2.5,0.5) -- (2.5,2.5);
\draw[every loop]
(G) edge[bend left] node [scale=1.2, above=0.1 of C] {} (H)
(H) edge[loop right] node [scale=1.2, below=0.1 of H] {$u_1\frac{1}{s}$} (H);

\end{tikzpicture}
\end{minipage}

\begin{minipage}{0.45\textwidth}
Condition: $x_{1}=0$, $x_{2}=0$, $y=1$, $\mathrm{TA}^3_{2}$=I. \\
Therefore, Type II, $x_{1} = 0$, $C_{3}=\neg x_{1}=1$. 
\end{minipage}
\begin{minipage}{0.35\textwidth}
\begin{tikzpicture}[node distance = .35cm, font=\Huge]
\tikzstyle{every node}=[scale=0.35]
\node[state] (E) at (1,1) {};
\node[state] (F) at (2,1) {};
\node[state] (G) at (3,1) {};
\node[state] (H) at (4,1) {};
\node[state] (A) at (1,2) {};
\node[state] (B) at (2,2) {};
\node[state] (C) at (3,2) {};
\node[state] (D) at (4,2) {};
\node[thick] at (0,1) {$R$};
\node[thick] at (0,2) {$P$};
\node[thick] at (1.5,3) {$I$};
\node[thick] at (3.5,3) {$E$};
\draw[dotted, thick] (2.5,0.5) -- (2.5,2.5);
\draw[every loop]
(D) edge[bend right] node [scale=1.2, above=0.1 of C] {} (C)
(C) edge[bend right] node [scale=1.2, above=0.1 of B] {$u_2\times1$} (B);

\end{tikzpicture}
\end{minipage}


{\bf We now study $\mathrm{TA}^3_{2}$ with action ``Include''.}


\begin{minipage}{0.45\textwidth}
Condition: $x_{1}=0$, $x_{2}=1$, $y=1$, $\mathrm{TA}^3_{1}$=E. \\
Therefore, Type I, $\neg x_{1} = 1$, $C_3=\neg x_{1}=1$.
\end{minipage}
\begin{minipage}{0.35\textwidth}
\begin{tikzpicture}[node distance = .35cm, font=\Huge]
\tikzstyle{every node}=[scale=0.35]
\node[state] (E) at (1,1) {};
\node[state] (F) at (2,1) {};
\node[state] (G) at (3,1) {};
\node[state] (H) at (4,1) {};
\node[state] (A) at (1,2) {};
\node[state] (B) at (2,2) {};
\node[state] (C) at (3,2) {};
\node[state] (D) at (4,2) {};
\node[thick] at (0,1) {$R$};
\node[thick] at (0,2) {$P$};
\node[thick] at (1.5,3) {$I$};
\node[thick] at (3.5,3) {$E$};
\draw[dotted, thick] (2.5,0.5) -- (2.5,2.5);
\draw[every loop]
(F) edge[bend right] node [scale=1.2, above=0.1 of E] {$u_1\frac{s-1}{s}$} (E)
(E) edge[loop left] node [scale=1.2, above=0.1 of E] {} (E);

\end{tikzpicture}
\end{minipage}

\begin{minipage}{0.45\textwidth}
Condition: $x_{1}=0$, $x_{2}=1$, $y=1$, $\mathrm{TA}^3_{1}$=I. \\
Therefore, Type I, $\neg x_{1} = 1$, $C_3=0$. \\
\end{minipage}
\begin{minipage}{0.35\textwidth}
\begin{tikzpicture}[node distance = .35cm, font=\Huge]
\tikzstyle{every node}=[scale=0.35]
\node[state] (E) at (1,1) {};
\node[state] (F) at (2,1) {};
\node[state] (G) at (3,1) {};
\node[state] (H) at (4,1) {};
\node[state] (A) at (1,2) {};
\node[state] (B) at (2,2) {};
\node[state] (C) at (3,2) {};
\node[state] (D) at (4,2) {};
\node[thick] at (0,1) {$R$};
\node[thick] at (0,2) {$P$};
\node[thick] at (1.5,3) {$I$};
\node[thick] at (3.5,3) {$E$};
\draw[dotted, thick] (2.5,0.5) -- (2.5,2.5);
\draw[every loop]
(A) edge[bend left] node [scale=1.2, above=0.1 of C] {} (B)
(B) edge[bend left] node [scale=1.2, above=0.1 of C] {~~$u_1\frac{1}{s}$} (C);

\end{tikzpicture}
\end{minipage}

\pagebreak

{\bf We now study $\mathrm{TA}^3_{2}$ with action ``Exclude''.}



\begin{minipage}{0.45\textwidth}
Condition: $x_{1}=1$, $x_{2}=1$, $y=0$, $\mathrm{TA}^3_{1}$=E.\\
Therefore, Type II, $\neg x_{1} = 0$, $C_3=x_{1}=1$. 
\end{minipage}
\begin{minipage}{0.35\textwidth}
\begin{tikzpicture}[node distance = .35cm, font=\Huge]
\tikzstyle{every node}=[scale=0.35]
\node[state] (E) at (1,1) {};
\node[state] (F) at (2,1) {};
\node[state] (G) at (3,1) {};
\node[state] (H) at (4,1) {};
\node[state] (A) at (1,2) {};
\node[state] (B) at (2,2) {};
\node[state] (C) at (3,2) {};
\node[state] (D) at (4,2) {};
\node[thick] at (0,1) {$R$};
\node[thick] at (0,2) {$P$};
\node[thick] at (1.5,3) {$I$};
\node[thick] at (3.5,3) {$E$};
\draw[dotted, thick] (2.5,0.5) -- (2.5,2.5);
\draw[every loop]
(D) edge[bend right] node [scale=1.2, above=0.1 of C] {$u_2\times1$} (C)
(C) edge[bend right] node [scale=1.2, above=0.1 of C] {} (B);

\end{tikzpicture}
\end{minipage}

\begin{minipage}{0.45\textwidth}
Condition: $x_{1}=0$, $x_{2}=1$, $y=1$, $\mathrm{TA}^3_{1}$=E. \\
Therefore, Type I, $\neg x_{1} = 0$, $C_3=x_{1}=1$.
\end{minipage}
\begin{minipage}{0.35\textwidth}
\begin{tikzpicture}[node distance = .35cm, font=\Huge]
\tikzstyle{every node}=[scale=0.35]
\node[state] (E) at (1,1) {};
\node[state] (F) at (2,1) {};
\node[state] (G) at (3,1) {};
\node[state] (H) at (4,1) {};
\node[state] (A) at (1,2) {};
\node[state] (B) at (2,2) {};
\node[state] (C) at (3,2) {};
\node[state] (D) at (4,2) {};
\node[thick] at (0,1) {$R$};
\node[thick] at (0,2) {$P$};
\node[thick] at (1.5,3) {$I$};
\node[thick] at (3.5,3) {$E$};
\draw[dotted, thick] (2.5,0.5) -- (2.5,2.5);
\draw[every loop]
(G) edge[bend left] node [scale=1.2, above=0.1 of C] {} (H)
(H) edge[loop right] node [scale=1.2, below=0.1 of H] {$u_1\frac{1}{s}$} (H);

\end{tikzpicture}
\end{minipage}

\begin{minipage}{0.45\textwidth}
Condition: $x_{1}=1$, $x_{2}=1$, $y=0$, $\mathrm{TA}^3_{1}$=I. \\
Therefore, Type II, $\neg x_{1} = 0$, $C_3=x_{1}=1$. 
\end{minipage}
\begin{minipage}{0.35\textwidth}
\begin{tikzpicture}[node distance = .35cm, font=\Huge]
\tikzstyle{every node}=[scale=0.35]
\node[state] (E) at (1,1) {};
\node[state] (F) at (2,1) {};
\node[state] (G) at (3,1) {};
\node[state] (H) at (4,1) {};
\node[state] (A) at (1,2) {};
\node[state] (B) at (2,2) {};
\node[state] (C) at (3,2) {};
\node[state] (D) at (4,2) {};
\node[thick] at (0,1) {$R$};
\node[thick] at (0,2) {$P$};
\node[thick] at (1.5,3) {$I$};
\node[thick] at (3.5,3) {$E$};
\draw[dotted, thick] (2.5,0.5) -- (2.5,2.5);
\draw[every loop]
(D) edge[bend right] node [scale=1.2, above=0.1 of C] {$u_2\times1$} (C)
(C) edge[bend right] node [scale=1.2, above=0.1 of C] {} (B);

\end{tikzpicture}
\end{minipage}

\begin{minipage}{0.45\textwidth}
Condition: $x_{1}=0$, $x_{2}=1$, $y=1$, $\mathrm{TA}^3_{1}$=I. \\
Therefore, Type I, $\neg x_{1} = 0$, $C_3=x_{1}=1$.
\end{minipage}
\begin{minipage}{0.35\textwidth}
\begin{tikzpicture}[node distance = .35cm, font=\Huge]
\tikzstyle{every node}=[scale=0.35]
\node[state] (E) at (1,1) {};
\node[state] (F) at (2,1) {};
\node[state] (G) at (3,1) {};
\node[state] (H) at (4,1) {};
\node[state] (A) at (1,2) {};
\node[state] (B) at (2,2) {};
\node[state] (C) at (3,2) {};
\node[state] (D) at (4,2) {};
\node[thick] at (0,1) {$R$};
\node[thick] at (0,2) {$P$};
\node[thick] at (1.5,3) {$I$};
\node[thick] at (3.5,3) {$E$};
\draw[dotted, thick] (2.5,0.5) -- (2.5,2.5);
\draw[every loop]
(G) edge[bend left] node [scale=1.2, above=0.1 of C] {} (H)
(H) edge[loop right] node [scale=1.2, below=0.1 of H] {$u_1\frac{1}{s}$} (H);

\end{tikzpicture}
\end{minipage}


\noindent {\bf Case 4} \\
Here $\mathrm{TA}^3_3$ is frozen as ``Include" and $\mathrm{TA}^3_4$ is as ``Include". In this situation, the output of $\mathrm{TA}^3_3$ and $\mathrm{TA}^3_4$ is 0.

{\bf We now study $\mathrm{TA}^3_{1}$ with action ``Include''.}

\begin{minipage}{0.45\textwidth}
Condition: $x_{1}=0$, $x_{2}=1$, $y=1$, $\mathrm{TA}^3_{2}$=E. \\
Therefore, Type I, $x_{1} = 0$, $C_{3}=0$. 
\end{minipage}
\begin{minipage}{0.35\textwidth}
\begin{tikzpicture}[node distance = .35cm, font=\Huge]
\tikzstyle{every node}=[scale=0.35]
\node[state] (E) at (1,1) {};
\node[state] (F) at (2,1) {};
\node[state] (G) at (3,1) {};
\node[state] (H) at (4,1) {};
\node[state] (A) at (1,2) {};
\node[state] (B) at (2,2) {};
\node[state] (C) at (3,2) {};
\node[state] (D) at (4,2) {};
\node[thick] at (0,1) {$R$};
\node[thick] at (0,2) {$P$};
\node[thick] at (1.5,3) {$I$};
\node[thick] at (3.5,3) {$E$};
\draw[dotted, thick] (2.5,0.5) -- (2.5,2.5);
\draw[every loop]
(A) edge[bend left] node [scale=1.2, above=0.1 of C] {} (B)
(B) edge[bend left] node [scale=1.2, above=0.1 of C] {$~~~~~u_1\frac{1}{s}$} (C);

\end{tikzpicture}
\end{minipage}

\begin{minipage}{0.45\textwidth}
Condition: $x_{1}=0$, $x_{2}=1$, $y=1$, $\mathrm{TA}^3_{2}$=I. \\
Therefore, Type I, $x_{1} = 0$, $C_{3}=0$
\end{minipage}
\begin{minipage}{0.35\textwidth}
\begin{tikzpicture}[node distance = .35cm, font=\Huge]
\tikzstyle{every node}=[scale=0.35]
\node[state] (E) at (1,1) {};
\node[state] (F) at (2,1) {};
\node[state] (G) at (3,1) {};
\node[state] (H) at (4,1) {};
\node[state] (A) at (1,2) {};
\node[state] (B) at (2,2) {};
\node[state] (C) at (3,2) {};
\node[state] (D) at (4,2) {};
\node[thick] at (0,1) {$R$};
\node[thick] at (0,2) {$P$};
\node[thick] at (1.5,3) {$I$};
\node[thick] at (3.5,3) {$E$};
\draw[dotted, thick] (2.5,0.5) -- (2.5,2.5);
\draw[every loop]
(A) edge[bend left] node [scale=1.2, above=0.1 of C] {} (B)
(B) edge[bend left] node [scale=1.2, above=0.1 of C] {$~~~~~u_1\frac{1}{s}$} (C);

\end{tikzpicture}
\end{minipage}


{\bf We now study $\mathrm{TA}^3_{1}$ with action ``Exclude''.}

\begin{minipage}{0.45\textwidth}
Condition: $x_{1}=0$, $x_{2}=1$, $y=1$, $\mathrm{TA}^3_{2}$=E. \\
Therefore, Type I, $x_{1}=0$, $C_{3}=0$. 
\end{minipage}
\begin{minipage}{0.35\textwidth}
\begin{tikzpicture}[node distance = .35cm, font=\Huge]
\tikzstyle{every node}=[scale=0.35]
\node[state] (E) at (1,1) {};
\node[state] (F) at (2,1) {};
\node[state] (G) at (3,1) {};
\node[state] (H) at (4,1) {};
\node[state] (A) at (1,2) {};
\node[state] (B) at (2,2) {};
\node[state] (C) at (3,2) {};
\node[state] (D) at (4,2) {};
\node[thick] at (0,1) {$R$};
\node[thick] at (0,2) {$P$};
\node[thick] at (1.5,3) {$I$};
\node[thick] at (3.5,3) {$E$};
\draw[dotted, thick] (2.5,0.5) -- (2.5,2.5);
\draw[every loop]
(G) edge[bend left] node [scale=1.2, above=0.1 of C] {} (H)
(H) edge[loop right] node [scale=1.2, below=0.1 of H] {$u_1\frac{1}{s}$} (H);

\end{tikzpicture}
\end{minipage}

\begin{minipage}{0.45\textwidth}
Condition: $x_{1}=0$, $x_{2}=1$, $y=0$, $\mathrm{TA}^3_{2}$=I. \\
Therefore, Type II, $x_{1}=0$, $C_{3}=0$. 
\end{minipage}
\begin{minipage}{0.35\textwidth}
\begin{tikzpicture}[node distance = .35cm, font=\Huge]
\tikzstyle{every node}=[scale=0.35]
\node[state] (E) at (1,1) {};
\node[state] (F) at (2,1) {};
\node[state] (G) at (3,1) {};
\node[state] (H) at (4,1) {};
\node[state] (A) at (1,2) {};
\node[state] (B) at (2,2) {};
\node[state] (C) at (3,2) {};
\node[state] (D) at (4,2) {};
\node[thick] at (0,1) {$R$};
\node[thick] at (0,2) {$P$};
\node[thick] at (1.5,3) {$I$};
\node[thick] at (3.5,3) {$E$};
\draw[dotted, thick] (2.5,0.5) -- (2.5,2.5);
\draw[every loop]
(G) edge[bend left] node [scale=1.2, above=0.1 of C] {} (H)
(H) edge[loop right] node [scale=1.2, above=0.1 of H] {$u_1\frac{1}{s}$} (H);

\end{tikzpicture}
\end{minipage}

\pagebreak

{\bf We now study $\mathrm{TA}^3_{2}$ with action ``Include''.}

\begin{minipage}{0.45\textwidth}
Condition: $x_{1}=0$, $x_{2}=1$, $y=1$, $\mathrm{TA}^3_{1}$=E. \\
Therefore, Type I, $\neg x_{1}=1$, $C_3=0$. 
\end{minipage}
\begin{minipage}{0.35\textwidth}
\begin{tikzpicture}[node distance = .35cm, font=\Huge]
\tikzstyle{every node}=[scale=0.35]
\node[state] (E) at (1,1) {};
\node[state] (F) at (2,1) {};
\node[state] (G) at (3,1) {};
\node[state] (H) at (4,1) {};
\node[state] (A) at (1,2) {};
\node[state] (B) at (2,2) {};
\node[state] (C) at (3,2) {};
\node[state] (D) at (4,2) {};
\node[thick] at (0,1) {$R$};
\node[thick] at (0,2) {$P$};
\node[thick] at (1.5,3) {$I$};
\node[thick] at (3.5,3) {$E$};
\draw[dotted, thick] (2.5,0.5) -- (2.5,2.5);
\draw[every loop]
(A) edge[bend left] node [scale=1.2, above=0.1 of C] {} (B)
(B) edge[bend left] node [scale=1.2, above=0.1 of C] {~~$u_1\frac{1}{s}$} (C);

\end{tikzpicture}
\end{minipage}

\begin{minipage}{0.45\textwidth}
Condition: $x_{1}=0$, $x_{2}=1$, $y=1$, $\mathrm{TA}^3_{1}$=I. \\
Therefore, Type I, $\neg x_{1}=1$, $C_3=0$.
\end{minipage}
\begin{minipage}{0.35\textwidth}
\begin{tikzpicture}[node distance = .35cm, font=\Huge]
\tikzstyle{every node}=[scale=0.35]
\node[state] (E) at (1,1) {};
\node[state] (F) at (2,1) {};
\node[state] (G) at (3,1) {};
\node[state] (H) at (4,1) {};
\node[state] (A) at (1,2) {};
\node[state] (B) at (2,2) {};
\node[state] (C) at (3,2) {};
\node[state] (D) at (4,2) {};
\node[thick] at (0,1) {$R$};
\node[thick] at (0,2) {$P$};
\node[thick] at (1.5,3) {$I$};
\node[thick] at (3.5,3) {$E$};
\draw[dotted, thick] (2.5,0.5) -- (2.5,2.5);
\draw[every loop]
(A) edge[bend left] node [scale=1.2, above=0.1 of C] {} (B)
(B) edge[bend left] node [scale=1.2, above=0.1 of C] {~~$u_1\frac{1}{s}$} (C);

\end{tikzpicture}
\end{minipage}


{\bf We now study $\mathrm{TA}^3_{2}$ with action ``Exclude''.}

\begin{minipage}{0.45\textwidth}
Condition: $x_{1}=0$, $x_{2}=1$, $y=1$, $\mathrm{TA}^3_{1}$=E. \\
Therefore, Type I, $\neg x_{1}=1$, $C_{3}=0$.
\end{minipage}
\begin{minipage}{0.35\textwidth}
\begin{tikzpicture}[node distance = .35cm, font=\Huge]
\tikzstyle{every node}=[scale=0.35]
\node[state] (E) at (1,1) {};
\node[state] (F) at (2,1) {};
\node[state] (G) at (3,1) {};
\node[state] (H) at (4,1) {};
\node[state] (A) at (1,2) {};
\node[state] (B) at (2,2) {};
\node[state] (C) at (3,2) {};
\node[state] (D) at (4,2) {};
\node[thick] at (0,1) {$R$};
\node[thick] at (0,2) {$P$};
\node[thick] at (1.5,3) {$I$};
\node[thick] at (3.5,3) {$E$};
\draw[dotted, thick] (2.5,0.5) -- (2.5,2.5);
\draw[every loop]
(G) edge[bend left] node [scale=1.2, above=0.1 of C] {} (H)
(H) edge[loop right] node [scale=1.2, below=0.1 of H] {$u_1\frac{1}{s}$} (H);

\end{tikzpicture}
\end{minipage}

\begin{minipage}{0.45\textwidth}
Condition: $x_{1}=0$, $x_{2}=1$, $y=1$, $\mathrm{TA}^3_{1}$=I. \\
Therefore, Type I, $\neg x_{1}=1$, $C_{3}=0$. 
\end{minipage}
\begin{minipage}{0.35\textwidth}
\begin{tikzpicture}[node distance = .35cm, font=\Huge]
\tikzstyle{every node}=[scale=0.35]
\node[state] (E) at (1,1) {};
\node[state] (F) at (2,1) {};
\node[state] (G) at (3,1) {};
\node[state] (H) at (4,1) {};
\node[state] (A) at (1,2) {};
\node[state] (B) at (2,2) {};
\node[state] (C) at (3,2) {};
\node[state] (D) at (4,2) {};
\node[thick] at (0,1) {$R$};
\node[thick] at (0,2) {$P$};
\node[thick] at (1.5,3) {$I$};
\node[thick] at (3.5,3) {$E$};
\draw[dotted, thick] (2.5,0.5) -- (2.5,2.5);
\draw[every loop]
(G) edge[bend left] node [scale=1.2, above=0.1 of C] {} (H)
(H) edge[loop right] node [scale=1.2, above=0.1 of H] {$u_1\frac{1}{s}$} (H);

\end{tikzpicture}
\end{minipage}


Based on the analysis performed above, we can show the directions of transitions for $\mathrm{TA}^3_1$ and $\mathrm{TA}^3_2$ given different configurations of $\mathrm{TA}^3_3$ and $\mathrm{TA}^3_4$.

\textbf{Scenario 1:} Study $\mathrm{TA}^3_1$ = I and $\mathrm{TA}^3_2$ = E.

\begin{minipage}{0.5\textwidth}
\textbf{Case 1:} we can see that \\
$\mathrm{TA}^3_1$ $\rightarrow$ E \\
$\mathrm{TA}^3_2$ $\rightarrow$ E 
\end{minipage}
\begin{minipage}{0.5\textwidth}
\textbf{Case 2:} we can see that \\
$\mathrm{TA}^3_1$ $\rightarrow$ E \\
$\mathrm{TA}^3_2$ $\rightarrow$ I, E 
\end{minipage}

\begin{minipage}{0.5\textwidth}
\textbf{Case 3:} we can see that \\
$\mathrm{TA}^3_1$ $\rightarrow$ E \\
$\mathrm{TA}^3_2$ $\rightarrow$ I
\end{minipage}
\begin{minipage}{0.5\textwidth}
\textbf{Case 4:} we can see that \\
$\mathrm{TA}^3_1$ $\rightarrow$ E \\
$\mathrm{TA}^3_2$ $\rightarrow$ E 
\end{minipage}

\vspace{.5cm}

\textbf{Scenario 2:} Study $\mathrm{TA}^3_1$ = I and $\mathrm{TA}^3_2$ = I.

\begin{minipage}{0.5\textwidth}
\textbf{Case 1:} we can see that \\
$\mathrm{TA}^3_1$ $\rightarrow$ E \\
$\mathrm{TA}^3_2$ $\rightarrow$ E 
\end{minipage}
\begin{minipage}{0.5\textwidth}
\textbf{Case 2:} we can see that \\
$\mathrm{TA}^3_1$ $\rightarrow$ E \\
$\mathrm{TA}^3_2$ $\rightarrow$ E 
\end{minipage}

\begin{minipage}{0.5\textwidth}
\textbf{Case 3:} we can see that \\
$\mathrm{TA}^3_1$ $\rightarrow$ E \\
$\mathrm{TA}^3_2$ $\rightarrow$ E 
\end{minipage}
\begin{minipage}{0.5\textwidth}
\textbf{Case 4:} we can see that \\
$\mathrm{TA}^3_1$ $\rightarrow$ E \\
$\mathrm{TA}^3_2$ $\rightarrow$ E 
\end{minipage}

\vspace{.5cm}

\textbf{Scenario 3:} Study $\mathrm{TA}^3_1$ = E and $\mathrm{TA}^3_2$ = I.

\begin{minipage}{0.5\textwidth}
\textbf{Case 1:} we can see that \\
$\mathrm{TA}^3_1$ $\rightarrow$ I, E \\
$\mathrm{TA}^3_2$ $\rightarrow$ E
\end{minipage}
\begin{minipage}{0.5\textwidth}
\textbf{Case 2:} we can see that \\
$\mathrm{TA}^3_1$ $\rightarrow$ E \\
$\mathrm{TA}^3_2$ $\rightarrow$ I
\end{minipage}

\begin{minipage}{0.5\textwidth}
\textbf{Case 3:} we can see that \\
$\mathrm{TA}^3_1$ $\rightarrow$ I \\
$\mathrm{TA}^3_2$ $\rightarrow$ I
\end{minipage}
\begin{minipage}{0.5\textwidth}
\textbf{Case 4:} we can see that \\
$\mathrm{TA}^3_1$ $\rightarrow$ E \\
$\mathrm{TA}^3_2$ $\rightarrow$ E 
\end{minipage}

\vspace{.5cm}

\textbf{Scenario 4:} Study $\mathrm{TA}^3_3$ = E and $\mathrm{TA}^3_4$ = E.

\begin{minipage}{0.5\textwidth}
\textbf{Case 1:} we can see that \\
$\mathrm{TA}^3_1$ $\rightarrow$ I, E \\
$\mathrm{TA}^3_2$ $\rightarrow$ E
\end{minipage}
\begin{minipage}{0.5\textwidth}
\textbf{Case 2:} we can see that \\
$\mathrm{TA}^3_1$ $\rightarrow$ E\\
$\mathrm{TA}^3_2$ $\rightarrow$ I 
\end{minipage}

\begin{minipage}{0.5\textwidth}
\textbf{Case 3:} we can see that \\
$\mathrm{TA}^3_1$ $\rightarrow$ E\\
$\mathrm{TA}^3_2$ $\rightarrow$ I, E 
\end{minipage}
\begin{minipage}{0.5\textwidth}
\textbf{Case 4:} we can see that \\
$\mathrm{TA}^3_1$ $\rightarrow$ E \\
$\mathrm{TA}^3_2$ $\rightarrow$ E 
\end{minipage}

Clearly, from the above transitions, we can conclude that state $\mathrm{TA}^3_1$=E and $\mathrm{TA}^3_2$=I is absorbing when the state $\mathrm{TA}^3_3$=I and $\mathrm{TA}^3_4$=E are frozen. Similarly, state $\mathrm{TA}^3_1$=E and $\mathrm{TA}^3_2$=E is also absorbing when $\mathrm{TA}^3_3$ and $\mathrm{TA}^3_4$ are both frozen as Include. The other states are not absorbing. 


\end{document}

%% file: Figures/TA.tex
\begin{figure}
\centering
\resizebox{0.92\textwidth}{!}{
\begin{minipage}{1\textwidth}
\begin{tikzpicture}[node distance = .35cm, font=\Huge]
    \tikzstyle{every node}=[scale=0.35]
    \node[state] (A) at (0,2) {~~~1~~~~};
    \node[state] (B) at (1.5,2) {~~~2~~~~};
    
    \node[state,draw=white] (M) at (3,2) {~~~$....$~~~};
    
    \node[state] (C) at (4.5,2) {$N-1$};
    \node[state] (D) at (6,2) {~~~\!$N$~~~};
    
    \node[state] (E) at (7.5,2) {$N+1$};
    \node[state] (F) at (9,2) {$N+2$};
    
    \node[state,draw=white] (G) at (10.5,2) {~~~$....$~~~};
    
    \node[state] (H) at (12,2) {\!$2N-1$};
    \node[state] (I) at (13.5,2) {~~\!$2N$~~~~};

    \node[thick] at (4,4) {$Action~1$};
    \node[thick] at (9.5,4) {$Action~2$};
    
    \node[thick] at (1.2,1) {$Reward~(R):~\dashrightarrow$~~~~$Penalty~(P):~\rightarrow$};

    \draw[every loop]
    (A) edge[bend left] node [scale=1.2, above=0.1 of B]{} (B)
    (B) edge[bend left] node  [scale=1.2, above=0.1 of M] {} (M)
    (M) edge[bend left] node  [scale=1.2, above=0.1 of C] {} (C)
    (C) edge[bend left] node [scale=1.2, above=0.1 of D] {} (D)
    (D) edge[bend left] node  [scale=1.2, above=0.1 of E] {} (E);

    \draw[every loop]
    (I) edge[bend left] node [scale=1.2, below=0.1 of H] {} (H)
    (H) edge[bend left] node  [scale=1.2, below=0.1 of G] {} (G)
    (G) edge[bend left] node [scale=1.2, below=0.1 of F] {} (F)
    (F) edge[bend left] node  [scale=1.2, below=0.1 of E] {} (E)
    (E) edge[bend left] node  [scale=1.2, below=0.1 of D] {} (D);

    
    \draw[dashed,->]
    (B) edge[bend left] node  [scale=1.2, above=0.1 of A] {} (A)
    (M) edge[bend left] node [scale=1.2, above=0.1 of B] {} (B)
    (C) edge[bend left] node [scale=1.2, above=0.1 of M] {} (M)
    (D) edge[bend left] node [scale=1.2, above=0.1 of C] {} (C)
    (A) edge[loop left] node [scale=1.2, below=0.1 of D] {} (D);
    
    \draw[dashed,->]

    (H) edge[bend left ] node [scale=1.2, below=0.1 of I] {} (I)
    (G) edge[bend left] node  [scale=1.2, below=0.1 of H] {} (H)
    (F) edge[bend left] node  [scale=1.2, below=0.1 of G] {} (G)
    (E) edge[bend left ] node [scale=1.2, below=0.1 of F] {} (F)
    (I) edge[loop right] node [scale=1.2, below=0.1 of E] {} (E);
    
      \draw[dotted, thick] (6.75,0.6) -- (6.75,3);

\end{tikzpicture}
\end{minipage}
}
\caption{A two-action Tsetlin Automaton with $2N$ states.}
\label{figure:TAarchitecture_basic_2n}
\end{figure}

%% file: Figures/TAteam.tex
\begin{figure}[htbp]
\begin{center}
\begin{minipage}{1\textwidth}
\resizebox{1\textwidth}{!}{
\begin{tikzpicture}[node distance = .35cm]

    \node[label=left:{\bf Inputs}] at (0.2,1.5) {};
    \node[label=right:{\bf Literals}] at (0.5,1.5) {};
    \node[label=right:{\bf TA team}] at (2.8,1.5) {};
    \node[label=right:{\bf TA decisions}] at (5.6,1.5) {};
    \node[label=right:{\bf Output}] at (11.3,1.5) {};

    \node[label=left:$x_1$] at (0,0) {};
    \node[label=left:$x_2$] at (0,-2) {};
    \node[label=left:$x_o$] at (0,-5) {};
    
    \node[label=right:$x_1$] at (1,0.5) {};
    \node[label=right:$\neg x_1$] at (1,-0.5) {};
    \node[label=right:$x_2$] at (1,-1.5) {};
    \node[label=right:$\neg x_2$] at (1,-2.5) {};
    \node[label=right:$x_o$] at (1,-4.5) {};
    \node[label=right:$\neg x_o$] at (1,-5.5) {};
    
    \node[label=right:$\mathrm{TA}_1^{i,j}$] at (3,0.5) {};
    \node[label=right:$\mathrm{TA}_2^{i,j}$] at (3,-0.5) {};
    \node[label=right:$\mathrm{TA}_3^{i,j}$] at (3,-1.5) {};
    \node[label=right:$\mathrm{TA}_4^{i,j}$] at (3,-2.5) {};
    \node[label=right:$\mathrm{TA}_{2o-1}^{i,j}$] at (3,-4.5) {};
    \node[label=right:$\mathrm{TA}_{2o}^{i,j}$] at (3,-5.5) {};
    \draw (3.1, 1) -- (4.6, 1) -- (4.6, -6) -- (3.1, -6) -- (3.1, 1); 
    
    \node[label=right:$I(x_1)~\text{or}~ E(x_1)$] at (5.7,0.5) {};
    \node[label=right:$I(\neg x_1)~\text{or}~ E(\neg x_1)$] at (5.7,-0.5) {};
    \node[label=right:$I(x_2)~\text{or}~ E(x_2)$] at (5.7,-1.5) {};
    \node[label=right:$I(\neg x_2)~\text{or}~ E(\neg x_2)$] at (5.7,-2.5) {};
    \node[label=right:$I(x_o)~\text{or}~ E(x_o)$] at (5.7,-4.5) {};
    \node[label=right:$I(\neg x_o)~\text{or}~ E(\neg x_o)$] at (5.7,-5.5) {};
    
    \node[label=right:{$C^i_j= \bigwedge\limits_{k'=1}^{2o}$ (decision of $\mathrm{TA}_{k'}^{i,j}$)}] at (11.4,-2.5) {};
    
    \draw [-{stealth[length=4mm]}] (0,0) -- (1,0.5);
    \draw [-{stealth[length=4mm]}] (0,0) -- (1,-0.5);
    \draw [-{stealth[length=4mm]}] (0,-2) -- (1,-1.5);
    \draw [-{stealth[length=4mm]}] (0,-2) -- (1,-2.5);
    \draw [-{stealth[length=4mm]}] (0,-5) -- (1,-4.5);
    \draw [-{stealth[length=4mm]}] (0,-5) -- (1,-5.5);
    
    \draw [-{stealth[length=4mm]}] (2,0.5) -- (3, 0.5);
    \draw [-{stealth[length=4mm]}] (2,-0.5) -- (3, -0.5);
    \draw [-{stealth[length=4mm]}] (2,-1.5) -- (3, -1.5);
    \draw [-{stealth[length=4mm]}] (2,-2.5) -- (3, -2.5);
    \draw [-{stealth[length=4mm]}] (2.4,-4.5) -- (3, -4.5);
    \draw [-{stealth[length=4mm]}] (2.4,-5.5) -- (3, -5.5);
    
    \draw [-{stealth[length=4mm]}] (4.7,0.5) -- (5.7, 0.5);
    \draw [-{stealth[length=4mm]}] (4.7,-0.5) -- (5.7, -0.5);
    \draw [-{stealth[length=4mm]}] (4.7,-1.5) -- (5.7, -1.5);
    \draw [-{stealth[length=4mm]}] (4.7,-2.5) -- (5.7, -2.5);
    \draw [-{stealth[length=4mm]}] (4.7,-4.5) -- (5.7, -4.5);
    \draw [-{stealth[length=4mm]}] (4.7,-5.5) -- (5.7, -5.5);
    
    \draw [-{stealth[length=4mm]}] (9.8,0.5) -- (11.6, -2.5);
    \draw [-{stealth[length=4mm]}] (9.8,-0.5) -- (11.6, -2.5);
    \draw [-{stealth[length=4mm]}] (9.8,-1.5) -- (11.6, -2.5);
    \draw [-{stealth[length=4mm]}] (9.8,-2.5) -- (11.6, -2.5);
    \draw [-{stealth[length=4mm]}] (9.8,-4.5) -- (11.6, -2.5);
    \draw [-{stealth[length=4mm]}] (9.8,-5.5) -- (11.6, -2.5);

    \draw
    [dotted] (-0.5,-3.25) -- (-0.5,-3.75)
    [dotted] (1.5,-3.25) -- (1.5,-3.75)
    [dotted] (3.8,-3.25) -- (3.8,-3.75)
    [dotted] (7,-3.25) -- (7,-3.75)
    [dotted] (10.1,-3.25) -- (10.1,-3.75);
    
\end{tikzpicture}
}

\end{minipage}
\end{center}
\caption{\label{fig:tateam} A TA team $G^i_j$ consisting of $2o$ TAs \cite{zhang2020convergence}. Here $I(x_1)$ means ``Include $x_1$'' and $E(x_1)$ means ``Exclude $x_1$''.  }
\end{figure}

%% file: Figures/voting.tex
\begin{figure}[htbp]
\begin{center}
\begin{minipage}{0.66\textwidth}
\begin{tikzpicture}[node distance = .35cm]

\node[label=left:TA team $1~~~~~~$] at (0,0) {};
\node[label=left:TA team $2~~~~~~$] at (0,-1) {};
\node[label=left:TA team $m-1$] at (0,-3) {};
\node[label=left:TA team $m~~~~$] at (0,-4) {};
\draw (-3, 0.4) -- (-3, -0.4) -- (0, -0.4) -- (0, 0.4) -- (-3, 0.4);
\draw (-3, -0.6) -- (-3, -1.4) -- (0, -1.4) -- (0, -0.6) -- (-3, -0.6);
\draw (-3, -2.6) -- (-3, -3.4) -- (0, -3.4) -- (0, -2.6) -- (-3, -2.6);
\draw (-3, -3.6) -- (-3, -4.4) -- (0, -4.4) -- (0, -3.6) -- (-3, -3.6);

\node[label=right: $C^i_1$] at (1,0) {};
\node[label=right: $C^i_2$] at (1,-1) {};
\node[label=right: $C^i_{m-1}$] at (1,-3) {};
\node[label=right: $C^i_m$] at (1,-4) {};
\node[label=right:$+$] at (2,0) {};
\node[label=right:$+$] at (2,-1) {};
\node[label=right:$+$] at (2,-3) {};
\node[label=right:$+$] at (2,-4) {};

\node[label=right: $\sum\limits_{j=1}^{m} C^i_j$] at (4.5,-2) {};
    
\draw [-{stealth[length=4mm]}] (0.1,0) -- (1,0);
\draw [-{stealth[length=4mm]}] (0.1,-1) -- (1,-1);
\draw [-{stealth[length=4mm]}] (0.1,-3) -- (1,-3);
\draw [-{stealth[length=4mm]}] (0.1,-4) -- (1,-4);

\draw [-{stealth[length=4mm]}] (2.6,0) -- (4.5, -2);
\draw [-{stealth[length=4mm]}] (2.6,-1) -- (4.5, -2);
\draw [-{stealth[length=4mm]}] (2.6,-3) -- (4.5, -2);
\draw [-{stealth[length=4mm]}] (2.6,-4) -- (4.5, -2);
    
\draw [dotted] (-2.5,-1.8) -- (-2.5,-2.2);
\draw [dotted] (1.5,-1.8) -- (1.5,-2.2);
\draw [dotted] (3,-1.8) -- (3,-2.2);

\end{tikzpicture}

\end{minipage}
\end{center}
\caption{\label{fig:TMVoting} TM voting architecture.}
\end{figure}

%% file: main.bbl
\begin{thebibliography}{10}
\providecommand{\url}[1]{#1}
\csname url@samestyle\endcsname
\providecommand{\newblock}{\relax}
\providecommand{\bibinfo}[2]{#2}
\providecommand{\BIBentrySTDinterwordspacing}{\spaceskip=0pt\relax}
\providecommand{\BIBentryALTinterwordstretchfactor}{4}
\providecommand{\BIBentryALTinterwordspacing}{\spaceskip=\fontdimen2\font plus
\BIBentryALTinterwordstretchfactor\fontdimen3\font minus
  \fontdimen4\font\relax}
\providecommand{\BIBforeignlanguage}[2]{{%
\expandafter\ifx\csname l@#1\endcsname\relax
\typeout{** WARNING: IEEEtran.bst: No hyphenation pattern has been}%
\typeout{** loaded for the language `#1'. Using the pattern for}%
\typeout{** the default language instead.}%
\else
\language=\csname l@#1\endcsname
\fi
#2}}
\providecommand{\BIBdecl}{\relax}
\BIBdecl

\bibitem{granmo2018tsetlin}
O.-C. Granmo, ``The {T}setlin {M}achine - {A} {G}ame {T}heoretic {B}andit
  {D}riven {A}pproach to {O}ptimal {P}attern {R}ecognition with {P}ropositional
  {L}ogic,'' \emph{arXiv:1804.01508}, Apr 2018.

\bibitem{Tsetlin1961}
M.~L. Tsetlin, ``{On Behaviour of Finite Automata in Random Medium},''
  \emph{Avtomat. i Telemekh}, vol.~22, no.~10, pp. 1345--1354, 1961.

\bibitem{ribeiro2016should}
M.~T. Ribeiro, S.~Singh, and C.~Guestrin, ``{W}hy {S}hould {I} {T}rust {Y}ou?:
  {E}xplaining the {P}redictions of {A}ny {C}lassifier,'' in \emph{Proceedings
  of the 22nd ACM SIGKDD international conference on knowledge discovery and
  data mining}.\hskip 1em plus 0.5em minus 0.4em\relax ACM, 2016, pp.
  1135--1144.

\bibitem{Rudin2019}
C.~Rudin, ``{Stop Explaining Black Box Machine Learning Models for High Stakes
  Decisions and Use Interpretable Models Instead},'' \emph{Nature Machine
  Intelligence}, vol.~1, no.~5, pp. 206--215, 2019.

\bibitem{valiant12}
L.~G. Valiant, ``{A} {T}heory of the {L}earnable,'' \emph{Communications of the
  ACM}, vol.~27, no.~11, pp. 1134--1142, 1984.

\bibitem{Blakely2020}
C.~D. Blakely and O.-C. Granmo, ``{Closed-Form Expressions for Global and Local
  Interpretation of Tsetlin Machines with Applications to Explaining
  High-Dimensional Data},'' \emph{arXiv preprint arXiv:2007.13885}, 2020.

\bibitem{wheeldon2020learning}
A.~{Wheeldon}, R.~{Shafik}, T.~{Rahman}, J.~{Lei}, A.~{Yakovlev}, and O.-C.
  {Granmo}, ``{Learning Automata based Energy-efficient AI Hardware Design for
  IoT},'' \emph{Philosophical Transactions of the Royal Society A}, 2020.

\bibitem{shafik2020explainability}
R.~{Shafik}, A.~{Wheeldon}, and A.~{Yakovlev}, ``{Explainability and
  Dependability Analysis of Learning Automata based AI Hardware},'' in
  \emph{IEEE 26th International Symposium on On-Line Testing and Robust System
  Design (IOLTS)}.\hskip 1em plus 0.5em minus 0.4em\relax IEEE, 2020.

\bibitem{granmo2019convolutional}
O.-C. Granmo, S.~Glimsdal, L.~Jiao, M.~Goodwin, C.~W. Omlin, and G.~T. Berge,
  ``{The Convolutional {T}setlin Machine},'' \emph{arXiv preprint
  arXiv:1905.09688}, 2019.

\bibitem{abeyrathna2019nonlinear}
K.~D. {Abeyrathna}, O.-C. {Granmo}, X.~{Zhang}, L.~{Jiao}, and M.~{Goodwin},
  ``{The Regression Tsetlin Machine - A Novel Approach to Interpretable
  Non-Linear Regression},'' \emph{Philosophical Transactions of the Royal
  Society A}, vol. 378, 2019.

\bibitem{abeyrathna2020integerregression}
K.~D. Abeyrathna, O.-C. Granmo, and M.~Goodwin, ``{A Regression Tsetlin Machine
  with Integer Weighted Clauses for Compact Pattern Representation,},'' in
  \emph{International Conference on Industrial, Engineering and Other
  Applications of Applied Intelligent Systems}.\hskip 1em plus 0.5em minus
  0.4em\relax Springer, 2020.

\bibitem{berge2019using}
G.~T. Berge, O.-C. Granmo, T.~O. Tveit, M.~Goodwin, L.~Jiao, and B.~V.
  Matheussen, ``{Using the {T}setlin Machine to Learn Human-interpretable Rules
  for High-accuracy Text Categorization with Medical Applications},''
  \emph{IEEE Access}, vol.~7, pp. 115\,134--115\,146, 2019.

\bibitem{rohan2021AAAI}
R.~Yadav, L.~Jiao, O.-C. Granmo, and M.~Goodwin, ``{Human-Level Interpretable
  Learning for Aspect-Based Sentiment Analysis},'' in \emph{AAAI}, 2021.

\bibitem{abeyrathna2019scheme}
K.~D. Abeyrathna, O.-C. Granmo, X.~Zhang, and M.~Goodwin, ``{A Scheme for
  Continuous Input to the {T}setlin Machine with Applications to Forecasting
  Disease Outbreaks},'' in \emph{International Conference on Industrial,
  Engineering and Other Applications of Applied Intelligent Systems}.\hskip 1em
  plus 0.5em minus 0.4em\relax Springer, 2019, pp. 564--578.

\bibitem{abeyrathna2020integer}
K.~D. Abeyrathna, O.-C. Granmo, and M.~Goodwin, ``{Extending the Tsetlin
  Machine With Integer-Weighted Clauses for Increased Interpretability},''
  \emph{IEEE Access}, 2021.

\bibitem{zhang2020convergence}
X.~Zhang, L.~Jiao, O.-C. Granmo, and M.~Goodwin, ``{On the Convergence of
  Tsetlin Machines for the IDENTITY-and NOT Operators},'' \emph{arXiv preprint
  arXiv:2007.14268}, 2020.

\bibitem{Narendra1989LearningIntroduction}
K.~S. Narendra and M.~A.~L. Thathachar, \emph{{Learning Automata: An
  Introduction}}.\hskip 1em plus 0.5em minus 0.4em\relax Prentice-Hall, Inc.,
  1989.

\bibitem{zhang2019conclusive}
X.~Zhang, L.~Jiao, B.~J. Oommen, and O.-C. Granmo, ``{A Conclusive Analysis of
  the Finite-time Behavior of the Discretized Pursuit Learning Automaton},''
  \emph{IEEE Transactions on Neural Networks and Learning Systems}, vol.~31,
  no.~1, pp. 284--294, 2020.

\end{thebibliography}
